\documentclass[sigconf,authorversion]{acmart}

\usepackage{hyperref}
\usepackage{url}

\usepackage{dsfont}
\usepackage{amsmath}
\usepackage{xspace}
\usepackage{booktabs}
\usepackage{subcaption}
\usepackage{tcolorbox}
\usepackage{pgfplotstable, pgfmath} 
\usepackage[noenumitem]{eda} 
\usepackage{colortbl}
\usepackage{prettyplots} 
\usepackage{wrapfig}
\usepackage{thm-restate}
\usepackage{algorithm}
\usepackage{algpseudocode}
\usepackage{xcolor}
\usepackage{stfloats}      

\DeclareMathOperator*{\argmax}{arg\,max}

\pgfplotsset{
    colormap/viridis,
}
\graphicspath{ {./figs/} }
\pgfplotsset{compat=1.18}
\usetikzlibrary{calc} 
\usepgfplotslibrary{groupplots}
\usetikzlibrary{positioning}
\usetikzlibrary{backgrounds}
\usetikzlibrary{external}
\usepgfplotslibrary{fillbetween}
\usepgfplotslibrary{colorbrewer} 
\usetikzlibrary{pgfplots.colormaps} 

\title[Differential Subgroup Discovery]{Differential Subgroup Discovery: \\ Characterizing Where Two Populations Differ, and Why}

\author{Sascha Xu}
\affiliation{%
  \institution{CISPA Helmholtz Center \\for Information Security}
  \city{Saarbrücken}
  \country{Germany}
}
\email{sascha.xu@cispa.de}
\author{Jilles Vreeken}
\affiliation{%
  \institution{CISPA Helmholtz Center \\for Information Security}
  \city{Saarbrücken}
  \country{Germany}
}
\email{vreeken@cispa.de}

\newcommand{\featvar}{X}
\newcommand{\targvar}{Y}
\newcommand{\attributevar}{A}
\newcommand{\generality}{\mathcal{G}_{\gamma}}
\newcommand{\exceptionality}{\mathcal{E}}
\newcommand{\sufficiency}{\mathcal{C}}
\newcommand{\subgroup}{S}
\newcommand{\subgroupf}{s}
\newcommand{\subgroupfhat}{\hat{\subgroupf}}
\newcommand{\phat}{\hat{p}}

\newcommand{\attributei}{\attribute^{(i)}}
\newcommand{\feati}{\feat^{(i)}}
\newcommand{\targi}{\targ^{(i)}}

\newcommand{\doit}{\mathit{do}}

\newcommand{\temp}{t}

\newcommand{\pred}{\pi}

\newcommand{\nfeat}{d}
\newcommand{\nsamples}{n}

\newcommand{\nclasses}{m}

\newcommand{\feat}{x}
\newcommand{\targ}{y}
\newcommand{\attribute}{a}

\newcommand{\featdomain}{\mathcal{X}}

\newcommand{\fvec}[1]{\mathbf{#1}} 

\newcommand{\lowerbound}{a}

\newcommand{\upperbound}{b}

\newcommand{\andweight}{w}
\newcommand{\andweightvec}{\fvec{\andweight}}

\newcommand{\ourmethod}{\textsc{DiffSub}\xspace}
\newcommand{\syflow}{\textsc{SyFlow}\xspace}
\newcommand{\causaltree}{\textsc{HonestTree}\xspace}
\newcommand{\uplifttree}{\textsc{UpliftTree}\xspace}
\newcommand{\pysubgroup}{\textsc{PySubgroup}\xspace}
\newcommand{\xlearner}{\textsc{XLearner}\xspace}
\newcommand{\causalforest}{\textsc{CausalForest}\xspace}

\definecolor{cbc1}{RGB}{1,0.2,0.3} \definecolor{cbc2}{RGB}{254,219,199} \definecolor{cbc3}{RGB}{1,2,3} \definecolor{cbc4}{RGB}{209,229,240}

\pgfkeys{
    /name mapper/.cd,
    our-rule-list/.code = {\ourmethod},
    greedy_rule_list/.code = {\greedy},
    mdl-rule-list/.code = {\mdlrl},
    rlnet/.code = {\rlnet},
    drs/.code = {\drnet},
    rrl/.code = {\rrl},
    optimal_rule_list/.code = {\corels},
    bayesian_rule_list/.code = {\sbrl},
    xgboost/.code = {\xgboost}
}

\newtheorem{assumption}{Assumption}

\def\R{{\mathbb{R}}}

\newcommand{\indep}{\mathop{\perp\!\!\!\perp}\nolimits}

\begin{document}

\begin{abstract}
We study the problem of understanding where two populations differ within a feature space, which
we formalize in the concept of a differential subgroup: a subset of individuals from both populations who, despite sharing similar characteristics, exhibit exceptional differences in a target outcome.
Differential subgroups reveal the regions of the feature space where population-level gaps are most pronounced and can help
practitioners identify the covariate combinations that are structurally responsible for these differences, e.g.\ in clinical analysis, model diagnostics, or treatment-effect studies.
We introduce a general optimization objective for discovering differential subgroups and establish conditions under which the resulting subgroups admit a causal interpretation of population differences.
We propose \ourmethod, a gradient-based approach that discovers interpretable differential subgroups in tabular data.
Across synthetic benchmarks, medical case studies, model-error analyses, and treatment-effect settings, \ourmethod identifies informative subgroups that reveal where population differences arise and why.
\end{abstract}
\maketitle


\definecolor{tabblue}{HTML}{1F77B4}
\definecolor{taborange}{HTML}{FF7F0E}

\begin{figure*}[t]
\centering

\begin{minipage}[t]{0.2\textwidth}
\vspace*{0.35cm}
\centering
\begin{tikzpicture}
\pgfplotsset{every axis/.style={tick style={black}, ticklabel style={font=\small}}}
\begin{axis}[
    title={Age vs Sex},
    xlabel={Age},
    ytick={0,1},
    yticklabels={Men, Women},
    xmin=30, xmax=80,
    scatter,
    pretty scatter,
    pretty labelshift,
    width=\textwidth,
    height=3.6cm,
]
\addplot+[only marks, mark=*, mark size=1.2pt,scatter=false,draw=tabblue] table[x=age, y=sex, col sep=comma] {expres/intro_example_scatter_0.csv};
\addplot+[only marks, mark=*, mark size=1.2pt, scatter=false, draw=taborange] table[x=age, y=sex, col sep=comma] {expres/intro_example_scatter_1.csv};
\end{axis}
\end{tikzpicture}
\subcaption{Presence/absence of heart disease vs.~age and sex.}
\label{fig:intro-scatter}
\end{minipage}
\hfill
\begin{minipage}[t]{0.4\textwidth}
\vspace*{0.cm}
\centering
\begin{tikzpicture}
\pgfplotsset{every axis/.style={tick style={black}, ticklabel style={font=\small}}}
\begin{groupplot}[
    group style={group size=3 by 1, horizontal sep=0.7cm},
    width=0.4\textwidth,
    height=3.6cm,
    legend columns=2,
    legend pos=north west,
    legend style={/tikz/every even column/.append style={column sep=0.6em}, draw=none, font=\small,
    at={(0,1.6)}},
]

\nextgroupplot[
    title={Population},
    ylabel={Probability},
    ymin=0, ymax=1,
    xtick={0,1},
    xticklabels={Men, Women},
    xlabel={Sex},
    pretty ybar stacked,
    pretty labelshift,
]
\addplot+[ybar, bar width=12pt, color=tabblue] coordinates {(0,0.43) (1,0.73)};
\addlegendentry{No Heart Disease}
\addplot+[ybar, bar width=12pt, color=taborange] coordinates {(0,0.57) (1,0.27)};
\addlegendentry{Heart Disease}
\hfill
\nextgroupplot[
    title={Age 45--55},
    ymin=0, ymax=1,
    xtick={0,1},
    xticklabels={Men, Women},
    xlabel={Sex},
    pretty ybar stacked,
    pretty labelshift
]
\addplot+[ybar, bar width=12pt, tabblue] coordinates {(0,0.51) (1,0.89)};
\addplot+[ybar, bar width=12pt, taborange] coordinates {(0,0.49) (1,0.11)};
\hfill
\nextgroupplot[
    title={Age 56--62},
    ymin=0, ymax=1,
    xtick={0,1},
    xticklabels={Men, Women},
    xlabel={Sex},
    pretty ybar stacked,
    pretty labelshift
]
\addplot+[ybar, bar width=12pt, tabblue] coordinates {(0,0.29) (1,0.43)};
\addplot+[ybar, bar width=12pt, taborange] coordinates {(0,0.71) (1,0.57)};

\end{groupplot}
\end{tikzpicture}
\subcaption{Heart disease per sex for overall population and age groups. The difference between men and women is larger around ages 45--55 (middle) and 
smaller at 56--62 (right).}
\label{fig:intro-bars}
\end{minipage}
\hfill
\begin{minipage}[t]{0.34\textwidth}
\vspace*{0pt}
\centering
\begin{minipage}[t]{0.5\textwidth}
\vspace*{20pt}
\fbox{\parbox[t]{\textwidth}{
{\footnotesize \textbf{Subgroup}}\\[-2pt]
{\footnotesize $\texttt{Age} \in (36,63)$}\\[-2pt]
{\footnotesize $\texttt{Cholesterol} \in (235,500)$}\\[-2pt]
{\footnotesize $\texttt{Max Heartrate} \in (156,193)$}\\[2pt]
{\footnotesize \textbf{Coverage}}\\[-2pt]
{\footnotesize Men: 15\%, Women: 28\%}
}}
\end{minipage}
\hfill
\begin{minipage}[t]{0.4\textwidth}
\vspace*{0.43cm}
\begin{tikzpicture}
\pgfplotsset{every axis/.style={tick style={black}, ticklabel style={font=\small}}}
\begin{axis}[
    title={Subgroup},
    ymin=0, ymax=1,
    xtick={0,1},
    xticklabels={Men, Women},
    xlabel={Sex},
    width=\textwidth,
    height=3.6cm,
    ytick={0,0.5,1},
    pretty ybar stacked,
    pretty labelshift,
]
\addplot+[ybar, bar width=10pt, color=tabblue] coordinates {(0,0.76) (1,0.76)};
\addplot+[ybar, bar width=10pt, color=taborange] coordinates {(0,0.24) (1,0.24)};
\end{axis}
\end{tikzpicture}
\end{minipage}
\subcaption{A \emph{differential subgroup}: Individuals with low age, high cholesterol and high max heart rate 
show similar heart disease rates.}
\label{fig:intro-subgroup}
\end{minipage}

\caption{Analysis of heart disease dataset \citep{detrano1989international}.
Rates of heart disease differ substantially between women and men in the overall population (a). 
When stratifying by age groups, these differences change in magnitude (b), 
but it remains unclear which combinations of risk factors drive them.
On the right, we show a discovered \emph{differential subgroup}:
younger individuals with high cholesterol and high maximum heart rate exhibit similarly high heart disease rates, 
even after regularizing the influence of other covariates (c).
}
\label{fig:intro-example}
\end{figure*}
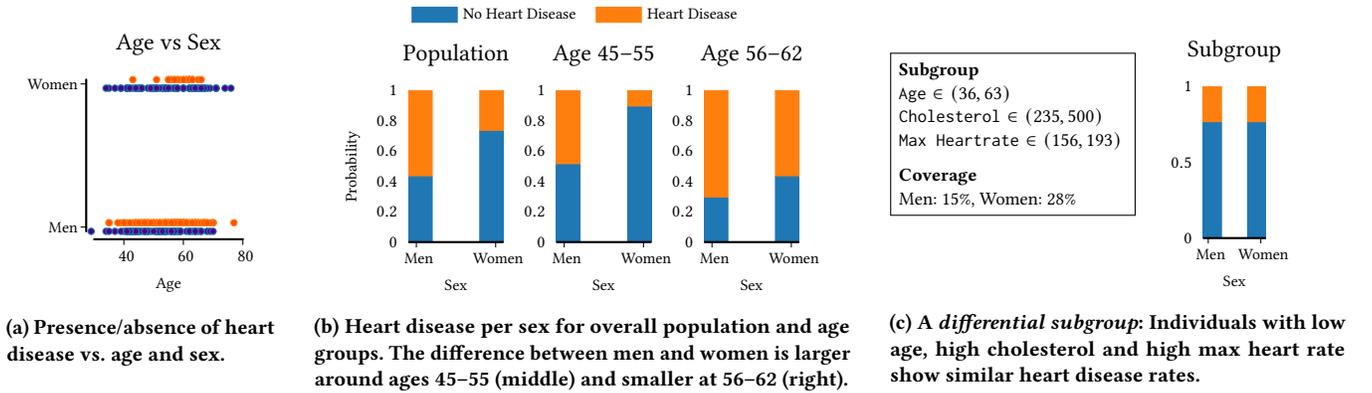

\section{Introduction}
What are the differences between two populations, and where do they stem from?
This question appears across sociology, medicine, and many other fields, but 
it rarely has a straightforward answer.
Population-level differences often increase, decrease, 
or even reverse once we examine specific sub-populations that capture confounding or interacting covariates.
Consider the heart disease dataset \citep{detrano1989international} shown in Figure \ref{fig:intro-example}.
Averaged over all individuals, men have a substantially higher rate of heart disease than women.
However, when stratifying by age groups, these differences vary considerably:
The gap widens for individuals aged 45--55 but narrows for those aged 56--62 (Figure~\ref{fig:intro-bars}).
This suggests that age modulates the relationship between sex and heart disease, raising a broader question:
which combinations of features drive such shifts, and can we identify them automatically from data?

Existing methods provide only partial answers.
Traditional subgroup discovery \citep{atzmueller2015subgroup} identifies sub-populations with exceptional behavior, 
but it operates within a single population and does not compare two groups to understand how they differ.
In clinical trials, the contrast between two groups (treatment and control) is central, yet subgroup analyses are typically defined a priori to preserve statistical power \citep{cook2004subgroup}.
Data-driven approaches that do discover subgroups rely on restrictive data-generating assumptions specific to randomized control trials \citep{foster2011subgroup,ballarini2018subgroup} or structured observational studies \citep{athey2016recursive}, 
limiting their applicability to more general, non-controlled settings.
Fairness analysis offers tools for quantifying and mitigating individual- or group-level disparities \citep{mehrabi2021survey}, but does not reveal where in the feature space such differences arise or why they occur.
In general, it remains an open problem to identify the subgroups in which the target variable exhibits pronounced differences between two populations.

To address this challenge, we formalize the problem of differential subgroup discovery.
Its goal is to discover and name sets of individuals from both groups who, despite sharing similar characteristics, exhibit significant differences in the local distribution of a target variable between the two populations.
This perspective generalizes traditional subgroup discovery from isolating exceptional behavior within one population to identifying where two populations diverge or converge.

In the heart disease example, our method uncovers a differential subgroup consisting of younger individuals with high cholesterol 
and high maximum heart rate who exhibit similarly high rates of heart disease regardless of sex (Figure \ref{fig:intro-subgroup}).
This aligns with well-established medical findings indicating that age, cholesterol, and heart rate are risk factors for heart disease \citep{wilson1998prediction,lang2010elevated}.

Differential subgroups thereby help understand \emph{where} and \emph{why} population-level differences arise by identifying the combinations of covariates that drive these differences.
To discover such patterns, we introduce \ourmethod, a method for differential subgroup discovery in tabular data.
Our key contributions are:
\begin{enumerate}
    \item We formalize the notion of differential subgroups and introduce the three desiderata such subgroups should satisfy:
    \begin{itemize}
        \item \emph{Exceptionality}: The target distributions differ exceptionally within the subgroup.
        \item \emph{Generality}: The subgroup represents a significant portion of both populations.
        \item \emph{Covariate Independence}: The observed difference cannot be attributed to other covariates.
    \end{itemize}
    \item We identify the conditions under which differential subgroups reflect causal differences attributable to population membership.
    \item We develop \ourmethod, a gradient-based approach that efficiently discovers differential subgroups 
    and outputs interpretable rule-based descriptions.
\end{enumerate}
Across synthetic datasets with known ground truth, real-world medical and treatment-effect case studies, and analyses comparing prediction errors between models, 
\ourmethod discovers interpretable subgroups that expose the structure underlying population-level differences.

\section{Preliminaries}
We consider a dataset of $\nsamples$ individuals $(\feat^{(i)}, \attribute^{(i)}, \targ^{(i)})$, $i \in \{1,\dots,\nsamples\}$, where $\targ$ is the \textbf{target variable}, 
$\feat$ represents the \textbf{descriptive features}, and $\attribute \in \{0, 1\}$ is the \textbf{binary attribute} that partitions the population into two groups. 
From a statistical perspective, we assume that each sample $(\feat^{(i)}, \attribute^{(i)}, \targ^{(i)})$ is a realization of the joint distribution of $P(\featvar, \attributevar, \targvar)$.
We denote random variables by capitals, write $p$ for their density, and $P$ for their distributions.

A \textbf{subgroup} selects a subset of individuals with shared feature characteristics through its membership function $\subgroupf: \R^{\nfeat} \to \{0, 1\}$.
The subgroup membership $\subgroupf(\featvar)=1$ indicates that an individual with features $\featvar$ belongs to the subgroup, while $\subgroupf(\featvar)=0$ indicates non-membership.
We denote by $\mathcal{S}$ the family of subgroup membership functions considered in this work.

For a subgroup $\subgroupf \in \mathcal{S}$ and a group $\attribute \in \{0,1\}$,
the distribution of the target variable within that subgroup is
\begin{equation}
    P_{\attribute,\subgroupf}(\targvar) 
    \coloneqq 
    P\bigl(\targvar \mid \attributevar=\attribute,\; \subgroupf(\featvar)=1\bigr)\,.
\end{equation}

To ensure that subgroup-specific quantities are well-defined for both populations, we require that each group occurs with positive probability across the support of $X$.

\begin{assumption}[Positivity]
For all covariate values $\feat$ in the support of $\featvar$, i.e.~where $p(\feat)>0$, both
populations have positive probability of occurring, i.e.,
\begin{equation}
    \forall \feat \text{ with } p(\feat)>0:\quad
    0 < \mathbb{P}(A=1 \mid X=\feat) < 1\;.
\end{equation}
\label{as:positivity}
\end{assumption}

Positivity ensures that for any subgroup $s \in \mathcal{\subgroup}$ intersecting the support of $X$, the conditional distributions $P_{0,s}(Y)$ and $P_{1,s}(Y)$ are both identifiable and comparable.
With these definitions in place, we now introduce differential subgroups.

\section{Differential Subgroups}
From the space of all possible subgroups,
our goal is to discover a subgroup $\subgroupf \in \mathcal{\subgroup}$, 
for which the group membership $\attributevar\in \{0,1\}$ results in an exceptional difference in the distribution of the target variable $P_{0,\subgroupf}(\targvar)$ and $P_{1,\subgroupf}(\targvar)$.
In the following, we motivate and formally define the properties of an interesting differential subgroup $\subgroupf$.

\subsection{Exceptionality}
In the differential setting, we are interested not in how a subgroup differs from the overall population, but in how the two populations differ within the subgroup. 
To capture this local contrast, we compare the subgroup-conditioned target distributions
$P_{0,\subgroupf}(\targvar)$ and $P_{1,\subgroupf}(\targvar)$ using a divergence measure $D$.
This motivates the following definition of differential exceptionality.

\begin{definition}[Differential Exceptionality]
    \label{def:exceptionality}
    The differential exceptionality of a subgroup $\subgroupf$ is defined as the divergence $D$ between the target variable distributions in the two groups as per
    \begin{equation}
        \mathcal{E}(\subgroupf) = D\left(P_{0,\subgroupf}(\targvar), P_{1,\subgroupf}(\targvar)\right)\;.
    \end{equation}
\end{definition}

We show the difference between differential and standard (non-differential) exceptionality in Figure \ref{fig:differential_example}.
Under the same subgroup $\subgroupf$, the overall distribution $P(\targvar)$ and the subgroup-conditioned distribution $P(\targvar \mid \subgroupf(\featvar)=1)$ are virtually identical.
However, if we contrast the two populations, we observe a significant difference in the income of men and women in this particular subgroup.
That is, differential exceptionality focuses only on the difference within the subgroup, independent of the distribution outside the subgroup.
This enables it to discover characteristics which one would not obtain by maximizing standard exceptionality between subgroup and overall population.

The choice of divergence measure $D$ is flexible in our framework.
It can be any measure that quantifies the difference between two distributions, e.g.~the Kullback-Leibler/Jensen-Shannon divergence, or Wasserstein distance.
In general, the higher the divergence, the more exceptional the subgroup is.
To search for subgroups with minimal difference, the objective changes to minimize the divergence instead of maximizing it.

\subsection{Support}
To be able to compare the target distribution between the two populations, any subgroup $\subgroupf$ must have support in both populations,
i.e.
\begin{equation}
    E[\subgroupf(\featvar) \mid \attributevar=0] > 0 \quad \text{and} \quad E[\subgroupf(\featvar) \mid \attributevar=1] > 0\;.
\end{equation}
On a specific dataset $\left\{\feat^{(i)}\right\}_{i=1}^{\nsamples}$, this means that there must be at least one sample $\feat^{(i)}$ from each group where $\subgroupf(\feat^{(i)}) = 1$.
In practice, supports beyond the minimal requirement are desirable,
such that a subgroup $\subgroupf$ is supported by a significant portion of the population.
Large supports also yield more reliable estimates of $P_{0,\subgroupf}(\targvar)$ and $P_{1,\subgroupf}(\targvar)$.

\begin{restatable}[Generality]{definition}{defgenerality}
    We define the generality of a subgroup $\subgroupf$ over two populations as the geometric mean over the individual supports as per 
    \begin{align}
        \generality(\subgroupf) = \left(\sqrt{E[\subgroupf(\featvar) \mid \attributevar=0] \cdot E[\subgroupf(\featvar) \mid \attributevar=1]}\right)^{\gamma}\;.
    \end{align}
\end{restatable}

Under the generality $\generality$, a subgroup $\subgroupf$ is supported, iff $\mathcal{G}_{\gamma}(\subgroupf) > 0$ which is the case if both $E[\subgroupf(\featvar) \mid \attributevar=0] > 0$ and $E[\subgroupf(\featvar) \mid \attributevar=1] > 0$.
The parameter $\gamma$ controls the sensitivity of the generality to the individual supports. 
Larger values of $\gamma$ penalize subgroups with small supports more heavily, while smaller values allow more flexibility in subgroup size.
\begin{figure*}[t]
    \centering
    \begin{subfigure}[t]{0.48\textwidth}
        \centering
        \includegraphics[width=\linewidth]{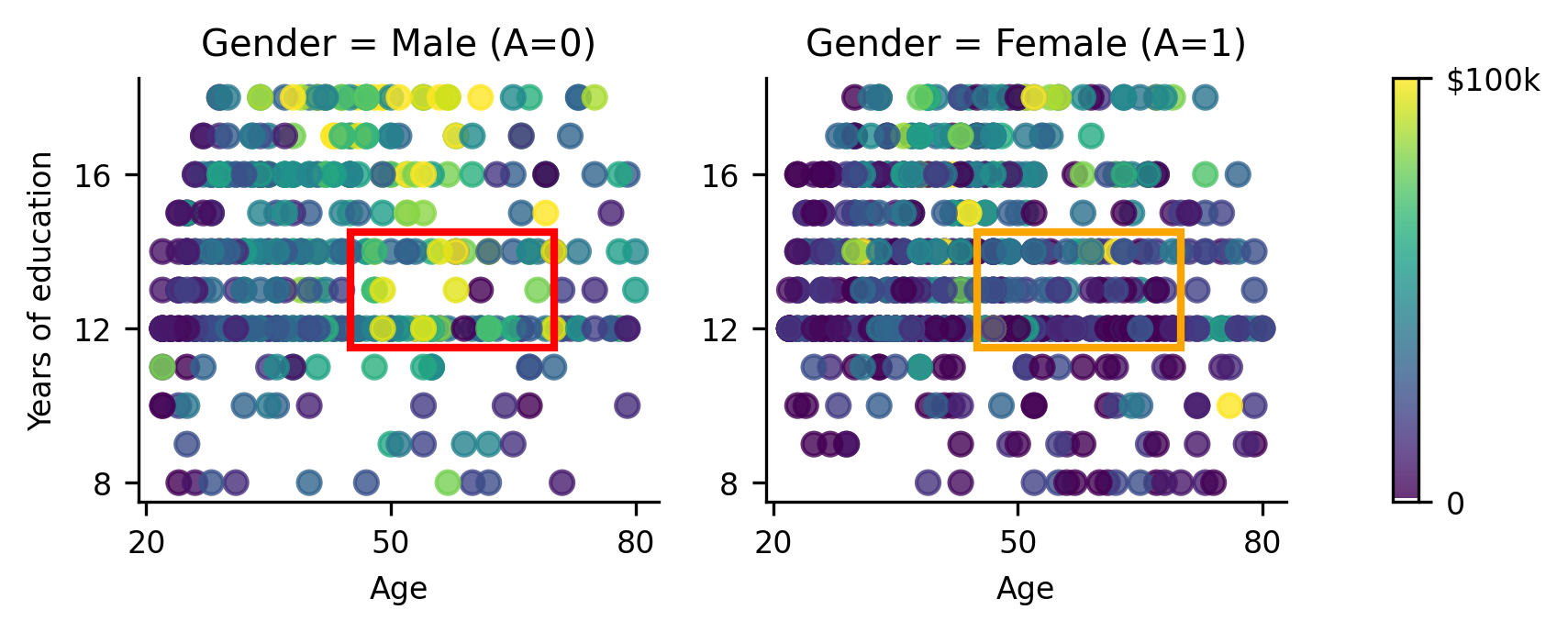}
        \caption{Wage distribution for men (left) and women (right) in relation to age and education level.}
    \end{subfigure}
    \hfill
    \begin{subfigure}[t]{0.48\textwidth}
        \centering
            \includegraphics[width=\linewidth]{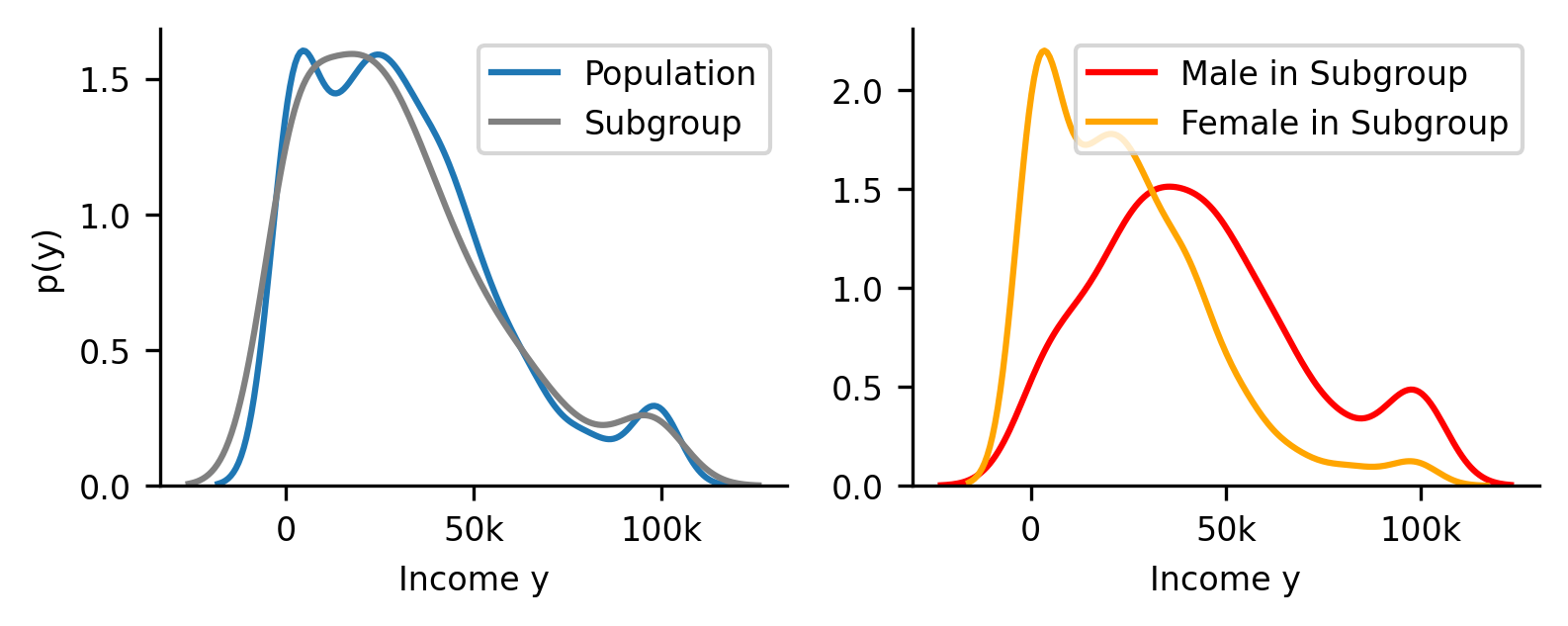}
            \caption{Wage distribution of \textcolor{blue}{population} vs \textcolor{gray}{subgroup} (left) and of \textcolor{red}{men} vs \textcolor{orange}{women} in subgroup (right).}
            \end{subfigure}
    \caption{The income distribution of the subgroup $\subgroupf$ with $\texttt{Age}\in [45, 70]$ and $\texttt{Education}\in [12, 14]$ does not deviate significantly
    from the overall population (left). But, when comparing men and women within that subgroup, there is a significant difference comparing the respective distributions (right).
    This subgroup is differentially exceptional.}
    \label{fig:differential_example}
\end{figure*}

\subsection{Covariate Dependence}
Exceptionality and generality characterize how pronounced and how well supported a subgroup difference is, but they do not address whether this difference reflects the effect of group membership itself. 
Within a subgroup, individuals may still vary considerably in their features $\featvar$, and such covariate influence can induce apparent differences between the two populations even when no true population effect is present. 
To separate meaningful population differences from those arising purely due to variation in $\featvar$ within the subgroup, we require a criterion that assesses how much the target distribution still depends on the features once subgroup membership is fixed.  
This motivates the following notion.
\begin{definition}[Covariate Dependence]
    \label{def:covariate}
    We define the covariate dependence as:
    \begin{equation}
        \sufficiency_0(\subgroupf) = E_{\featvar \mid \attributevar=0, \subgroupf(\featvar)=1}\left[D\left(P_{0,\subgroupf}(\targvar \mid  \featvar = \feat), P_{0,\subgroupf}(\targvar)\right)\right]
    \end{equation}
    and $\sufficiency_1(\subgroupf)$ analogously for the second group.
    The covariate dependence $\sufficiency(\subgroupf)$ is then given by
    \begin{equation}
        \sufficiency(\subgroupf) = \sufficiency_0(\subgroupf) + \sufficiency_1(\subgroupf)\;.
    \end{equation}
    \end{definition}
The covariate dependence $\sufficiency(\subgroupf)$ measures how much the target distribution within the subgroup is influenced by the features.
In the ideal case, the subgroup membership $\subgroupf$ captures all relevant information about the target variable $\targvar$ and $\sufficiency(\subgroupf)=0$.
Then the local target is conditionally independent of the features $\featvar$, i.e.~$(\targvar \indep \featvar) \mid (\subgroupf(\featvar)=1, \attributevar=0)$ and $(\targvar \indep \featvar) \mid (\subgroupf(\featvar)=1, \attributevar=1)$.
That means that the observed difference in the target variable can solely be attributed to the subgroup membership $\attributevar$ alone, and not due to any confounding effect from the features $\featvar$.

Minimizing the covariate dependence $\sufficiency(\subgroupf)$ aims to minimize the influence of indirect causal effects and thus avoid 
discovering false positive subgroups. 
By ensuring that the features do not locally influence the target variable, we can be more confident that the observed differences are 
not due to a shift in features caused by $\attributevar$, but instead directly driven by $\attributevar$ itself.
As we show in Section~\ref{sec:causal}, $\mathcal{C}(s)$ is the key quantity that determines when differential subgroup differences admit a causal interpretation.

\subsection{Problem Statement}
We can now define the problem of discovering differential subgroups as follows:
\begin{equation}
    \label{eq:objective_general}
    \argmax_{\subgroupf \in \mathcal{\subgroup}} \generality (\subgroupf) \cdot \mathcal{E}(\subgroupf) - \lambda \mathcal{C}(\subgroupf)\;.
\end{equation}
Here, $\lambda$ controls the trade-off between capturing a strong population-level difference and enforcing covariate independence within the subgroup. 
The objective balances these components: $\generality(\subgroupf)$ ensures that only subgroups supported in both populations are considered and scales $\mathcal{E}(\subgroupf)$, which quantifies the local outcome difference between them. 
Together with the covariate dependence penalty $\mathcal{C}(\subgroupf)$, the objective favors subgroups that are simultaneously general, exceptional, and minimally confounded.
\section{Causality of Differential Subgroups}
\label{sec:causal}
In the previous section, we defined differential subgroups through statistical contrast. 
Here, we analyze under which conditions such differences can be given a causal interpretation.
We will use Pearl's $\doit$-calculus \citep{pearl2009causality} to analyze the causal relationships between the group membership $\attributevar$, features $\featvar$, and target $\targvar$.
We now examine how differential subgroups behave under several prototypical causal structures.

\subsection{Observational Study ($\featvar \to \attributevar, \featvar \to \targvar, \attributevar \to \targvar$)}
In an observational study, the goal is to study the effect between treatment and control group, where membership $\attributevar$ is not randomized,
but rather determined by the features $\featvar$ (Fig.~\ref{fig:graph_observational}).
Because $X$ influences both treatment assignment and the outcome, the observational quantity $P(Y \mid A=a)$ differs from the interventional $P(Y \mid \text{do}(A=a))$.

Under certain conditions however, it is possible to identify the interventional distribution.
If $X$ satisfies the backdoor criterion, i.e.~if $\featvar$ blocks all backdoor paths between $\attributevar$ and $\targvar$,
then the interventional distribution can be identified as
\begin{equation}
    p(\targ \mid \doit(\attributevar=a)) = \int_{\feat} p(\targ \mid \featvar = \feat, \attributevar=a) p(\featvar=\feat)\;.
    \label{eq:backdoor}
\end{equation}
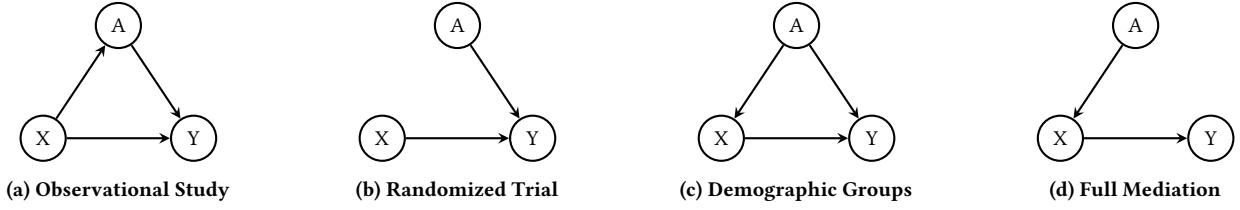
\begin{figure*}[t]
    \begin{subfigure}[t]{0.24\linewidth}
        \centering
        \begin{tikzpicture}[->, >=stealth, thick, every node/.style={circle, draw, minimum size=6mm, font=\small}]
            \node (A) at (0,1.5) {A};
            \node (X) at (-1,0) {X};
            \node (Y) at (1,0) {Y};
            \draw (X) -- (Y);
            \draw (X) -- (A);
            \draw (A) -- (Y);
        \end{tikzpicture}
        \caption{Observational Study}
        \label{fig:graph_observational}
    \end{subfigure}
    \hfill
    \begin{subfigure}[t]{0.24\linewidth}
        \centering
        \begin{tikzpicture}[->, >=stealth, thick, every node/.style={circle, draw, minimum size=6mm, font=\small}]
            \node (A) at (0,1.5) {A};
            \node (X) at (-1,0) {X};
            \node (Y) at (1,0) {Y};
            \draw (X) -- (Y);
            \draw (A) -- (Y);
        \end{tikzpicture}
        \caption{Randomized Trial}
        \label{fig:graph_randomized}
    \end{subfigure}
    \hfill
    \centering
    \begin{subfigure}[t]{0.24\linewidth}
        \centering
        \begin{tikzpicture}[->, >=stealth, thick, every node/.style={circle, draw, minimum size=6mm, font=\small}]
            \node (A) at (0,1.5) {A};
            \node (X) at (-1,0) {X};
            \node (Y) at (1,0) {Y};
            \draw (A) -- (X);
            \draw (A) -- (Y);
            \draw (X) -- (Y);
        \end{tikzpicture}
        \caption{Demographic Groups}
        \label{fig:graph_demographic}
    \end{subfigure}
    \hfill
    \begin{subfigure}[t]{0.24\linewidth}
        \centering
        \begin{tikzpicture}[->, >=stealth, thick, node distance=2cm, every node/.style={circle, draw, minimum size=6mm, font=\small}]
            \node (A) at (0,1.5) {A};
            \node (X) at (-1,0) {X};
            \node (Y) at (1,0) {Y};
            \draw (A) -- (X);
            \draw (X) -- (Y);
        \end{tikzpicture}
        \caption{Full Mediation}
        \label{fig:graph_mediator}
    \end{subfigure}
    \caption{Structural causal models for group indicator $\attributevar$, features $\featvar$, and target $\targvar$. When the covariate dependence is minimized, 
    i.e.~$\sufficiency(\subgroupf)=0$, differential subgroups capture the interventional distribution (\textbf{a}) or the controlled direct effect (\textbf{c}) of $\attributevar$ on $\targvar$.
    In a randomized trial (\textbf{b}), the interventional distribution is equal to the observational distribution regardless of $\sufficiency(\subgroupf)$.
    In the case of full mediation (\textbf{d}), there cannot be a differential subgroup with $\sufficiency(\subgroupf)=0$ and $\generality(\subgroupf)>0$.
    }
    \label{fig:causal_structure}
\end{figure*}

That is, if $\featvar$ is a valid backdoor, then it is possible to infer the interventional distribution
by integrating over the features $\featvar$.
Conditioning on $s(X)=1$ restricts attention to a region of the feature space, but does not itself remove the influence of $X$ on $Y$. 
Consequently, the resulting subgroup distribution $P_{a,s}(Y)$ need not match the interventional distribution. 
However, if the covariate dependence is minimized, i.e.~$\sufficiency(\subgroupf)=0$, 
then all remaining feature-based heterogeneity is removed within the subgroup
and the subgroup $\subgroupf$ captures the interventional distribution within that subgroup.

\begin{restatable}{proposition}{propCausalEffect}
    \label{prop:causal_effect}
    Let $\featvar$ be a valid backdoor to infer the interventional distribution as per Eq.~\eqref{eq:backdoor}.
    Then, any subgroup $\subgroupf$ that minimizes the covariate dependence $\sufficiency(\subgroupf)=0$ with generality $\generality(\subgroupf)>0$
    captures the interventional distribution within the subgroup
    \begin{equation}
        P_{\attribute,\subgroupf}(\targvar) = P(\targvar \mid \doit(\attributevar=a), \subgroupf(\featvar)=1)\;.
    \end{equation}
\end{restatable}
We provide the proof in Appendix \ref{sec:proof_causal_effect}.
The result suggests that given data from an observational study, differential subgroups search for a set of data points with $\subgroupf(\featvar)=1$ 
that capture the interventional distribution of the target variable $\targvar$, and which have an exceptional treatment effect. 

\subsection{Randomized Trial ($\featvar \to \targvar, \attributevar \to \targvar, \attributevar \indep \featvar$)}
If the treatment assignment $\attributevar$ is randomized, i.e.~$\attributevar$ is independent of the covariates $\featvar$,
the interventional distribution is equal to the observational distribution.
This extends naturally to all sub-distributions induced by a subgroup $\subgroupf$, regardless of whether the covariate dependence is minimal or not,
i.e.~$\sufficiency(\subgroupf) > 0$.

Minimizing $\sufficiency(\subgroupf)$ can still be desirable.
A low covariate dependence mainly ensures that the difference in target distribution is driven by $\attributevar$ and not some spurious dependence on $\featvar$.
In general, in a randomized control trial, differential subgroups may inform treatment policies by discovering patients with exceptional treatment effects.

\subsection{Demographic Group ($\attributevar \to \featvar, \featvar \to \targvar, \attributevar \to \targvar$)}
Next, we examine the setting where the attribute $\attributevar$ has an influence on the feature distribution $P(\featvar)$. 
This structure is common when searching for disparities between protected demographic groups, where membership $\attributevar$ is an inherent attribute like gender or ethnicity that can influence mediating features $\featvar$ (e.g., job, income).

The effect of $\attributevar$ on $\targvar$ is twofold:
Firstly, $\attributevar$ can have a direct effect ($\attributevar \to \targvar$).
Secondly, it can have an indirect effect mediated through the features ($\attributevar \to \featvar \to \targvar$).

To distinguish these pathways, causal inference defines the Controlled Direct Effect (CDE). The CDE measures the effect of changing $\attributevar$ while holding the mediating features $\featvar$ constant at a specific level $\feat$. Following \citet{pearl2022direct}, it is defined as:
\begin{align}
    \label{eq:cde}
    \text{CDE}(\feat) = E[\targvar | \doit(\attributevar=1, \featvar=\feat)]
    - E[\targvar | \doit(\attributevar=0, \featvar=\feat)]
\end{align}
Our framework aims to find subgroups where this causal relationship is simplified. If a subgroup $\subgroupf$ satisfies the conditional independence $(\targvar \indep \featvar) \mid (\attributevar=\attribute, \subgroupf(\featvar)=1)$ in the case of $\sufficiency(\subgroupf)=0$,
then the indirect path through $\featvar$ is effectively neutralized within that subgroup.
For such a subgroup, the CDE does not vary with $\feat$ among members of the subgroup and is identifiable directly from the observed difference in means
for any member $\feat$ with $\subgroupf(\feat)=1$:
\begin{equation}
    \quad \text{CDE}(\feat) = E[\targvar \mid \attributevar=1, \subgroupf(\featvar)=1] - E[\targvar \mid \attributevar=0, \subgroupf(\featvar)=1]
\end{equation}
For mediating variables, the objective in Eq.~\eqref{eq:objective_general} balances discovering subgroups with a large total effect against discovering subgroups whose effect is largely direct. 
Increasing the regularization parameter $\lambda$ places more weight on minimizing the covariate dependence $\mathcal{C}(s)$, thereby favoring subgroups in which the indirect pathway $\attributevar \to \featvar \to \targvar$ is neutralized and the observed exceptionality reflects a controlled direct effect. 
Conversely, setting $\lambda$ close to zero prioritizes the magnitude of the total effect, regardless of whether it arises through direct or mediated pathways.

In this way, differential subgroups provide a flexible tool for studying demographic disparities: the practitioner can tune the search toward subgroups that reveal overall disparities or toward those that isolate a stable, causally interpretable direct effect.

\subsection{Full Mediation ($\attributevar \to \featvar \to \targvar$ and $(\targvar \indep \attributevar) \mid \featvar$)}
Finally, we consider a scenario of full mediation, where there is no direct causal effect of $\attributevar$ on $\targvar$.
Often times, it is unclear whether a direct effect exists or not.
In this case, we can use differential subgroups to disprove the absence of a direct effect.
The following proposition shows that if the model of full mediation is correct, it is impossible to find a subgroup that both has no covariate dependence and exhibits an exceptional outcome. 
Therefore, if we do find such a subgroup in practice, it serves as evidence against the full mediation assumption, suggesting a direct effect likely exists.
We provide the complete proof in the Appendix \ref{sec:proof_indirect_only}.
\begin{restatable}{proposition}{propIndirect}
    \label{prop:indirect_only}
    In a model of full mediation, where $\attributevar \to \featvar \to \targvar$ and $(\targvar \indep \attributevar) \mid \featvar$,
     there cannot exist a subgroup $\subgroupf$ that simultaneously satisfies $\sufficiency(\subgroupf)=0$ and $\mathcal{E}(\subgroupf) > 0$.
\end{restatable}

\textbf{General Discussion.}
This section has established a causal foundation for differential subgroups, demonstrating that their interpretation is fundamentally tied to the assumed causal structure of the data. 
We have shown that under different, plausible causal models, the discovery of an exceptional subgroup has a distinct meaning. 
In observational studies, we can identify the local average treatment effect, while in settings with demographic attributes, we isolate the controlled direct effect.
This transforms subgroup discovery from a purely descriptive pattern-finding tool into a more powerful instrument 
for discovering causal drivers of group disparities.


\section{Discovering Differential Subgroups}
\label{sec:method}
In this section, we introduce \ourmethod, a method to discover differential subgroups given a dataset of $\nsamples$ individuals $(\feat^{(i)}, \attribute^{(i)}, \targ^{(i)})$.
We focus on tabular data with a continuous feature space $\featvar \in \R^{\nfeat}$, where we one-hot encode the categorical features.
Regarding the target variable, we present solutions for a 
discrete-valued as well as an univariate continuous target $\targvar$.

\subsection{Differentiable Rule Learning}
To capture the characteristics of a subgroup, we use a rule-based membership function $\subgroupf: \R^{\nfeat} \to [0,1]$.
In this work, we focus on logical conjunctions over single-feature conditions, as used by decision trees and regular subgroup discovery methods~\citep{atzmueller2015subgroup}.
These lead to interpretable descriptions such as "$\texttt{Age} \in (36,63) \,\&\, \texttt{Cholesterol} \in (235,500) \,\&\, \texttt{Max Heartrate} \in (156,193)$".

We encode the presence of a characteristic in a feature $j \in [\nfeat]$ as $\pred(\feat_j; \lowerbound_j, \upperbound_j)=\mathds{1}(\lowerbound_j < \feat_j < \upperbound_j)$.
The rule function $\subgroupf$ aggregates the conditions over all dimensions with parameters $\theta=\{\lowerbound_j, \upperbound_j\}_{j=1}^{\nfeat}$ using a logical conjunction
\begin{equation}
    \subgroupf(\feat;\theta) = \bigwedge_{j=1}^{\nfeat} \pred(\feat_j; \lowerbound_j, \upperbound_j)\;.
\end{equation}

We instantiate \ourmethod using the differentiable rule learner $\hat{\subgroupf}_{\temp}: \R^{\nfeat} \to [0,1]$ from the \syflow framework for subgroup discovery \citep{xu2024syflow}.
The parameters of $\hat{\subgroupf}_{\temp}$ are learned by gradient descent and hence do not require restrictive pre-processing and scale well to high-dimensional data.

We briefly summarize the key components of the differentiable rule learner $\hat{\subgroupf}_{\temp}$.
For each feature $j$, the learner places soft conditions $\hat{\pred}_{\temp}(\feat_j; \lowerbound_j, \upperbound_j)$ with learnable thresholds $\lowerbound_j, \upperbound_j \in \R$ as per
\begin{equation}
    \hat{\pred}_{\temp}(\feat_j; \lowerbound_j, \upperbound_j) = \frac{1}{1 + \exp\left(-\frac{\feat_j - \lowerbound_j}{\temp}\right)+ \exp\left(-\frac{\upperbound_j - \feat_j}{\temp}\right)}\;,
\end{equation}
where $\temp$ is a temperature parameter that controls the smoothness.
When annealing the temperature $\temp \to 0$, the condition approaches the binary thresholding function \citep{xu2024syflow}.
To combine the conditions into a logical conjunction, the weighted harmonic mean is used as 
\begin{equation}
    \label{eq:soft_rule}
    \hat{\subgroupf}_{\temp}(\feat;\theta,\andweightvec) = \frac{\sum_{j=1}^{\nfeat} \andweight_j}{\sum_{j=1}^{\nfeat} \andweight_j \cdot \hat{\pred}(\feat_j; \lowerbound_j, \upperbound_j, \temp)^{-1}} \quad \text{with} \quad \andweight_j \in \R^+_0\;.
\end{equation}
The weights $\andweight_j$ allow to exclude conditions from the rule by setting $\andweight_j=0$ and thereby simplify 
the rule's description.
Over the remaining conditions, the soft rule $\hat{\subgroupf}_{\temp}(\feat;\theta,\andweightvec)=1$ only if all conditions $\hat{\pred}_{\temp}(\feat_l;\lowerbound_l,\upperbound_l)=1$.
On the other hand, if there is a condition which is not met with $\andweight_l>0$, the soft rule tends to zero.
Hence, for the binary conditions obtained in the limit of $\temp \to 0$, $\hat{\subgroupf}$ mimics the behavior of a logical conjunction, whilst flexibly learning
which features to include ($\andweight_j$) and where to place the thresholds ($\lowerbound_j, \upperbound_j$).

\subsection{Optimization Objective}
Using the differentiable rule learner $\hat{\subgroupf}_{\temp}$, we now instantiate the objective proposed in Equation \eqref{eq:objective_general}.
For brevity, we will omit the dependence on the parameters $\theta$ and $\andweightvec$ in the following.

\textbf{Generality.}  
We interpret the output of the soft rule 
$\subgroupfhat_{\temp} : \R^{\nfeat} \to [0,1]$ 
as the probability that a sample $\feat$ belongs to the subgroup $\subgroupf$,  
i.e.~$\subgroupfhat_{\temp}(\feat) = \hat{\Pr}(\subgroupf(\feat)=1)$.  
Over a dataset $\{\feati, \attributei\}_{i=1}^{\nsamples}$, let  
$\nsamples_0 = \sum_{i=1}^{\nsamples} \mathds{1}(\attributei=0)$  
and $\nsamples_1 = \sum_{i=1}^{\nsamples} \mathds{1}(\attributei=1)$.  
Then the generality of $\subgroupfhat_{\temp}$ is computed as  
\begin{align}
    \label{eq:generality}
    \generality(\subgroupfhat_{\temp})
    &= \left( 
        \frac{1}{\nsamples_0}\sum_{\{i : \attributei=0\}} 
        \subgroupfhat_{\temp}(\feati)
        \;\cdot\;
        \frac{1}{\nsamples_1}\sum_{\{i : \attributei=1\}} 
        \subgroupfhat_{\temp}(\feati)
      \right)^{\!\gamma/2} .
\end{align}

\textbf{Exceptionality.}  
To compute the exceptionality, we need to estimate the target variable distributions 
$\hat{P}_{0,\subgroupf}(\targvar)$ 
and $\hat{P}_{1,\subgroupf}(\targvar)$.  
We weight each sample $\targi$ by the subgroup membership $\subgroupfhat_{\temp}(\feati)$ and use for 
a \emph{discrete} target $\targvar \in \{0, \dots, \nclasses\}$ the empirical distribution, and a kernel density estimator for a 
\emph{continuous} target $\targvar \in \R$.
During training, we refit the estimators after every $k=10$ updates to $\subgroupfhat_{\temp}$. 

Now that we have the target distributions, we can compute the differential exceptionality (Def.~\ref{def:exceptionality}).  
As divergence measure $D$, we use the Jensen–Shannon divergence,  
\[
D_{\text{JS}}(P, Q) = \frac{1}{2}D_{\text{KL}}(P , M) + \frac{1}{2} D_{\text{KL}}(Q , M), 
\quad 
M = \tfrac{1}{2}(P+Q).
\]
Compared to the Kullback–Leibler divergence, $D_{\text{JS}}$ seamlessly handles cases where the supports of $P$ and $Q$ do not overlap.  
Let the mixture density be 
$m(\targ) = \tfrac{1}{2}(\phat_{0,\subgroupf}(\targ) + \phat_{1,\subgroupf}(\targ))$.  
Using the approximation of the KL divergence under subgroup membership \citep{xu2024syflow},  
\begin{align*}
    D_{\text{KL}}\left(\hat{P}_{\attribute,\subgroupf}(\targvar) , M(\targvar) \right)
    &\approx 
    \frac{1}{\nsamples_a} 
    \sum_{\{i : \attributei=a\}} 
        \subgroupfhat_{\temp}(\feati)\,
        \log\!\left(
            \frac{\phat_{a,\subgroupf}(\targi)}{m(\targi)}
        \right) ,
\end{align*}
we obtain for the exceptionality of $\subgroupfhat_{\temp}$ the following approximation  
\begin{align}
    \label{eq:exceptionality}
    \mathcal{E}(\subgroupfhat_{\temp})
    &= \frac{1}{2}D_{\text{KL}}(\hat{P}_{0,\subgroupf}(\targvar) , M(\targvar)) 
     + \frac{1}{2}D_{\text{KL}}(\hat{P}_{1,\subgroupf}(\targvar) , M(\targvar)).
\end{align}

\textbf{Covariate Dependence.}
Lastly, the covariate dependence $\sufficiency(\subgroupf)$ is needed to regulate the influence of the features $\featvar$ on the target variable $\targvar$.
To that end, we also need to estimate $\hat{P}(\targvar \mid \featvar=\feat, \attributevar=\attribute)$. For a \emph{discrete} target variable, we train a random forest model $f(\feat)$ to predict
the class probabilities $\hat{\mathbb{P}}(\targvar=l \mid \featvar=\feat,\attributevar=\attribute)$. For each sample $\feati$, we compute a local divergence as
\begin{align}
     c(\feati,\attributei) =&  \sum_{l=1}^{\nclasses} \hat{\mathbb{P}}(\targvar = l \mid \featvar=\feati,\attributevar=\attributei)\\
     & \cdot \log\left(\frac{\hat{\mathbb{P}}(\targvar = l \mid \featvar=\feati,\attributevar=\attributei)}{\hat{\mathbb{P}}_{\attributei,\subgroupf}(\targvar = l)}\right)\;.
\end{align}

For a \emph{continuous} target variable, assessing the conditional distribution is more difficult.
Therefore, to bypass the need for a full conditional distribution, we compare only the first moment of the local resp.~subgroup distribution.
For a continuous target variable $\targvar$, we likewise fit a random forest to compute the squared difference between the local mean and the subgroup mean as
\begin{align}
    c(\feati,\attributei) = (f(\feati,\attributei) - E[\targvar \mid \attributevar=\attributei,\subgroup=1])^2\;.
\end{align}
Then, the covariate dependence is computed as
\begin{align}
    \label{eq:covariate_dependence}
    \sufficiency(\subgroupfhat_{\temp}) =& \sum_{i=1}^{\nsamples} \subgroupfhat_{\temp}(\feati) \cdot c(\feati,\attributei)\;.
\end{align}
Thus, we obtain a fully differentiable approach to discovering differential subgroups. 
We train \ourmethod as a single end-to-end optimization problem, where we update the parameters of the soft rule $\hat{\subgroupf}_{\temp}$ by gradient descent on the objective in Eq.~\eqref{eq:objective_general}.
In the beginning of training, we set a high temperature $\temp_{\text{start}}$ to enable smooth optimization, and gradually anneal $\temp$ to obtain a sharper, more discrete rule.
Algorithm~\ref{alg:diffsub} summarizes the full procedure.

\textbf{Multiple Subgroups.}
While our method is designed to find a single differential subgroup, it can be extended to discover multiple subgroups.
A simple approach is to iteratively re-apply \ourmethod, excluding previously found subgroups from the dataset.
More sophisticated methods for discovering multiple subgroups simultaneously are left for future work.

\begin{algorithm}[t]
\caption{\ourmethod: Discovering Differential Subgroups}
\label{alg:diffsub}
\begin{algorithmic}[1]

\Require Dataset $\mathcal{D}=\{(x_i,a_i,y_i)\}_{i=1}^n$, 
         soft rule $\subgroupfhat_{\temp}(x_i;\theta,\andweightvec)$,
         distribution estimators $\hat{P}_{0,s}(Y)$, $\hat{P}_{1,s}(Y)$,
         hyperparameters $\gamma, \lambda$, 
         epochs $K$,
         temperatures $\temp_{\text{start}},\temp_{\text{end}},\temp_{\text{step}}=\frac{\temp_{\text{start}} - \temp_{\text{end}}}{K}$

\State $\temp \gets \temp_{\text{start}}$
\For{$k = 1,\dots,K$}
    \For {$i = 1,\dots,n$}
        \State $m_i \gets \subgroupfhat_{\temp}(x_i;\theta,\andweightvec)$ \Comment{Subgroup Memberships}
    \EndFor
    \State $\hat{P}_{0,s}(Y) \gets \text{FitDistribution}(\{y_i : a_i=0\}, \{m_i : a_i=0\})$
    \State $\hat{P}_{1,s}(Y) \gets \text{FitDistribution}(\{y_i : a_i=1\}, \{m_i : a_i=1\})$
    \State $\exceptionality \gets D(\hat{P}_{0,s}(Y), \hat{P}_{1,s}(Y))$ \Comment{Exceptionality}
    \State $\generality \gets (\sum_{i=1, a_i=0}^n \frac{m_i}{n_0} )^{\gamma/2} (\sum_{i=1, a_i=1}^n \frac{m_i}{n_1})^{\gamma/2}$ \Comment{Generality}
    \State $\sufficiency \gets \frac{1}{n} \sum_{i=1}^n m_i \cdot c(x_i, a_i)$ \Comment{Covariate Dependence}
    \State $\mathcal{L} \gets \generality \cdot \exceptionality - \lambda \cdot \sufficiency$ \Comment{Loss}
    \State $(\theta, \andweightvec) \gets \text{GradientUpdate}(-\mathcal{L}, \theta, \andweightvec)$
    \State $\temp \gets \temp - \temp_{\text{step}}$ \Comment{Temperature Annealing}
\EndFor
\State $\text{Rule}\gets ""$ \Comment{Extract Rule}
\For {$j = 1,\dots,\nfeat$}
    \If{$\andweight_j > 0$}
        \State $\text{Rule} \gets \text{Rule} \;\&\;\texttt{Feature}_j  \in (\lowerbound_j, \upperbound_j)$
    \EndIf
\EndFor
\State \Return $\text{Rule}$, $\hat{P}_{0,s}(Y)$, $\hat{P}_{1,s}(Y)$
\end{algorithmic}
\end{algorithm}

\section{Related Work}
Subgroup analysis appears in several research areas, but with differing goals. 
Existing work in subgroup discovery, fairness, and treatment-effect heterogeneity addresses related questions, yet approaches them from distinct perspectives. 
Below, we situate our setting within these lines of work.

\textbf{Subgroup Discovery}
on tabular data is a well-established task \citep{atzmueller2015subgroup}.
The goal is to identify a subset of individuals that exhibit exceptional behavior, e.g.~a subgroup with a higher survival rate.
Subgroup exceptionality is quantified as the deviation from the overall population, such as differences in means \citep{lemmerich2016fast}, dispersion-aware measures \citep{boley2017identifying}, or by comparing the entire distribution \citep{xu2024syflow}.

\citet{kalofolias2017efficiently} extend subgroup discovery to a protected attribute.
They aim to minimize the effect of the sensitive attribute on the target variable by ensuring that the distribution of a sensitive attribute, e.g.~the ratio of males to females, is similar in the subgroup to the overall population.
However, they still only search for subgroups that diverge from the overall population.
Contrast set mining \citep{bay2001detecting} is another related direction, aiming to find feature characteristics that occur significantly more often in one group than another.
Unlike differential subgroups however, which analyze differences in a third variable such as survival rate, 
contrast set mining seeks to characterize the difference between the groups themselves, e.g.~the difference in features between men and women.

\textbf{Machine Learning Fairness}
is concerned with defining and quantifying disparity between a protected and an unprotected group.
Statistical parity \citep{calders2010three,kamishima2011fairness} requires that the marginal distribution of the target variable be equal across protected groups.
\citet{dwork2012fairness} instead propose individual fairness, formalizing the idea that similar individuals should be treated similarly.
However, their approach focuses on enforcing fairness through regularization rather than discovering where disparities arise.

Subgroup-based fairness frameworks likewise aim to train models that satisfy fairness constraints across many subgroups \citep{kearns2018preventing,martinez2021blind,shui2022learning}.
Similarly, multicalibration and multiaccuracy methods \citep{hebert2018multicalibration,kim2019multiaccuracy} adjust predictors to ensure accuracy or calibration across a pre-defined family of subgroups.

A complementary line of work, discrimination discovery, detects potential discrimination at the individual level by comparing each decision to matched counterparts \citep{zhang2016situation,zhang2018causal}.
These approaches identify whether discrimination may have occurred but do not yield interpretable subgroup descriptions.
Slicing methods \citep{chung2019slice,sagadeeva2021sliceline} discover subgroups with high error rates, but—like classical subgroup discovery—measure deviations relative to the overall population rather than contrasting outcome distributions between two groups.

\begin{figure*}[t!]
    \begin{subfigure}[t]{0.19\linewidth}
        \centering
        \begin{tikzpicture}
            \begin{axis}[
        pretty line,
        width=\linewidth,
        height=3.5cm,
        ylabel={$F_1$ score},
        xlabel={Subgroup strength $\tau$},
        ymin=0,
        ymax = 1, smooth ]

        \addplot+[dollarbill,
        error bars/.cd,
        y dir=both, y explicit]
        table [x index=0, y index=4, y error index=9, col sep=comma] {expres/synthetic_observational_F1.csv};

        \addplot+[gray,
        error bars/.cd,
        y dir=both, y explicit]
        table [x index=0, y index=1, y error index=6, col sep=comma] {expres/synthetic_observational_F1.csv};

        \addplot+[pink,
        error bars/.cd,
        y dir=both, y explicit]
        table [x index=0, y index=2, y error index=7, col sep=comma] {expres/synthetic_observational_F1.csv};

        \addplot+[teal,
        error bars/.cd,
        y dir=both, y explicit]
        table [x index=0, y index=3, y error index=8, col sep=comma] {expres/synthetic_observational_F1.csv};

        \addplot+[orange,
        error bars/.cd,
        y dir=both, y explicit]
        table [x index=0, y index=5, y error index=10, col sep=comma] {expres/synthetic_observational_F1.csv};
        \end{axis}
        \end{tikzpicture}
        \caption{Observational study}
        \label{fig:experiment_observational_F1}
        \end{subfigure}
    \hfill
    \begin{subfigure}[t]{0.19\linewidth}
        \centering
        \begin{tikzpicture}
            \begin{axis}[
        pretty line,
        width=\linewidth,
        height=3.5cm,
        xlabel={Subgroup strength $\tau$},
        ymin=0,
        ymax = 1, smooth ]
        \addplot+[dollarbill,
        line width=1pt,
        error bars/.cd,
        y dir=both, y explicit]
        table [x index=0, y index=4, y error index=9, col sep=comma] {expres/synthetic_interventional_F1.csv};

        \addplot+[gray,
        error bars/.cd,
        y dir=both, y explicit]
        table [x index=0, y index=1, y error index=6, col sep=comma] {expres/synthetic_interventional_F1.csv};

        \addplot+[pink,
        error bars/.cd,
        y dir=both, y explicit]
        table [x index=0, y index=2, y error index=7, col sep=comma] {expres/synthetic_interventional_F1.csv};

        \addplot+[teal,
        error bars/.cd,
        y dir=both, y explicit]
        table [x index=0, y index=3, y error index=8, col sep=comma] {expres/synthetic_interventional_F1.csv};

        \addplot+[orange,
        error bars/.cd,
        y dir=both, y explicit]
        table [x index=0, y index=5, y error index=10, col sep=comma] {expres/synthetic_interventional_F1.csv};
        \end{axis}
        \end{tikzpicture}
        \caption{Randomized trial}
        \label{fig:experiment_randomized_F1}
        \end{subfigure}
    \hfill
    \begin{subfigure}[t]{0.19\linewidth}
        \centering
        \begin{tikzpicture}
            \begin{axis}[
        pretty line,
        width=\linewidth,
        height=3.5cm,
        xlabel={Subgroup strength $\tau$},
        legend style={/tikz/overlay,
        legend columns=5,
        at={(0.5,1.18)}, 
        anchor=south,   
        font=\scriptsize,
        },
        legend entries={\ourmethod, \uplifttree,\causaltree,\pysubgroup,\syflow},
        ymin=0,
        ymax = 1, smooth ]

        \addplot+[dollarbill,
        error bars/.cd,
        y dir=both, y explicit]
        table [x index=0, y index=4, y error index=9, col sep=comma] {expres/synthetic_demographic_F1.csv};

        \addplot+[gray,
        error bars/.cd,
        y dir=both, y explicit]
        table [x index=0, y index=1, y error index=6, col sep=comma] {expres/synthetic_demographic_F1.csv};

        \addplot+[pink,
        error bars/.cd,
        y dir=both, y explicit]
        table [x index=0, y index=2, y error index=7, col sep=comma] {expres/synthetic_demographic_F1.csv};

        \addplot+[teal,
        error bars/.cd,
        y dir=both, y explicit]
        table [x index=0, y index=3, y error index=8, col sep=comma] {expres/synthetic_demographic_F1.csv};

        \addplot+[orange,
        error bars/.cd,
        y dir=both, y explicit]
        table [x index=0, y index=5, y error index=10, col sep=comma] {expres/synthetic_demographic_F1.csv};
        \end{axis}
        \end{tikzpicture}
        \caption{Demographic groups}
        \label{fig:experiment_demographic_F1}
        \end{subfigure}
    \hfill
    \begin{subfigure}[t]{0.19\linewidth}
        \centering
        \begin{tikzpicture}
            \begin{axis}[
        pretty line,
        width=\linewidth,
        height=3.5cm,
        xlabel={Variables $d$},
        ymin=0,
        ymax = 1, smooth ]

        \addplot+[dollarbill,
        error bars/.cd,
        y dir=both, y explicit]
        table [x index=0, y index=4, y error index=9, col sep=comma] {expres/noise_n_vars_F1.csv};

        \addplot+[gray,
        error bars/.cd,
        y dir=both, y explicit]
        table [x index=0, y index=1, y error index=6, col sep=comma] {expres/noise_n_vars_F1.csv};

        \addplot+[pink,
        error bars/.cd,
        y dir=both, y explicit]
        table [x index=0, y index=2, y error index=7, col sep=comma] {expres/noise_n_vars_F1.csv};

        \addplot+[teal,
        error bars/.cd,
        y dir=both, y explicit]
        table [x index=0, y index=3, y error index=8, col sep=comma] {expres/noise_n_vars_F1.csv};

        \addplot+[orange,
        error bars/.cd,
        y dir=both, y explicit]
        table [x index=0, y index=5, y error index=10, col sep=comma] {expres/noise_n_vars_F1.csv};
        \end{axis}
        \end{tikzpicture}
        \caption{Scalability: dimensionality $d$}
        \label{fig:experiment_noise_vars_F1}
        \end{subfigure}
    \hfill
    \begin{subfigure}[t]{0.19\linewidth}
        \centering
        \begin{tikzpicture}
            \begin{axis}[
        pretty line,
        width=\linewidth,
        height=3.5cm,
        xlabel={Number of samples $n$},
        xtick={1000,5000,10000},
        xticklabels={1k,5k,10k},
        ymin=0,
        ymax = 1, smooth ]

        \addplot+[dollarbill,
        error bars/.cd,
        y dir=both, y explicit]
        table [x index=0, y index=4, y error index=9, col sep=comma] {expres/scalability_n_samples_F1.csv};

        \addplot+[gray,
        error bars/.cd,
        y dir=both, y explicit]
        table [x index=0, y index=1, y error index=6, col sep=comma] {expres/scalability_n_samples_F1.csv};

        \addplot+[pink,
        error bars/.cd,
        y dir=both, y explicit]
        table [x index=0, y index=2, y error index=7, col sep=comma] {expres/scalability_n_samples_F1.csv};

        \addplot+[teal,
        error bars/.cd,
        y dir=both, y explicit]
        table [x index=0, y index=3, y error index=8, col sep=comma] {expres/scalability_n_samples_F1.csv};

        \addplot+[orange,
        error bars/.cd,
        y dir=both, y explicit]
        table [x index=0, y index=5, y error index=10, col sep=comma] {expres/scalability_n_samples_F1.csv};
        \end{axis}
        \end{tikzpicture}
        \caption{Scalability: samples $n$}
        \label{fig:experiment_sample_complexity_F1}
        \end{subfigure}
        \caption{$F_1$-score of discovered subgroup compared to known-ground truth in the synthetic benchmark data. \ourmethod differential subgroup approach is most accurate across all settings (a-c).
        In terms of scalability, \ourmethod handles both increasing dimensionality (d) and sample size (e) effectively. 
    }
    \label{fig:experiment_synthetic}
\end{figure*}
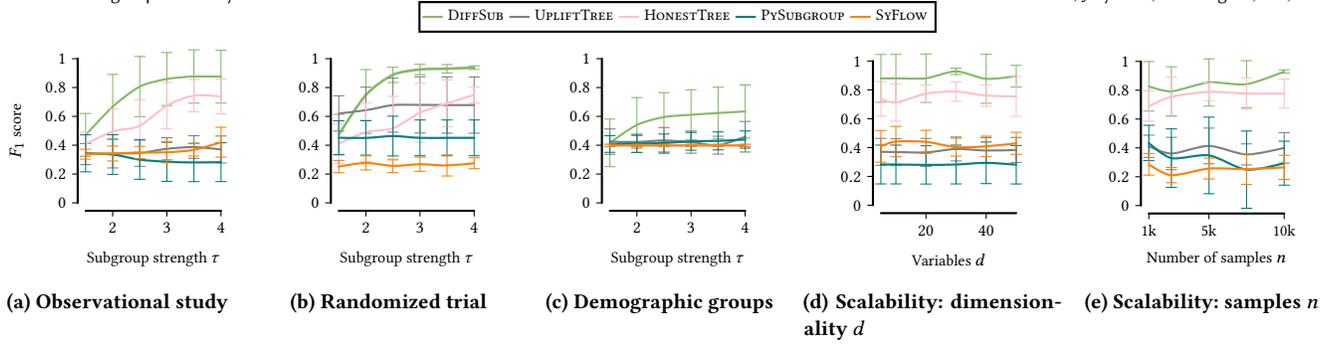

\textbf{Treatment-effect estimation} also faces the challenge of identifying meaningful subgroups.
In randomized trials, subgroup analysis methods include regression trees \citep{lipkovich2011subgroup}, trees with local linear models \citep{seibold2016model}, and Lasso-based approaches \citep{ballarini2018subgroup}.
In observational studies, the main difficulty lies in estimating heterogeneous treatment effects under confounding, typically addressed using flexible machine-learning models \citep{kunzel2019metalearners,shalit2017estimating}.
\citet{chikahara2022feature} consider distributional treatment effects but focus on selecting important features rather than identifying feature combinations defining meaningful subgroups.
Honest causal trees \citep{athey2016recursive} and causal forests \citep{wager2018estimation} provide unbiased estimation of treatment effects at the leaf level, though they rely on restrictive causal assumptions and do not explicitly target subgroup interpretability.

In general, differential subgroups are related to, yet distinct from, existing ideas in subgroup discovery, fairness, and treatment-effect estimation.
Discovering such subgroups offers a structured way to understand how population differences arise within localized regions of the feature space.

\section{Experiments}

We compare \ourmethod against regular subgroup discovery using the \pysubgroup package \citep{lemmerich2018pysubgroup} as well as \syflow \citep{xu2024syflow} using 
the authors implementation.
We also compare against \uplifttree, which implement recursive partitioning for randomized trials using the
causalml package \citep{causalml}, and \causaltree \citep{athey2016recursive}, which focuses on observational studies using econml \citep{econml}.
We provide the source code for all experiments in an anonymized repository \footnote{\url{https://anonymous.4open.science/r/diffsub-kdd-supplement-07EE/}}.

\subsection{Simulated Data}
We first evaluate \ourmethod on simulated data with known causal structure, following the observational, randomized, and demographic settings introduced in Figure~\ref{fig:causal_structure}.
The target variable is generated as $\targvar = f(\featvar,\attributevar) + N_{\targvar}$ with Gaussian noise, where $\featvar$ may be a cause or effect of $\attributevar$. 
We embed a ground-truth subgroup $s$ defined by two randomly selected features and amplify the treatment effect inside the subgroup by a parameter $\tau$ (details in Appendix~\ref{app:simulated_data}).  
Hyperparameters are tuned via grid search (Appendix~\ref{app:hyperparameters}).  
Performance is evaluated using the $F_1$-score between the true and recovered subgroup; for tree-based baselines, we report the best leaf.

Figures~\ref{fig:experiment_observational_F1}, \ref{fig:experiment_randomized_F1}, and \ref{fig:experiment_demographic_F1} show $F_1$-scores as $\tau$ increases for the three settings.
Classical subgroup discovery represented by \pysubgroup and \syflow fails due to its non-differential objective.
\uplifttree excels in randomized trials but struggles in observational ones, while \causaltree shows the opposite pattern. 
Both struggle in the demographic setting, where indirect effects distort the treatment contrast.  
Across all conditions, \ourmethod consistently identifies the correct subgroup and improves as the effect size $\tau$ grows.
By directly searching for differential subgroups, \ourmethod outperforms specialized baselines across all settings.

\textbf{Scalability.}
We further examine robustness to noise and sample size by varying the number of irrelevant features $\nfeat$ and the number of samples $\nsamples$ (Figures~\ref{fig:experiment_noise_vars_F1} and \ref{fig:experiment_sample_complexity_F1}), using the observational setting with fixed effect size $\tau=4$.
\ourmethod remains stable as noisy features are added, indicating that it effectively separates signal from noise.
As $\nsamples$ increases, its $F_1$-score improves and confidence intervals tighten, showing that additional data translates into more precise subgroup recovery.
Runtime grows linearly in $\nfeat$ but quadratically in $\nsamples$ due to KDE-based density estimation; for reference, a run with $\nfeat=5$ and $\nsamples=10{,}000$ takes roughly two minutes (Appendix~\ref{app:runtimes}).

We further conduct a sensitivity analysis on the hyperparameters $\lambda$ and $\gamma$ as well as the choice of estimator in Appendix~\ref{app:sensitivity}.
The results for $\gamma$ show that performance is stable in a range of $[0.1,0.3]$, while large values or disabling the generality ($\gamma=0$) degrade performance.
For $\lambda$, we find that the performance is stable up to $\lambda=0.5$ and then degrades for larger values, where all trends persist across the different settings.

\begin{figure}[t]
    \centering
        \begin{subfigure}[b]{0.45\linewidth}
            \centering
            \includegraphics[width=\linewidth]{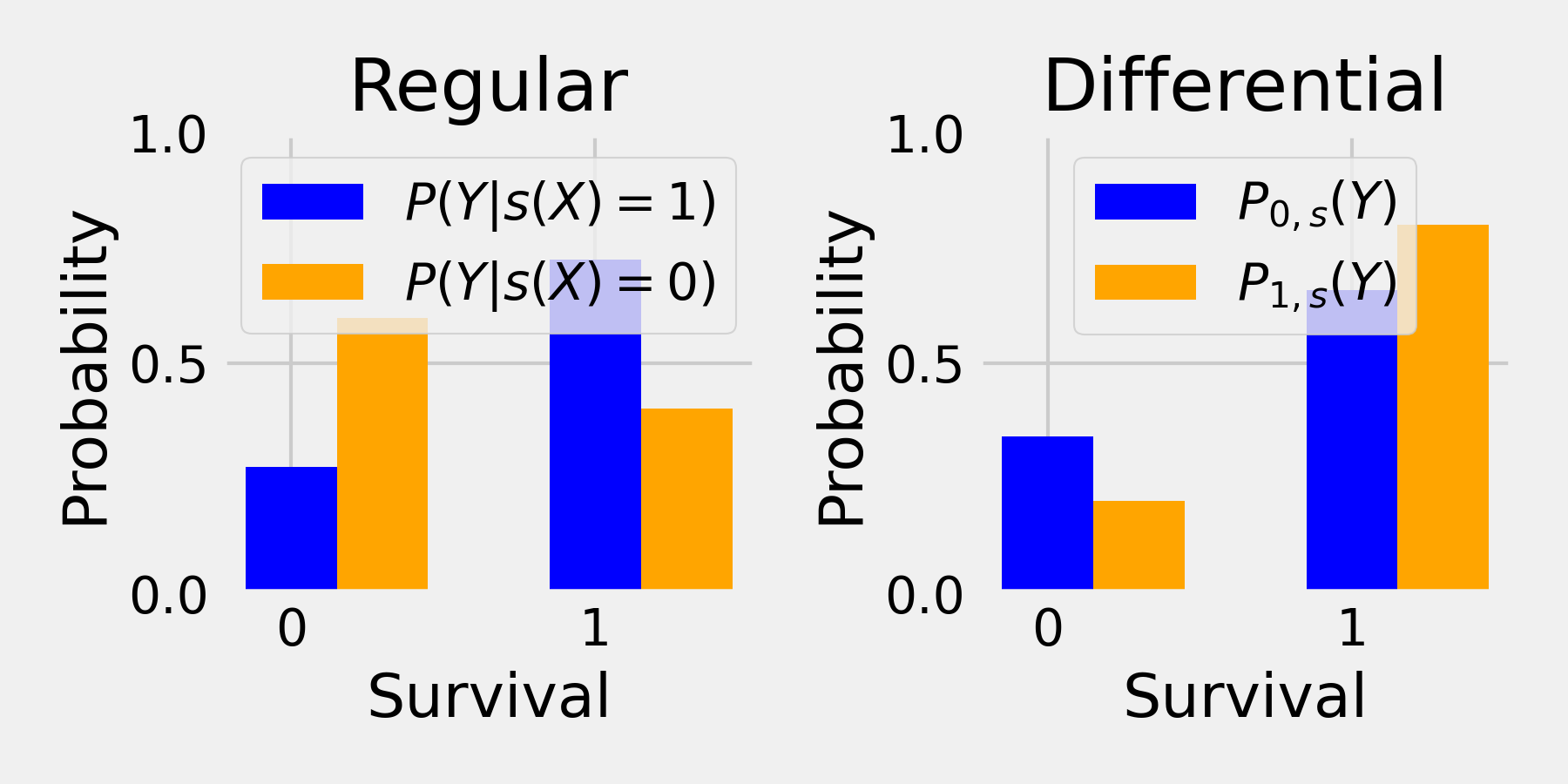}
            \caption{\pysubgroup subgroup}
            \label{fig:covid_ps}
        \end{subfigure}
        \hfill
        \begin{subfigure}[b]{0.45\linewidth}
            \centering
            \includegraphics[width=\linewidth]{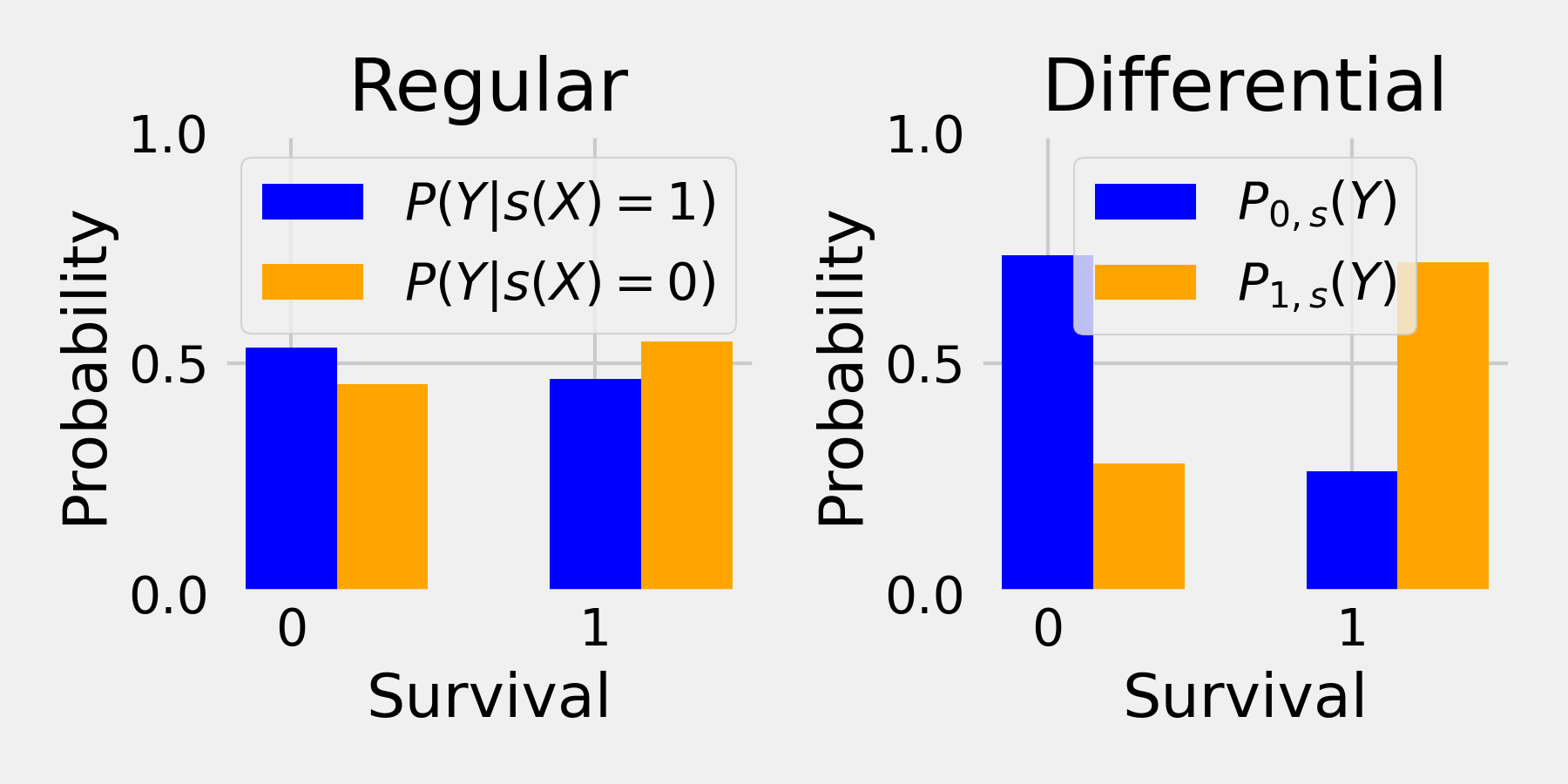}
            \caption{\ourmethod subgroup}
            \label{fig:covid_max}
        \end{subfigure}
        \caption{Subgroup discovery on COVID-19 dataset.}
        \label{fig:covid_example}
    \end{figure}

\subsection{Real-World Data}
Lastly, we qualitatively evaluate \ourmethod on downstream applications in medicine, treatment 
effects, and model error analysis.

\textbf{Medical Subgroups.}
We compare differential versus standard subgroup discovery on the COVID-19 dataset from \citet{lambert2022using}, containing biomarkers, comorbidities, and ICU outcomes from two New York City hospitals.  
We define gender as the binary attribute and seek subgroups exhibiting differential mortality between men and women.
We apply \pysubgroup to discover subgroups that deviate from the overall population and contrast these with the differential subgroups identified by \ourmethod.

Figure~\ref{fig:covid_example} summarizes the findings: the left panels show mortality inside versus outside the subgroup (standard), while the right panels compare male and female mortality within the subgroup (differential).  
The top subgroup discovered by \pysubgroup is
\[
\texttt{Age}<74 \wedge \texttt{No Coronary Artery/Cerebrovascular Dis.},
\]
corresponding to younger patients without severe comorbidities.  
These individuals exhibit reduced overall mortality but little gender disparity.
In contrast, \ourmethod discovers the subgroup  
\[
\texttt{Race}\neq\texttt{black} \wedge \texttt{Diabetes} \wedge 
\texttt{No Cerebrovascular/Hypertension}
\]
where male mortality substantially exceeds female mortality, despite not being high relative to the overall population, therewith not identifiable by standard subgroup discovery.
This pattern is consistent with prior clinical evidence: men with diabetes have been reported to experience higher COVID-19 mortality than diabetic women \citep{belice2020gender}, 
while Black women have shown disproportionately high COVID-19 mortality overall \citep{rushovich2021sex}. 
Note that these subgroups do not suffice to establish causal effect; rather, they provide hypothesis-generating patterns that warrant further clinical investigation.

\textbf{Treatment-Effects.}
We use the IHDP semi-synthetic benchmark from \citet{hill2011bayesian} with the 10 replications provided by \citet{louizos2017causal}.  
Baselines include \causaltree \citep{athey2016recursive}, \causalforest \citep{wager2018estimation}, and the \xlearner \citep{kunzel2019metalearners} with gradient-boosted trees.  

For \ourmethod, we estimate the treatment effect inside a discovered subgroup using the empirical difference in means.  
Proposition~\ref{prop:causal_effect} implies that whenever $\sufficiency(\subgroupf)=0$, the subgroup distribution equals the interventional distribution.  
Then, the treatment effect in a subgroup $s$ is given by
\begin{equation}
\hat{\tau}_{s} 
= \mathbb{E}[Y \mid A=1,\, s(X)=1]
- \mathbb{E}[Y \mid A=0,\, s(X)=1].
\end{equation}

\begin{wrapfigure}{l}{0.5\linewidth}
\centering
\begin{tikzpicture}
\begin{axis}[
    pretty ybar small,
    pretty labelshift,
    pretty enlargexlimits,
    width=5cm,
    height=2.8cm,
    ylabel={PEHE},
    ymin=0,
    symbolic x coords={subcon,ct,cf,xl,ablation},
    xtick={subcon,ct,cf,xl,ablation},
    xticklabels={\ourmethod, \causaltree, \causalforest, \xlearner,\ourmethod ($\lambda=0$)},
    x tick label style={rotate=25, anchor=east},
]
\addplot+[ bar width=7pt, bar shift=0pt, dollarbill]
    coordinates {(subcon,1.73)};

\addplot+[ bar width=7pt, bar shift=0pt, gray]
    coordinates {(ablation,3.77)};

\addplot+[ bar width=7pt, bar shift=0pt, pink]
    coordinates {(ct,2.33)};

\addplot+[ bar width=7pt, bar shift=0pt, teal]
    coordinates {(cf,2.17)};

\addplot+[ bar width=7pt, bar shift=0pt, orange]
    coordinates {(xl,1.19)};

\end{axis}
\end{tikzpicture}
\caption{IHDP in-subgroup PEHE (lower is better).}
\label{fig:ihdp_pehe_bar}
\end{wrapfigure}
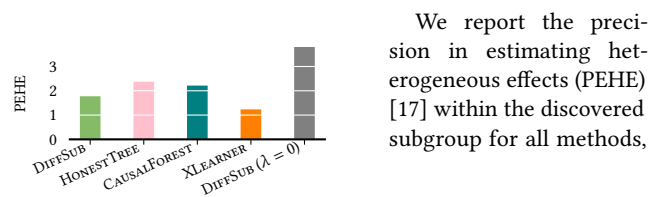

We report the precision in estimating heterogeneous effects (PEHE) \citep{hill2011bayesian} within the discovered subgroup for all 
methods, and further include an ablation without covariate regularization ($\lambda=0$).
We report the results in Figure~\ref{fig:ihdp_pehe_bar}.  
\ourmethod achieves a PEHE of $1.73$, outperforming the tree-based methods \causaltree ($2.33$) and \causalforest ($2.17$), while remaining less accurate than the black-box \xlearner ($1.19$).  
The ablation performs substantially worse ($3.77$), confirming that controlling covariate dependence is essential for identifying subgroups with faithful treatment-effect estimates.

\textbf{Model Errors.}
Lastly, we investigate differential subgroups as a tool for model error analysis.  
We train a Logistic Regression and a Gradient Boosting classifier on the Adult dataset \citep{adultdataset} and search for subgroups in which the models exhibit different error rates.

\definecolor{tabblue}{HTML}{1F77B4}
\definecolor{taborange}{HTML}{FF7F0E}

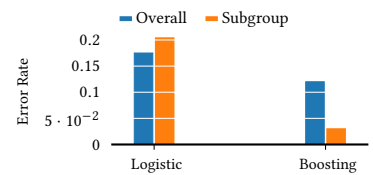
\begin{wrapfigure}{r}{0.5\linewidth}
\centering
\begin{tikzpicture}
\pgfplotsset{every axis/.style={tick style={black}, ticklabel style={font=\small}}}
\begin{axis}[
    width=5cm,
    height=3cm,
    ylabel={Error Rate},
    ymin=0,
    symbolic x coords={Logistic,Boosting},
    xtick=data,
    pretty ybar,
    pretty labelshift,
    legend columns=2,
    legend style={/tikz/every even column/.append style={column sep=0.6em}, draw=none, font=\small,
    at={(0.02,1.18)}, anchor=west},
    enlarge x limits=0.25,
]

\addplot+[ybar, bar width=7pt, bar shift=-4pt, color=tabblue] coordinates {(Logistic,0.175) (Boosting,0.12)};
\addlegendentry{Overall}
\addplot+[ybar, bar width=7pt, bar shift=4pt, color=taborange] coordinates {(Logistic,0.204) (Boosting,0.03)};
\addlegendentry{Subgroup}

\end{axis}
\end{tikzpicture}

\caption{Differential subgroup on Adult: $\texttt{Education} \neq \texttt{PhD} \wedge\ \texttt{Capital-Gains}  \in (2609,95725)$.}
\label{fig:error-rate-bars}
\end{wrapfigure}

Figure~\ref{fig:error-rate-bars} shows the subgroup discovered by \ourmethod.
For individuals without a PhD and with medium to high capital gains, Logistic Regression misclassifies substantially more often than Gradient Boosting (error rate $20\%$ vs.\ $3\%$).  
Notably, this subgroup contains a much larger share of high-income individuals (two-thirds) compared to the overall population (one-quarter), indicating that the linear model struggles precisely where income is strongly shifted relative to the global distribution.
This demonstrates the value of differential subgroups for model auditing: 
they reveal where two models diverge and therewith explain for which populations a simpler model may be unreliable.


\section{Conclusion}
We introduce differential subgroups, subsets of two populations that show exceptional disparities in a target variable.  
They are defined by three key properties: support in both populations, a significant outcome difference, and minimal covariate influence.
Our method \ourmethod, discovers such subgroups by jointly optimizing these criteria through a using efficient gradient-based search over interpretable rule-based descriptions.
We showed theoretically that differential subgroups admit a causal interpretation under certain assumptions for observational studies, randomized trials, and demographic comparisons.
In experiments, \ourmethod consistently outperforms existing approaches on synthetic benchmarks and recovers practically relevant subgroups in real-world datasets.  

\textbf{Limitations.}
\ourmethod relies on the assumption that the observed covariates $\featvar$ adequately capture the relevant causal structure. 
If important confounders are unobserved or if post-treatment variables are inadvertently included, the causal interpretation of discovered subgroups may be compromised.
Careful, domain-informed feature selection therefore remains essential.  
The covariate-dependence regularization depends on accurate estimation of $P(\targvar \mid \featvar, \attributevar)$, and using simpler, less sensitive estimators,
such as mean-based regularization for continuous targets, may reduce subgroup quality. 
As of now, the current framework identifies a single differential subgroup. 
Developing more principled multi-subgroup extensions is an important direction for future work.
Overall, the discovered subgroups should be interpreted as hypotheses rather than definitive causal claims, and require external validation or domain expertise for confirmation.
\clearpage
\bibliographystyle{ACM-Reference-Format}
\bibliography{bib/paper}

@article{atzmueller2015subgroup,
	title        = {Subgroup discovery},
	author       = {Atzmueller, Martin},
	year         = 2015,
	journal      = {Wiley Interdisciplinary Reviews: Data Mining and Knowledge Discovery},
	publisher    = {Wiley Online Library},
	volume       = 5,
	number       = 1,
	pages        = {35--49}
}

@inproceedings{zhang2016situation,
	title        = {Situation Testing-Based Discrimination Discovery: A Causal Inference Approach.},
	author       = {Zhang, Lu and Wu, Yongkai and Wu, Xintao},
	year         = 2016,
	booktitle    = {IJCAI},
	volume       = 16,
	pages        = {2718--2724}
}

@article{zhang2018causal,
	title        = {Causal modeling-based discrimination discovery and removal: Criteria, bounds, and algorithms},
	author       = {Zhang, Lu and Wu, Yongkai and Wu, Xintao},
	year         = 2018,
	journal      = {IEEE Transactions on Knowledge and Data Engineering},
	publisher    = {IEEE},
	volume       = 31,
	number       = 11,
	pages        = {2035--2050}
}

@inproceedings{kamishima2011fairness,
	title        = {Fairness-aware learning through regularization approach},
	author       = {Kamishima, Toshihiro and Akaho, Shotaro and Sakuma, Jun},
	booktitle    = {2011 IEEE 11th international conference on data mining workshops},
	pages        = {643--650},
  year     = 2011,
	organization = {IEEE}
}

@inproceedings{chikahara2022feature,
	title        = {Feature selection for discovering distributional treatment effect modifiers},
	author       = {Chikahara, Yoichi and Yamada, Makoto and Kashima, Hisashi},
	year         = 2022,
	booktitle    = {Uncertainty in Artificial Intelligence},
	pages        = {400--410},
	organization = {PMLR}
}

@article{detrano1989international,
	title        = {International application of a new probability algorithm for the diagnosis of coronary artery disease},
	author       = {Detrano, Robert and Janosi, Andras and Steinbrunn, Walter and Pfisterer, Matthias and Schmid, Johann-Jakob and Sandhu, Sarbjit and Guppy, Kern H and Lee, Stella and Froelicher, Victor},
	year         = 1989,
	journal      = {The American journal of cardiology},
	publisher    = {Elsevier},
	volume       = 64,
	number       = 5,
	pages        = {304--310}
}

@article{wilson1998prediction,
	title        = {Prediction of coronary heart disease using risk factor categories},
	author       = {Wilson, Peter WF and D'Agostino, Ralph B and Levy, Daniel and Belanger, Albert M and Silbershatz, Halit and Kannel, William B},
	year         = 1998,
	journal      = {Circulation},
	publisher    = {Lippincott Williams \& Wilkins},
	volume       = 97,
	number       = 18,
	pages        = {1837--1847}
}

@article{lang2010elevated,
	title        = {Elevated heart rate and cardiovascular outcomes in patients with coronary artery disease: clinical evidence and pathophysiological mechanisms},
	author       = {Lang, Chim C and Gupta, Sandeep and Kalra, Paul and Keavney, Bernard and Menown, Ian and Morley, Chris and Padmanabhan, Sandosh},
	year         = 2010,
	journal      = {Atherosclerosis},
	publisher    = {Elsevier},
	volume       = 212,
	number       = 1,
	pages        = {1--8}
}

@article{mehrabi2021survey,
	title        = {A survey on bias and fairness in machine learning},
	author       = {Mehrabi, Ninareh and Morstatter, Fred and Saxena, Nripsuta and Lerman, Kristina and Galstyan, Aram},
	year         = 2021,
	journal      = {ACM computing surveys (CSUR)},
	publisher    = {ACM New York, NY, USA},
	volume       = 54,
	number       = 6,
	pages        = {1--35}
}

@inproceedings{xu2024syflow,
	title        = {Learning Exceptional Subgroups by End-to-End Maximizing KL-Divergence},
	author       = {Xu, Sascha and Walter, Nils Philipp and Kalofolias, Janis and Vreeken, Jilles},
	year         = 2024,
	booktitle    = {International Conference on Machine Learning},
	pages        = {55267--55285},
	organization = {PMLR}
}

@article{bay2001detecting,
	title        = {Detecting group differences: Mining contrast sets},
	author       = {Bay, Stephen D and Pazzani, Michael J},
	year         = 2001,
	journal      = {Data mining and knowledge discovery},
	publisher    = {Springer},
	volume       = 5,
	number       = 3,
	pages        = {213--246}
}

@inproceedings{kalofolias2017efficiently,
	title        = {Efficiently discovering locally exceptional yet globally representative subgroups},
	author       = {Kalofolias, Janis and Boley, Mario and Vreeken, Jilles},
	year         = 2017,
	booktitle    = {2017 IEEE International Conference on Data Mining (ICDM)},
	pages        = {197--206},
	organization = {IEEE}
}

@article{boley2017identifying,
	title        = {Identifying consistent statements about numerical data with dispersion-corrected subgroup discovery},
	author       = {Boley, Mario and Goldsmith, Bryan R and Ghiringhelli, Luca M and Vreeken, Jilles},
	year         = 2017,
	journal      = {Data Mining and Knowledge Discovery},
	publisher    = {Springer},
	volume       = 31,
	number       = 5,
	pages        = {1391--1418}
}

@article{lemmerich2016fast,
	title        = {Fast exhaustive subgroup discovery with numerical target concepts},
	author       = {Lemmerich, Florian and Atzmueller, Martin and Puppe, Frank},
	year         = 2016,
	journal      = {Data Mining and Knowledge Discovery},
	publisher    = {Springer},
	volume       = 30,
	number       = 3,
	pages        = {711--762}
}

@article{athey2016recursive,
	title        = {Recursive partitioning for heterogeneous causal effects},
	author       = {Athey, Susan and Imbens, Guido},
	year         = 2016,
	journal      = {Proceedings of the National Academy of Sciences},
	publisher    = {National Academy of Sciences},
	volume       = 113,
	number       = 27,
	pages        = {7353--7360}
}

@article{wager2018estimation,
	title        = {Estimation and inference of heterogeneous treatment effects using random forests},
	author       = {Wager, Stefan and Athey, Susan},
	year         = 2018,
	journal      = {Journal of the American Statistical Association},
	publisher    = {Taylor \& Francis},
	volume       = 113,
	number       = 523,
	pages        = {1228--1242}
}

@article{foster2011subgroup,
	title        = {Subgroup identification from randomized clinical trial data},
	author       = {Foster, Jared C and Taylor, Jeremy MG and Ruberg, Stephen J},
	year         = 2011,
	journal      = {Statistics in medicine},
	publisher    = {Wiley Online Library},
	volume       = 30,
	number       = 24,
	pages        = {2867--2880}
}

@article{lipkovich2011subgroup,
	title        = {Subgroup identification based on differential effect search--a recursive partitioning method for establishing response to treatment in patient subpopulations},
	author       = {Lipkovich, Ilya and Dmitrienko, Alex and Denne, Jonathan and Enas, Gregory},
	year         = 2011,
	journal      = {Statistics in medicine},
	publisher    = {Wiley Online Library},
	volume       = 30,
	number       = 21,
	pages        = {2601--2621}
}

@article{seibold2016model,
	title        = {Model-based recursive partitioning for subgroup analyses},
	author       = {Seibold, Heidi and Zeileis, Achim and Hothorn, Torsten},
	year         = 2016,
	journal      = {The international journal of biostatistics},
	publisher    = {De Gruyter},
	volume       = 12,
	number       = 1,
	pages        = {45--63}
}

@article{ballarini2018subgroup,
	title        = {Subgroup identification in clinical trials via the predicted individual treatment effect},
	author       = {Ballarini, Nicol{\'a}s M and Rosenkranz, Gerd K and Jaki, Thomas and K{\"o}nig, Franz and Posch, Martin},
	year         = 2018,
	journal      = {PloS one},
	publisher    = {Public Library of Science San Francisco, CA USA},
	volume       = 13,
	number       = 10,
	pages        = {e0205971}
}

@inproceedings{kearns2018preventing,
	title        = {Preventing fairness gerrymandering: Auditing and learning for subgroup fairness},
	author       = {Kearns, Michael and Neel, Seth and Roth, Aaron and Wu, Zhiwei Steven},
	year         = 2018,
	booktitle    = {International conference on machine learning},
	pages        = {2564--2572},
	organization = {PMLR}
}

@article{shui2022learning,
	title        = {On learning fairness and accuracy on multiple subgroups},
	author       = {Shui, Changjian and Xu, Gezheng and Chen, Qi and Li, Jiaqi and Ling, Charles X and Arbel, Tal and Wang, Boyu and Gagn{\'e}, Christian},
	year         = 2022,
	journal      = {Advances in Neural Information Processing Systems},
	volume       = 35,
	pages        = {34121--34135}
}

@inproceedings{dwork2012fairness,
	title        = {Fairness through awareness},
	author       = {Dwork, Cynthia and Hardt, Moritz and Pitassi, Toniann and Reingold, Omer and Zemel, Richard},
	year         = 2012,
	booktitle    = {Proceedings of the 3rd innovations in theoretical computer science conference},
	pages        = {214--226}
}

@article{calders2010three,
	title        = {Three naive bayes approaches for discrimination-free classification},
	author       = {Calders, Toon and Verwer, Sicco},
	year         = 2010,
	journal      = {Data mining and knowledge discovery},
	publisher    = {Springer},
	volume       = 21,
	number       = 2,
	pages        = {277--292}
}

@inproceedings{shalit2017estimating,
	title        = {Estimating individual treatment effect: generalization bounds and algorithms},
	author       = {Shalit, Uri and Johansson, Fredrik D and Sontag, David},
	year         = 2017,
	booktitle    = {International conference on machine learning},
	pages        = {3076--3085},
	organization = {PMLR}
}

@article{kunzel2019metalearners,
	title        = {Metalearners for estimating heterogeneous treatment effects using machine learning},
	author       = {K{\"u}nzel, S{\"o}ren R and Sekhon, Jasjeet S and Bickel, Peter J and Yu, Bin},
	year         = 2019,
	journal      = {Proceedings of the national academy of sciences},
	publisher    = {National Academy of Sciences},
	volume       = 116,
	number       = 10,
	pages        = {4156--4165}
}

@inproceedings{martinez2021blind,
	title        = {Blind pareto fairness and subgroup robustness},
	author       = {Martinez, Natalia L and Bertran, Martin A and Papadaki, Afroditi and Rodrigues, Miguel and Sapiro, Guillermo},
	year         = 2021,
	booktitle    = {International Conference on Machine Learning},
	pages        = {7492--7501},
	organization = {PMLR}
}

@book{pearl2009causality,
	title        = {Causality},
	author       = {Pearl, Judea},
	year         = 2009,
	publisher    = {Cambridge University Press},
	place        = {Cambridge},
	edition      = 2
}

@incollection{pearl2022direct,
	title        = {Direct and indirect effects},
	author       = {Pearl, Judea},
	year         = 2022,
	booktitle    = {Probabilistic and causal inference: the works of Judea Pearl},
	pages        = {373--392}
}

@article{terrell1992variable,
	title        = {Variable kernel density estimation},
	author       = {Terrell, George R and Scott, David W},
	year         = 1992,
	journal      = {The Annals of Statistics},
	publisher    = {JSTOR},
	pages        = {1236--1265}
}

@inproceedings{lemmerich2018pysubgroup,
	title        = {pysubgroup: Easy-to-use subgroup discovery in python},
	author       = {Lemmerich, Florian and Becker, Martin},
	year         = 2018,
	booktitle    = {Joint European conference on machine learning and knowledge discovery in databases},
	pages        = {658--662},
	organization = {Springer}
}

@article{causalml,
	title        = {Causalml: Python package for causal machine learning},
	author       = {Chen, Huigang and Harinen, Totte and Lee, Jeong-Yoon and Yung, Mike and Zhao, Zhenyu},
	year         = 2020,
	journal      = {arXiv preprint arXiv:2002.11631}
}

@inproceedings{econml,
  title={EconML: A machine learning library for estimating heterogeneous treatment effects},
  author={Oprescu, Miruna and Syrgkanis, Vasilis and Battocchi, Keith and Hei, Maggie and Lewis, Greg},
  booktitle={33rd Conference on Neural Information Processing Systems},
  volume={6},
  year={2019},
  organization={Curran Associates, Inc}
}

@article{cook2004subgroup,
	title        = {Subgroup analysis in clinical trials},
	author       = {Cook, David I and Gebski, Val J and Keech, Anthony C},
	year         = 2004,
	journal      = {Medical Journal of Australia},
	publisher    = {Australasian Medical Publishing Company Proprietary, Ltd.},
	volume       = 180,
	number       = 6,
	pages        = 289
}

@article{lambert2022using,
	title        = {Using patient biomarker time series to determine mortality risk in hospitalised COVID-19 patients: A comparative analysis across two New York hospitals},
	author       = {Lambert, Ben and Stopard, Isaac J and Momeni-Boroujeni, Amir and Mendoza, Rachelle and Zuretti, Alejandro},
	year         = 2022,
	journal      = {Plos one},
	publisher    = {Public Library of Science San Francisco, CA USA},
	volume       = 17,
	number       = 8,
	pages        = {e0272442}
}

@misc{kumarajarshi_ray_2020,
	title        = {Life Expectancy (WHO)},
	author       = {Ray, Kumarajarshi},
	year         = 2020,
	publisher    = {Kaggle},
	url          = {https://www.kaggle.com/datasets/kumarajarshi/life-expectancy-who},
	howpublished = {\url{https://www.kaggle.com/datasets/kumarajarshi/life-expectancy-who}}
}

@inproceedings{kingma2015adam,
	title        = {Adam: A method for stochastic gradient descent},
	author       = {Kingma, Diederik P and Ba, Jimmy Lei},
	year         = 2015,
	booktitle    = {ICLR: international conference on learning representations},
	pages        = {1--15}
}

@article{paszke2019pytorch,
	title        = {Pytorch: An imperative style, high-performance deep learning library},
	author       = {Paszke, Adam and Gross, Sam and Massa, Francisco and Lerer, Adam and Bradbury, James and Chanan, Gregory and Killeen, Trevor and Lin, Zeming and Gimelshein, Natalia and Antiga, Luca and others},
	year         = 2019,
	journal      = {Advances in neural information processing systems},
	volume       = 32
}

@article{hill2011bayesian,
	title        = {Bayesian nonparametric modeling for causal inference},
	author       = {Hill, Jennifer L},
	year         = 2011,
	journal      = {Journal of Computational and Graphical Statistics},
	publisher    = {Taylor \& Francis},
	volume       = 20,
	number       = 1,
	pages        = {217--240}
}

@article{louizos2017causal,
	title        = {Causal effect inference with deep latent-variable models},
	author       = {Louizos, Christos and Shalit, Uri and Mooij, Joris M and Sontag, David and Zemel, Richard and Welling, Max},
	year         = 2017,
	journal      = {Advances in neural information processing systems},
	volume       = 30
}

@inproceedings{hebert2018multicalibration,
	title        = {Multicalibration: Calibration for the (computationally-identifiable) masses},
	author       = {H{\'e}bert-Johnson, Ursula and Kim, Michael and Reingold, Omer and Rothblum, Guy},
	year         = 2018,
	booktitle    = {International Conference on Machine Learning},
	pages        = {1939--1948},
	organization = {PMLR}
}

@inproceedings{kim2019multiaccuracy,
	title        = {Multiaccuracy: Black-box post-processing for fairness in classification},
	author       = {Kim, Michael P and Ghorbani, Amirata and Zou, James},
	year         = 2019,
	booktitle    = {Proceedings of the 2019 AAAI/ACM Conference on AI, Ethics, and Society},
	pages        = {247--254}
}

@inproceedings{chung2019slice,
	title        = {Slice finder: Automated data slicing for model validation},
	author       = {Chung, Yeounoh and Kraska, Tim and Polyzotis, Neoklis and Tae, Ki Hyun and Whang, Steven Euijong},
	year         = 2019,
	booktitle    = {2019 IEEE 35th International Conference on Data Engineering (ICDE)},
	pages        = {1550--1553},
	organization = {IEEE}
}

@inproceedings{sagadeeva2021sliceline,
	title        = {Sliceline: Fast, linear-algebra-based slice finding for ml model debugging},
	author       = {Sagadeeva, Svetlana and Boehm, Matthias},
	year         = 2021,
	booktitle    = {Proceedings of the 2021 international conference on management of data},
	pages        = {2290--2299}
}

@misc{adultdataset,
	title        = {{Adult}},
	author       = {Becker, Barry and Kohavi, Ronny},
	year         = 1996,
	note         = {{DOI}: https://doi.org/10.24432/C5XW20},
	howpublished = {UCI Machine Learning Repository}
}

@article{belice2020gender,
	title        = {The gender differences as a risk factor in diabetic patients with COVID-19},
	author       = {Belice, Tahir and Demir, Ismail},
	year         = 2020,
	journal      = {Iranian journal of microbiology},
	volume       = 12,
	number       = 6,
	pages        = 625
}

@article{rushovich2021sex,
	title        = {Sex disparities in COVID-19 mortality vary across US racial groups},
	author       = {Rushovich, Tamara and Boulicault, Marion and Chen, Jarvis T and Danielsen, Ann Caroline and Tarrant, Amelia and Richardson, Sarah S and Shattuck-Heidorn, Heather},
	year         = 2021,
	journal      = {Journal of General Internal Medicine},
	publisher    = {Springer},
	volume       = 36,
	number       = 6,
	pages        = {1696--1701}
}

\clearpage
\appendix
\onecolumn
\section{Proofs}
In this section, we provide a thorough investigation of the formal and causal properties of differential subgroups.
\subsection{Covariate Dependence}
We first show the connection between a minimized covariate dependence $\sufficiency(\subgroupf)=0$ (Def.~\ref{def:covariate}) and the 
independence of features.
\begin{restatable}{lemma}{independencelemma}
    \label{le:independence}
Given a distribution distance $D:\mathcal{P}(\mathcal{\targvar}) \times \mathcal{P}(\mathcal{\targvar}) \to [0,\infty)$, where it holds that $D(P,Q)=0$ if and only if $P=Q$.
A minimal covariate dependence implies a conditional independence between the features $\featvar$ 
and the target $\targvar$ within the scope of the subgroup $\subgroupf(\featvar)$, i.e.
 $$\sufficiency(\subgroupf)=0 \Rightarrow (\targvar \indep \featvar) \mid (\attributevar, \subgroupf(\featvar)=1)\;.$$
\end{restatable}
\begin{proof}
    The covariate dependence is defined as 
    \[
        \sufficiency_0(\subgroupf) = E_{\featvar \mid \attributevar=0, \subgroupf(\featvar)=1}\left[D\left(P_{0,\subgroupf}(\targvar \mid  \featvar = \feat), P_{0,\subgroupf}(\targvar)\right)\right]\;,
    \]
    for $\attributevar = 0$ and analogously for $\attributevar=1$.
    It is comprised of the integral over the feature space $\featdomain$ as
    \[
        \sufficiency_0(\subgroupf) = \int_{\feat} D\left(P_{0,\subgroupf}(\targvar \mid  \featvar = \feat), P_{0,\subgroupf}(\targvar)\right) p(\feat \mid \attributevar=0, \subgroupf(\featvar)=1) d\feat\;.
    \]    
    Given that $\sufficiency_0(\subgroupf)=0$ and $D$ has a domain of $[0,\infty)$, it holds that for every 
    $\feat \in \featdomain$ with $p(\feat \mid \attributevar=0, \subgroupf(\featvar)=1)>0$ the conditional and target distribution are equal, i.e.
    \[
        \forall \feat\in \featdomain,  p(\feat \mid \attributevar=0, \subgroupf(\featvar)=1)>0: P(\targvar\mid\attributevar=0, \featvar=\feat, \subgroupf(\featvar)=1)=P(\targvar\mid \attributevar=0,\subgroupf(\featvar)=1)\;.
    \]
    Therefore, it holds that
    \[
        P(\targvar\mid\featvar=\feat, \attributevar=0, \subgroupf(\featvar)=1) = P(\targvar\mid\attributevar=0, \subgroupf(\featvar)=1)\;,
    \]
    i.e.~$(\targvar \indep \featvar) \mid (\attributevar=0, \subgroupf(\featvar)=1)$.
    We can similarly derive the same result for $\attributevar=1$.
    This shows the claim that if $\sufficiency(\subgroupf)=0$, the features $\featvar$ and $\targvar$ are conditionally independent given $\subgroupf(\featvar)=1$ and $\attributevar$.
\end{proof}

\subsection{Proof of Proposition \ref{prop:causal_effect}}
\label{sec:proof_causal_effect}
\propCausalEffect*
\begin{proof}
    Firstly, the population support $\generality(\subgroupf)>0$ ensures positivity, i.e.~that there are samples from both populations within the subgroup, i.e.~$P(\subgroupf(\featvar)=1 \mid \attributevar=a)>0$ for $a\in\{0,1\}$.
    
    We start with the \textbf{interventional distribution} the right-hand side. Since $\featvar$ is a valid backdoor, we can identify the causal effect within the subgroup by adjusting for $\featvar$:
    $$
    P(\targvar \mid \doit(\attributevar=a), \subgroupf(\featvar)=1) = \int_{\feat} P(\targvar \mid \attributevar=a, \featvar=x, \subgroupf(\featvar)=1) p(\feat \mid \subgroupf(\featvar)=1) d\feat\;.
    $$
    If the covariate dependence is minimized so that $\sufficiency(\subgroupf)=0$ and $D$ is a distribution distance for which $D(P,Q)=0$ iff $P=Q$,
    then as per Lemma \ref{le:independence} we know that $(\targvar \indep \featvar) \mid (\attributevar, \subgroupf(\featvar)=1)$.
    
    This means that $P(\targvar \mid \attributevar=a, \featvar=x, \subgroupf(\featvar)=1) = P(\targvar \mid \attributevar=a, \subgroupf(\featvar)=1)$. Substituting this into the equation yields
    \begin{align*}
    P(\targvar \mid \doit(\attributevar=a), \subgroupf(\featvar)=1) &= \int_{\feat} P(\targvar \mid \attributevar=a, \subgroupf(\featvar)=1) p(\feat \mid \subgroupf(\featvar)=1) d\feat\\
    &= P(\targvar \mid \attributevar=a, \subgroupf(\featvar)=1) \int_{\feat} p(\feat \mid \subgroupf(\featvar)=1) d\feat
    \end{align*}
    The summation term $\int_{\feat} p(\feat\mid \subgroupf(\featvar)=1)$ is the integrated density of $\featvar$ over its support within the subgroup, which equals 1.
    $$
    P(\targvar \mid \doit(\attributevar=a), \subgroupf(\featvar)=1) = P(\targvar \mid \attributevar=a, \subgroupf(\featvar)=1) \cdot 1 = P_{\attribute,\subgroupf}(\targvar)\;.
    $$
    
    This shows that under the condition of zero covariate dependence, the interventional distribution (LHS) within the subgroup simplifies to the observational distribution (RHS), which completes the proof.
\end{proof}

\subsection{Proof of Proposition \ref{prop:indirect_only}}
\label{sec:proof_indirect_only}
\propIndirect*
\begin{proof}
We will prove this by contradiction. Assume that in a model of full mediation, there \textbf{exists} a subgroup $\subgroupf$ that simultaneously satisfies both conditions:
\begin{enumerate}
    \item Covariate independence: $\sufficiency(\subgroupf)=0$, which implies $(\featvar \indep \targvar) \mid (\attributevar, \subgroupf(\featvar)=1)$.
    \item Exceptional outcome: $\mathcal{E}(\subgroupf) > 0$, which implies that the outcome distributions differ, i.e., $P(\targvar \mid \attributevar=1, \subgroupf(\featvar)=1) \neq P(\targvar \mid \attributevar=0, \subgroupf(\featvar)=1)$.
\end{enumerate}
Now, we analyze the consequences of the other assumptions. Let x be any feature vector from the support of the subgroup, i.e., $\subgroupf(x)=1$. The premise that the supports overlap ensures such an x exists for both populations.

From Condition 1 ($\sufficiency(\subgroupf)=0$), Lemma \ref{le:independence} tells us that the target and features are conditionally independent given the attribute and subgroup membership. This allows us to state that for any x with $\subgroupf(x)=1$, the distribution of $\targvar$ is constant within the subgroup for a given attribute a:
$$P(\targvar \mid \attributevar=a, \subgroupf(\featvar)=1) = P(\targvar \mid \attributevar=a, \featvar=x, \subgroupf(\featvar)=1) \quad \text{for } a \in \{0,1\}$$
From the problem setup, we assume a model of full mediation, which is formally defined as $(\targvar \indep \attributevar) \mid \featvar$. This means that once the features x are known, the attribute a provides no additional information about the target $\targvar$. This gives us the following equality:
As per the causal structure of the data, $\featvar$ is a mediator between $\attributevar$ and $\targvar$, so that it holds for every sample $\feat$:
$$P(\targvar \mid \featvar=\feat, \attributevar=a)=P(\targvar \mid \featvar=\feat)\;.$$

We can now combine these two equalities. Beginning with the $\attributevar=1$ population:
\begin{align*}
P(\targvar \mid \attributevar=1, \subgroupf(\featvar)=1) &= P(\targvar \mid \attributevar=1, \featvar=x, \subgroupf(\featvar)=1) \
&= P(\targvar \mid \featvar=x, \subgroupf(\featvar)=1)
\end{align*}
We now apply the same transformation to the $\attributevar=0$ population:
\begin{align*}
P(\targvar \mid \attributevar=0, \subgroupf(\featvar)=1) &= P(\targvar \mid \attributevar=0, \featvar=x, \subgroupf(\featvar)=1) \
&= P(\targvar \mid \featvar=x, \subgroupf(\featvar)=1)
\end{align*}
By chaining these equalities, we have shown that both population-specific distributions are equal to the same underlying distribution, $P(\targvar \mid \featvar=x, \subgroupf(\featvar)=1)$. Therefore, they must be equal to each other:

$$P(\targvar \mid \attributevar=1, \subgroupf(\featvar)=1) = P(\targvar \mid \attributevar=0, \subgroupf(\featvar)=1)$$

This implies that the exceptionality score $\mathcal{E}(\subgroupf)$, which measures the divergence between these two distributions, must be 0. This is a direct contradiction to our initial assumption that $\mathcal{E}(\subgroupf) > 0$.
In a model of full mediation, there cannot exist a subgroup $\subgroupf$ that satisfies both $\sufficiency(\subgroupf)=0$ and $\mathcal{E}(\subgroupf) > 0$.
\end{proof}

\section{Implementation}
\label{sec:training_procedure}
We implement our method using \texttt{PyTorch} \citep{paszke2019pytorch}.
We optimize the parameters $\theta$ of the neural network $\subgroupfhat_{\temp}(\feat;\theta,\andweightvec)$ and the attribute weights $\andweightvec$ using Adam \citep{kingma2015adam}.
We describe the chosen hyperparameters for the optimization in Appendix \ref{app:hyperparameters}.


\subsection{Density Estimation}
For a \textbf{discrete} target $\targvar \in \{0, \dots, \nclasses\}$, we use the empirical distribution weighted by the subgroup membership, e.g.  
\[
\hat{P}_{0,\subgroupfhat}(\targvar = l)
= \frac{\sum_{\{i : \attributei=0\}} 
    \subgroupfhat_{\temp}(\feati) \,\mathds{1}(\targi=l)}
   {\sum_{\{i : \attributei=0\}} 
    \subgroupfhat_{\temp}(\feati)} \;.
\]

For a \textbf{continuous} target $\targvar \in \R$, we use a kernel density estimator \citep{terrell1992variable} using 
the implementation from \texttt{scipy.stats.gaussian\_kde} where each sample is weighted by 
$\subgroupfhat_{\temp}(\feati)$.
The bandwidth is selected using Scott's rule of thumb.
This yields two subgroup density estimators 
$\phat_{0,\subgroupf}(\targ)$ 
and 
$\phat_{1,\subgroupf}(\targ)$.  

\subsection{Regularization}
As part of the estimation of covariate dependence, we estimate the conditional distribution $\hat{P}(\targvar \mid \featvar=\feat)$ for 
discrete and continuous target variables.
For a \textbf{discrete} target variable, we train a boosting ensemble $f(\feat)$ using \texttt{RandomForestClassifier} from \texttt{sklearn} to predict the class probabilities $\hat{\mathbb{P}}(\targvar=l \mid \featvar=\feat)$.
For a \textbf{continuous} target variable, we train a regression model using \texttt{RandomForestRegressor} from \texttt{sklearn}.
In both models, we use the default hyperparameters of \texttt{sklearn} (100 estimators, unrestricted depth, gini criterion and least squares loss respectively).
At this time, we do not perform hyperparameter tuning for these models as it would significantly increase the computational cost of training.

\subsection{Temperature Annealing}
\label{app:temperature-schedules}
Temperature schedules are an important mechanism in optimization with soft relaxations of discrete functions. They provide a smooth transition from soft to hard decision boundaries, improving both convergence and final performance. By adjusting the temperature parameter, we control the degree of smoothness: higher values encourage exploration in the early stages of training, while lower values lead to sharper, more precise solutions as training progresses.  

We employ a linear decay schedule during the second half of training for the temperature parameter. The temperature $\temp$ decays from $0.2$ to $0.05$.
These ranges were selected via hyperparameter optimization and remain fixed across all experiments.
The update at each epoch is implemented as follows:
\begin{verbatim}
temp_start = 0.2
temp_end = 0.05
temp = temp_start
step_size = (temp_start - temp_end)/(total_epochs)
for epoch in range(total_epochs):
    temp = temp - step_size
\end{verbatim}
\section{Experiments}
All experiments were conducted on a consumer grade laptop.

\subsection{Synthetic Data Generation}
\begin{figure*}[ht]
    \begin{subfigure}[t]{0.3\linewidth}
        \centering
        \begin{tikzpicture}[->, >=stealth, thick, every node/.style={circle, draw, minimum size=8mm, font=\small}]
            \node (A) at (0,1.5) {A};
            \node (X) at (-1,0) {X};
            \node (Y) at (1,0) {Y};
            \draw (X) -- (Y);
            \draw (X) -- (A);
            \draw (A) -- (Y);
        \end{tikzpicture}
        \caption{Observational Study: $\mathbb{P}(\attributevar=1|\featvar)=\sigma(\beta_{\attributevar}^T \featvar)$, $\featvar \sim \mathcal{U}(0,1)^d$.}
        \label{fig:synth_observational}
    \end{subfigure}
    \hfill
    \begin{subfigure}[t]{0.3\linewidth}
        \centering
        \begin{tikzpicture}[->, >=stealth, thick, every node/.style={circle, draw, minimum size=8mm, font=\small}]
            \node (A) at (0,1.5) {A};
            \node (X) at (-1,0) {X};
            \node (Y) at (1,0) {Y};
            \draw (X) -- (Y);
            \draw (A) -- (Y);
        \end{tikzpicture}
        \caption{Randomized Trial: $\mathbb{P}(\attributevar=1)=0.5$. $\featvar \sim \mathcal{U}(0,1)^d$.}
        \label{fig:synth_randomized}
    \end{subfigure}
    \hfill
    \centering
    \begin{subfigure}[t]{0.3\linewidth}
        \centering
        \begin{tikzpicture}[->, >=stealth, thick, every node/.style={circle, draw, minimum size=8mm, font=\small}]
            \node (A) at (0,1.5) {A};
            \node (X) at (-1,0) {X};
            \node (Y) at (1,0) {Y};
            \draw (A) -- (X);
            \draw (A) -- (Y);
            \draw (X) -- (Y);
        \end{tikzpicture}
        \caption{Demographic Groups: $\mathbb{P}(\attributevar=1)=0.5$. $\featvar \sim \mathcal{U}(0,1)^d+\mathds{1}(\attributevar=1) \cdot\boldsymbol{\mu}$, where $\boldsymbol{\mu} \in [-0.3,0.3]^{\nfeat}$.}
        \label{fig:synth_demographic}
    \end{subfigure}
    \caption{Data generation process for the three settings: observational studies, randomized trials, and demographic groups. In all settings, the target is a linear function of the covariates with a subgroup-dependent mean-shift between 
    $P(\targvar|\attributevar=1)$ and $P(\targvar|\attributevar=0)$ such that $\targvar = f(\featvar, \attributevar) + N_Y$.
    }
    \label{fig:causal_structure}
\end{figure*}
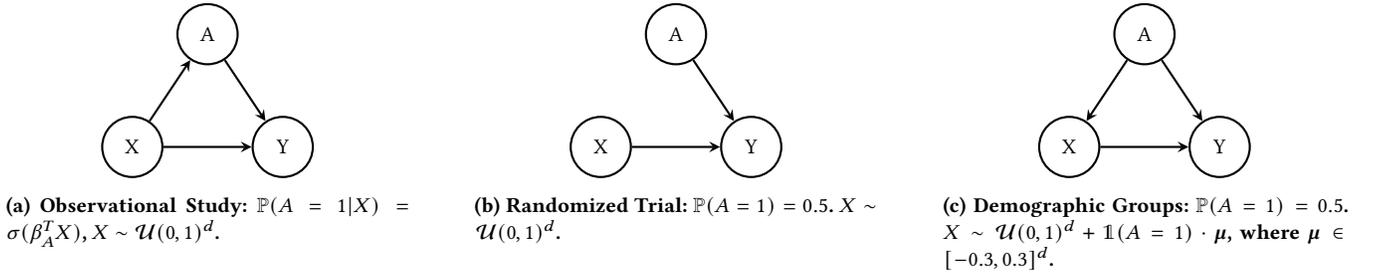
\label{app:simulated_data}

To evaluate the performance of methods in recovering differential subgroups, we generate datasets under three different causal structures.
In all settings, the target variable $\targvar$ is a function of a binary attribute $\attributevar$ and a set of covariates $\featvar$ as per 
\begin{equation}
\targvar = f(\featvar, \attributevar) + N_Y\;,
\end{equation}

\begin{figure}[ht]
    \centering
    \includegraphics[width=\linewidth]{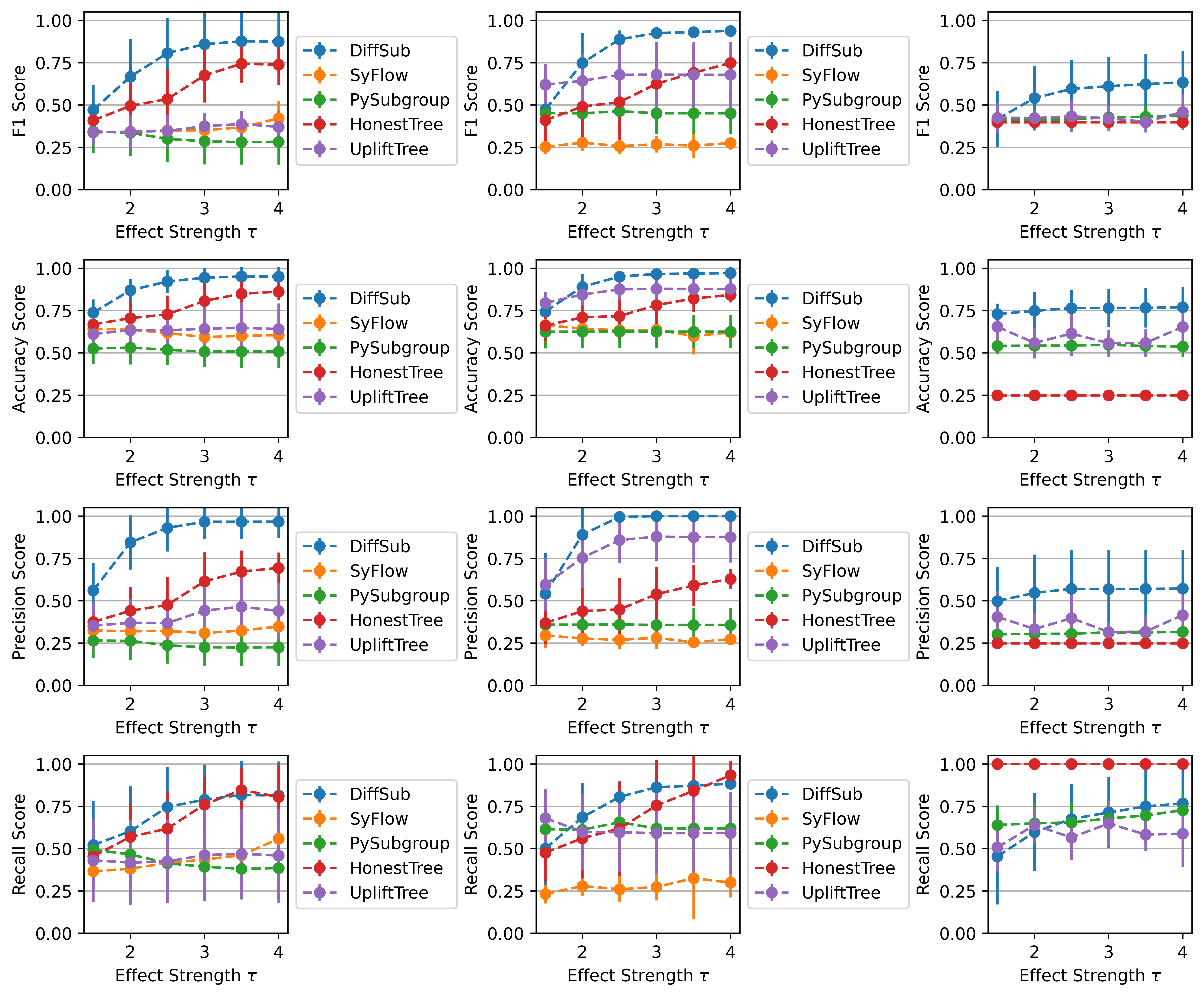}
    \caption{From top to bottom row: $F_1$-score, accuracy, precision, recall of recovered subgroup.
    \ourmethod outperforms the competing methods in the observational (left column), randomized (middle column) and demographic (right column)causal 
    data generating mechanisms.}
    \label{fig:all_metrics}
\end{figure}

We summarize the respective data generation process for each setting in Figure \ref{fig:causal_structure}.
The core components of the generated data are:
\begin{itemize}
    \item \textbf{Subgroup $s^*(X)$}: A ground-truth subgroup defined by an axis-aligned box rule on a small subset of features $\mathcal{J} \subset \{1, \dots, d\}$. A sample $\feat$ belongs to the subgroup if its covariates $\feat_j$ satisfy the rule:
    $$
    s^*(\feat) = \bigwedge_{j \in \mathcal{J}} (\lowerbound_j < \feat_j < \upperbound_j)
    $$
    where $\lowerbound_j$ and $\upperbound_j$ are lower and upper bounds for feature $j$.
    \item \textbf{Target Variable $\targvar$}: A continuous outcome variable that depends on the binary attribute $\attributevar$ and a set of covariates $\featvar$ as per  
    \begin{equation}
    \targvar = f(\featvar, \attributevar) + N_Y\;,
    \end{equation}
    where $N_Y$ is Gaussian noise with mean zero and variance $\sigma^2$. $f$ consist of a linear function $\beta_{\targvar}^T \featvar$, $\beta_{\targvar} \in [-1,1]^{\nfeat}$, with an subgroup dependent interaction term for the difference between
    $\attributevar=1$ and $\attributevar=0$:
    \begin{equation}
    f(\featvar, \attributevar) = \beta_{\targvar}^T \featvar + s^*(\featvar)\left(\frac{\tau}{2}\mathds{1}(\attributevar=1)-\frac{\tau}{2}\mathds{1}(\attributevar=0)\right)
    + (1-s^*(\featvar))\left(\frac{\eta}{2}\mathds{1}(\attributevar=1)-\frac{\eta}{2}\mathds{1}(\attributevar=0)\right)\;,
\end{equation}
    where $\tau$ is the treatment effect for individuals in the subgroup, i.e.~$s^*(\featvar)=1$, and $\eta$ is the treatment effect for individuals outside the subgroup, i.e.~$s^*(\featvar)=0$.
    \item \textbf{Covariates $X$}: A feature matrix $X \in \mathbb{R}^{n \times d}$. For the setting "observational" and "interventional", each feature $\featvar_j$ is sampled uniformly from the interval $[0, 1]$. 
    In the "demographic shift" setting, the features are mean-shifted conditional on the binary attribute $\attributevar$.
    \item \textbf{Population Indicator $A$}: A binary population assignment $A \in \{0, 1\}$. For the setting "interventional" and "demographic" shift, $A$ is sampled uniformly at random, $P(A=1)=0.5$.
    In the "observational" setting, $A$ is sampled from a Bernoulli distribution with probability dependent on the covariates $X$.
\end{itemize}

Across all settings, we vary the direct effect in the subgroup $\tau \in \{1.5,2,2.5,3,3.5,4\}$ compared to a outside effect of $\eta=1$.
The subgroup is defined by two randomly chosen features $\featvar_i$ and $\featvar_j$, with random bounds such that $P(s^*(\featvar)=1)=0.3$, i.e.~30\% of the samples belong to the subgroup.
The number of samples is set to $\nsamples=2000$ and the number of features to $\nfeat=5$.
The noise variance is set to $\sigma^2=0.5$.

As per the assumed causal struture, we obtain the following distributions for $\attributevar$ and $\featvar$:
\begin{itemize}
    \item \textbf{Observational}: The covariates $\featvar$ are sampled uniformly from the interval $[0, 1]^d$. The binary attribute $\attributevar$ is sampled from a Bernoulli distribution with probability dependent on the covariates $\featvar$
    $$
    P(\attributevar=1\mid \featvar) = \sigma(\beta_{\attributevar} ^T \featvar)\;,
    $$
    where $\sigma(\cdot)$ is the sigmoid function and $\beta_{\attributevar} \in [-1,1]^{\nfeat}$ is a vector of coefficients that determines the treatment assignment probability.
    \item \textbf{Interventional}: The covariates $\featvar$ are sampled uniformly from the interval $[0, 1]^d$. The binary attribute $\attributevar$ is sampled uniformly at random, $P(\attributevar=1)=0.5$.
    \item \textbf{Demographic Shift}: The binary attribute $\attributevar$ is sampled uniformly at random, $P(\attributevar=1)=0.5$.
    The covariates $\featvar$ are sampled uniformly from the interval $[0, 1]^d$, but with a mean shift conditional on the binary attribute $\attributevar$ as per 
    $$
    \featvar_i \sim \mathcal{U}(0,1)^d + \mu_i \cdot \mathds{1}(\attributevar=1)\;.
    $$
    The feature shift $\mu_i\in [-0.3,0.3]$ is the distribution shift of $\featvar_i$ induced by $\attributevar$.
\end{itemize}

\subsection{Results of Synthetic Experiments}
We generate data using causal mechanism as detailed above.
We evaluate the correctness of the recovered subgroup for each sample $\feati$, which we denote as $\subgroupfhat(\feat_i)$,
and compared to the ground truth $\subgroupf^*(\feati)$.

We report the following metrics :
\begin{itemize}
    \item $F_1 = \frac{2 TP}{2 TP + FP + FN}$
    \item Accuracy $= \frac{TP+TN}{TP+FP+TN+FN}$
    \item Precision $=\frac{TP}{TP+FP}$
    \item Recall $=\frac{TP}{TP+FN}$
\end{itemize}
We display all metrics for the synthetic experiments in Fig.~\ref{fig:all_metrics}.
With increasing strength $\tau$ of the direct causal effect, \ourmethod recovers the ground truth with high 
precision and good recall.
\ourmethod beats the domain specific methods \causaltree for observational data and \uplifttree for randomized control trials,
and is the only one to recover a reasonable result in the demographic setting.
On the other hand, the regular subgroup discovery methods \pysubgroup and \syflow do not work well in any setting,
due to their non differential approach.

\subsection{Runtimes}
\label{app:runtimes}
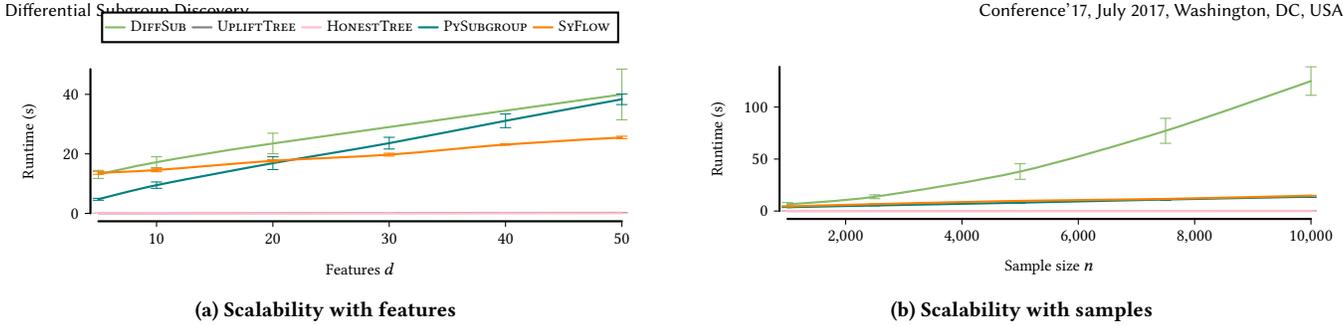
\begin{figure*}[t]
    \begin{subfigure}[t]{0.48\linewidth}
        \centering
        \begin{tikzpicture}
            \begin{axis}[
        pretty line,
        width=\linewidth,
        height=3.5cm,
        ylabel={Runtime (s)},
        xlabel={Features $d$},
        legend style={/tikz/overlay,
        legend columns=5,
        at={(0.5,1.18)},
        anchor=south,
        font=\scriptsize,
        },
        legend entries={\ourmethod, \uplifttree,\causaltree,\pysubgroup,\syflow},
        ymin=0, smooth ]

        \addplot+[dollarbill,
        error bars/.cd,
        y dir=both, y explicit]
        table [x index=0, y index=4, y error index=9, col sep=comma] {expres/synthetic_n_vars_Runtime.csv};

        \addplot+[gray,
        error bars/.cd,
        y dir=both, y explicit]
        table [x index=0, y index=1, y error index=6, col sep=comma] {expres/synthetic_n_vars_Runtime.csv};

        \addplot+[pink,
        error bars/.cd,
        y dir=both, y explicit]
        table [x index=0, y index=2, y error index=7, col sep=comma] {expres/synthetic_n_vars_Runtime.csv};

        \addplot+[teal,
        error bars/.cd,
        y dir=both, y explicit]
        table [x index=0, y index=3, y error index=8, col sep=comma] {expres/synthetic_n_vars_Runtime.csv};

        \addplot+[orange,
        error bars/.cd,
        y dir=both, y explicit]
        table [x index=0, y index=5, y error index=10, col sep=comma] {expres/synthetic_n_vars_Runtime.csv};
        \end{axis}
        \end{tikzpicture}
        \caption{Scalability with features}
        \label{fig:experiment_runtime_features}
        \end{subfigure}
    \hfill
    \begin{subfigure}[t]{0.48\linewidth}
        \centering
        \begin{tikzpicture}
            \begin{axis}[
        pretty line,
        width=\linewidth,
        height=3.5cm,
        ylabel={Runtime (s)},
        xlabel={Sample size $n$},
        ymin=0, smooth ]

        \addplot+[dollarbill,
        error bars/.cd,
        y dir=both, y explicit]
        table [x index=0, y index=4, y error index=9, col sep=comma] {expres/scalability_n_samples_Runtime.csv};

        \addplot+[gray,
        error bars/.cd,
        y dir=both, y explicit]
        table [x index=0, y index=1, y error index=6, col sep=comma] {expres/scalability_n_samples_Runtime.csv};

        \addplot+[pink,
        error bars/.cd,
        y dir=both, y explicit]
        table [x index=0, y index=2, y error index=7, col sep=comma] {expres/scalability_n_samples_Runtime.csv};

        \addplot+[teal,
        error bars/.cd,
        y dir=both, y explicit]
        table [x index=0, y index=3, y error index=8, col sep=comma] {expres/scalability_n_samples_Runtime.csv};

        \addplot+[orange,
        error bars/.cd,
        y dir=both, y explicit]
        table [x index=0, y index=5, y error index=10, col sep=comma] {expres/scalability_n_samples_Runtime.csv};
        \end{axis}
        \end{tikzpicture}
        \caption{Scalability with samples}
        \label{fig:experiment_runtime_samples}
        \end{subfigure}
    \caption{Runtime scalability of each method in synthetic experiments.}
    \label{fig:runtime_synthetic}
\end{figure*}

We report the runtimes of each method on the synthetic data in Fig.~\ref{fig:runtime_synthetic} with respect 
to the number of samples $\nsamples$ and the number of features $\nfeat$.
\ourmethod is slightly slower than \syflow due to the increased complexity of the loss function,
but is still in the same order of magnitude for an increasing number of features.
In terms of sample complexity, \ourmethod's use of KDE for density estimation results in higher computational costs as the number of samples increases,
as KDE scales quadratically with the number of samples whereas \syflow scales linearly.
The tree-based methods \causaltree and \uplifttree are significantly faster.
\pysubgroup is faster for few features $\nfeat$, but scales poorly with increasing number of features due to its combinatorial search strategy.
For increasing number of samples $\nsamples$, \pysubgroup remains fast due to its beam search strategy. 

\section{Hyperparameter Configuration}
\label{app:hyperparameters}

For our experiments, we performed a comprehensive hyperparameter search for each method across the three distinct settings introduced in the synthetic experiments: \textbf{observational}, \textbf{interventional}, and \textbf{demographic}. 
We employed a grid search, exhaustively testing all combinations of the parameters specified for each algorithm. 
The rule-based models, \ourmethod and \syflow, share a common rule configuration where the predicate temperature is initialized at $0.2$ and linearly annealed over the course of training, reducing by a factor of 10. 

We use the setting of $\nsamples=2000$, $\nfeat=5$, $\tau=4$, $\sigma=0.5$ for the synthetic data, which is consistent across all methods.
We generate 10 different synthetic datasets for each method to ensure robustness in our evaluation
and we selected the best-performing configuration for each specific setting based on validation performance.

The hyperparameters and their tested ranges for each method are as follows:

\begin{itemize}
    \item \ourmethod: The grid search included the trade-off parameter $\lambda \in \{0.1, 0.5, 1.0\}$, the generality parameter $\gamma \in \{0.1, 0.5\}$, the number of training epochs $\in \{500, 1000\}$, and the classifier's learning rate $\in \{0.001, 0.005\}$.
    \item \syflow: We tuned the size rewards as $\alpha \in \{0.1, 0.3, 0.5\}$, the classifier's learning rate $\in \{10^{-3}, 10^{-4}, 10^{-5}\}$, and the number of subgroup training epochs $\in \{200, 500, 1000\}$.
    \item \pysubgroup: We explored different values for the beam width $\in \{50, 100, 200\}$, the number of bins for discretization $n_{bins} \in \{5, 10, 20\}$, and the size trade-off parameter $\alpha \in \{0.2, 0.5, 1.0\}$.
    \item \causaltree \& \uplifttree: For both tree-based methods, we searched over the minimum number of samples per leaf, specified as a fraction of the dataset, \texttt{min\_samples\_leaf} $\in \{0.01, 0.05, 0.1, 0.2, 0.3\}$ and the maximum tree depth \texttt{max\_depth} $\in \{2, 3, 5, \text{None}\}$.
\end{itemize}

We use the validation performance to select the best hyperparameters for each method in each setting.
The validation performance is measured by the area under the receiver operating characteristic curve (AUC) for the subgroup prediction task, which is a common metric for evaluating the performance of classification models in subgroup discovery

The table below details the final hyperparameter configurations chosen for each method in the respective experimental settings.

\begin{table}[t]
\centering
\caption{Selected hyperparameters for each method across the different experimental settings.}
\label{tab:hyperparameters}
\centering
\begin{tabular}{@{}lll@{}}
\toprule
\textbf{Method} & \textbf{Setting} & \textbf{Selected Hyperparameters} \\ \midrule
\addlinespace
\ourmethod & \begin{tabular}[c]{@{}l@{}}Observational\\ Interventional\\ Demographic\end{tabular} & $\lambda=0.5$, $\gamma=0.1$, \texttt{n\_epochs}=500, \texttt{lr\_classifier}=0.005 \\
\addlinespace \addlinespace \addlinespace
\hline
\addlinespace
\syflow & Observational & $\alpha=0.5$, \texttt{lr\_classifier}=$10^{-3}$, \texttt{subgroup\_train\_epochs}=1000 \\
 & Interventional & $\alpha=0.1$, \texttt{lr\_classifier}=$10^{-3}$, \texttt{subgroup\_train\_epochs}=1000 \\
 & Demographic & $\alpha=0.1$, \texttt{lr\_classifier}=$10^{-4}$, \texttt{subgroup\_train\_epochs}=200 \\
\addlinespace
\hline
\addlinespace
\pysubgroup & \begin{tabular}[c]{@{}l@{}}Observational\\ Demographic\end{tabular} & \texttt{beam\_width}=200, \texttt{n\_bins}=20, $\alpha=1.0$ \\
\addlinespace
 & Interventional & \texttt{beam\_width}=100, \texttt{n\_bins}=5, $\alpha=1.0$ \\
\addlinespace
\hline
\addlinespace
\causaltree & Observational & \texttt{min\_samples\_leaf}=0.01, \texttt{max\_depth}=3 \\
 & Interventional & \texttt{min\_samples\_leaf}=0.01, \texttt{max\_depth}=5 \\
 & Demographic & \texttt{min\_samples\_leaf}=0.05, \texttt{max\_depth}=5 \\
\addlinespace
\hline
\addlinespace
\uplifttree & Observational & \texttt{min\_samples\_leaf}=0.1, \texttt{max\_depth}=3 \\
 & Interventional & \texttt{min\_samples\_leaf}=0.1, \texttt{max\_depth}=2 \\
 & Demographic & \texttt{min\_samples\_leaf}=0.3, \texttt{max\_depth}=5 \\
\addlinespace
\bottomrule
\end{tabular}%
\end{table}

\subsection{Sensitivity}
\label{app:sensitivity}
Additionally, we conduct a sensitivity analysis of the hyperparameters $\lambda$ and $\gamma$ of \ourmethod.
We vary $\lambda$ in the range $[0., 1.0]$ and $\gamma$ in the range $[0., 1.0]$.
We keep all other hyperparameters fixed to the values selected in the observational setting (see Table \ref{tab:hyperparameters}).
The results are displayed in Fig.~\ref{fig:sensitivity_lambda_gamma}.
For $\gamma$, we observe that a small value of $\gamma\in [0.1,0.3]$ leads to the best performance across all settings.
When disabling size regularization, i.e.~$\gamma=0$, the performance drops significantly.
On the other hand, when prioritizing size too much, i.e.~$\gamma>0.5$, the performance also drops.
For $\lambda$, we observe that a value of $\lambda\in[0.0,0.5]$ leads to good performance.
Increasing the strength of covariate independence regularization too much, i.e.~$\lambda>0.5$, leads to a drop in performance.
We hypothesize that this is due to the fact that the optimization focuses too much on minimizing covariate dependence, which can lead to very small subgroups with low exceptionality.

\textbf{Estimators.}
We now compare the effect of using different estimators for the subgroup density $\hat{P}_{\attribute,\subgroupf}(\targvar)$ and the local conditional distribution $\hat{P}(\targvar \mid \featvar=\feat, \attributevar=\attribute)$.
We compare the kernel density estimator (KDE) with a Gaussian Mixture Model (GMM), and a simple Gaussian.
We display the results in Fig.~\ref{fig:sensitivity_density}.
In all settings, KDE slightly outperforms GMM, which in turn is better than a simple Gaussian.
Especially in the demographic setting, using a Gaussian leads to a significant drop in performance.
This is likely due to the fact that the simple Gaussian is not flexible enough to capture the true distribution of the target variable within the subgroup.

Regarding the local conditional distribution $\hat{P}(\targvar \mid \featvar=\feat, \attributevar=\attribute)$, we compare a Random Forest (RF) with a SVM and a neural network.
We display the results in Fig.~\ref{fig:sensitivity_regularizer}.
The results show now trend no clear winner across all settings.
In the observational setting, NN performs best, while in the interventional and observational settings, RF has a slight edge.
Ultimately, the choice of estimator for the local conditional distribution does not seem to have a significant impact on the performance of \ourmethod
in the synthetic experiments.
\begin{figure}[t]
    \centering
    \begin{subfigure}[t]{0.24\linewidth}
        \centering
        \includegraphics[width=\linewidth]{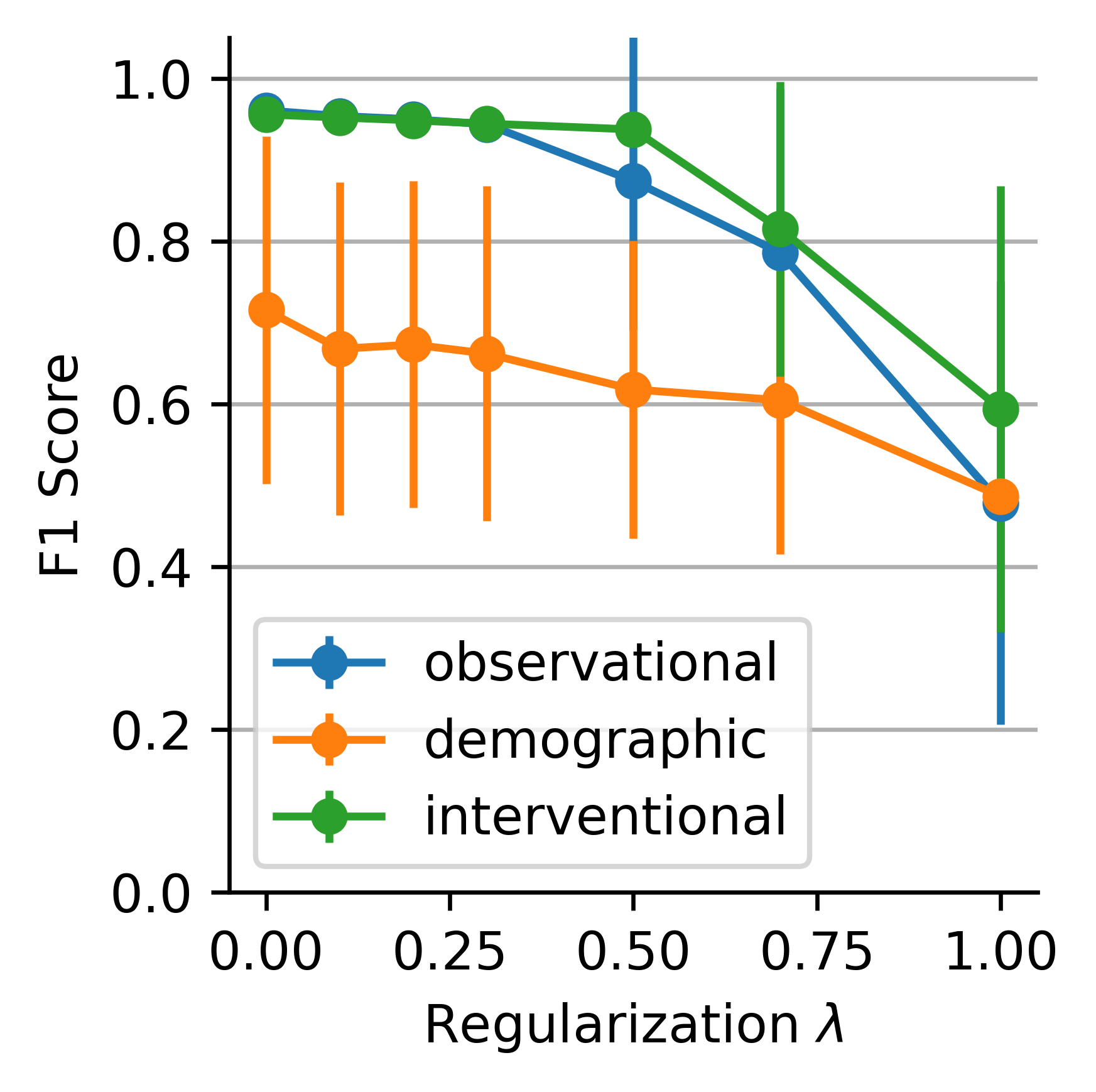}
        \caption{Sensitivity of $\lambda$.}
        \label{fig:sensitivity_lambda}
    \end{subfigure}
    \hfill
    \begin{subfigure}[t]{0.24\linewidth}
        \centering
        \includegraphics[width=\linewidth]{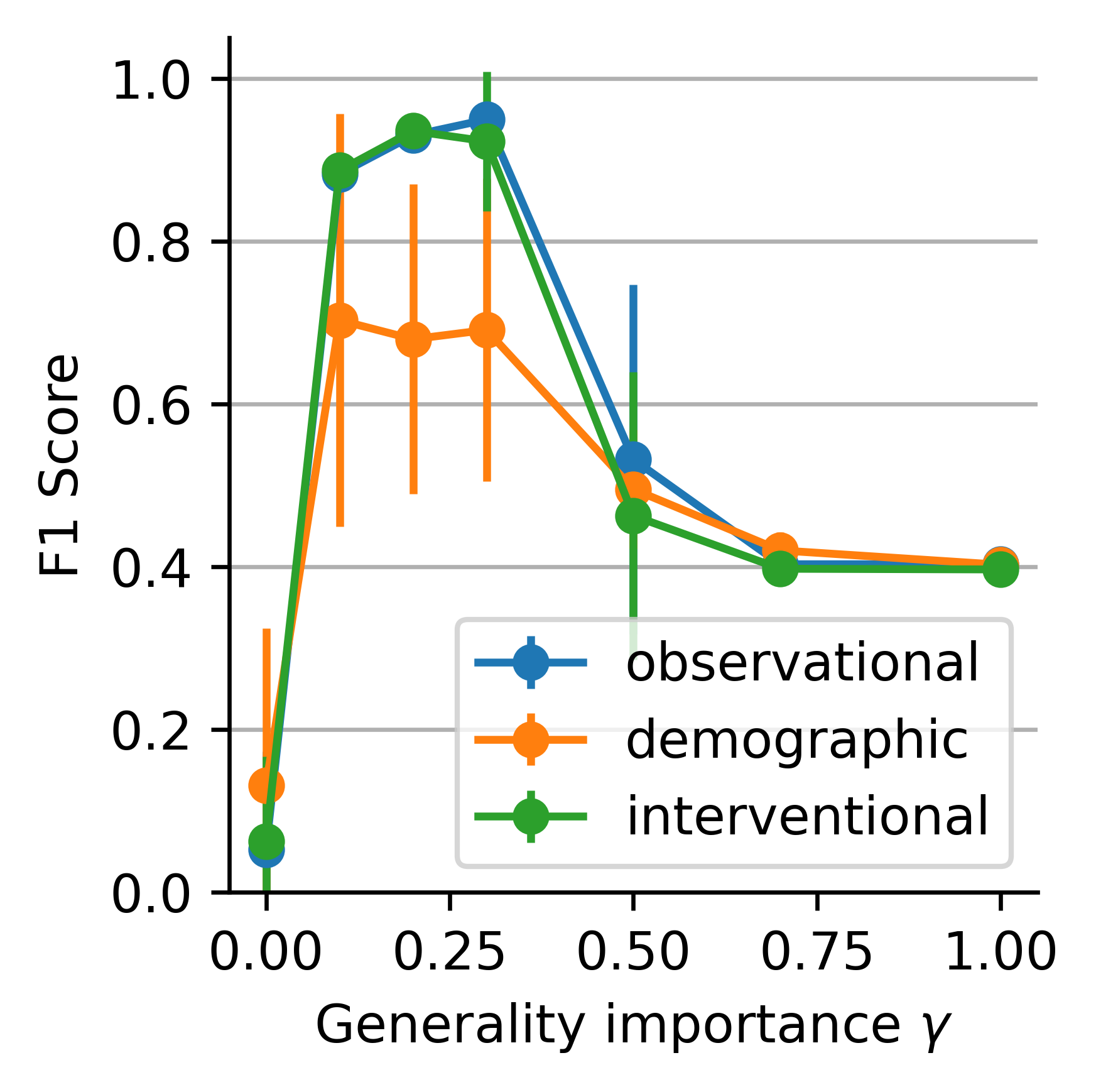}
        \caption{Sensitivity of $\gamma$.}
        \label{fig:sensitivity_gamma}
    \end{subfigure}
    \hfill
    \begin{subfigure}[t]{0.24\linewidth}
        \centering
        \includegraphics[width=\linewidth]{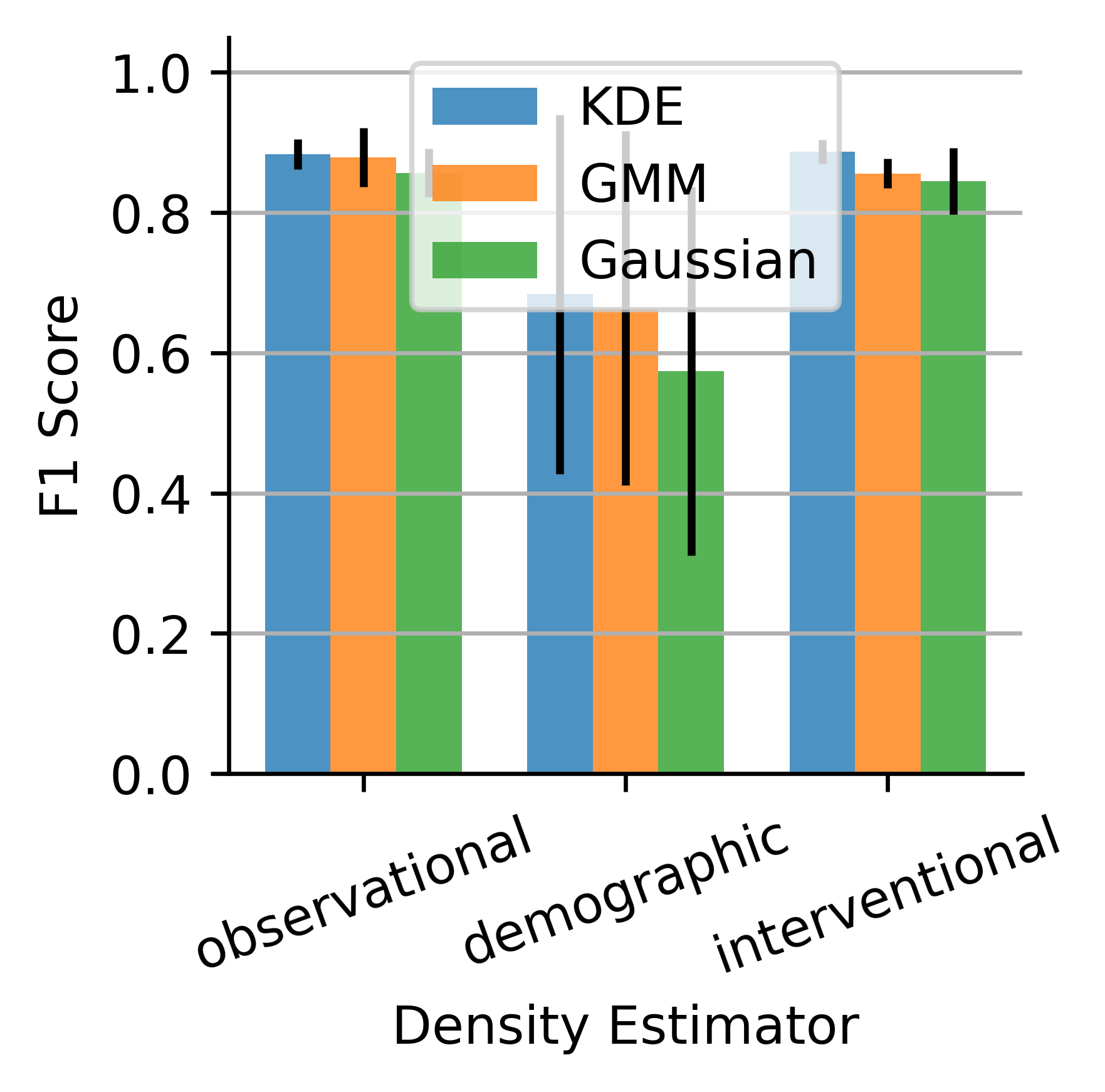}
        \caption{Sensitivity of subgroup density estimator.}
        \label{fig:sensitivity_density}
    \end{subfigure}
    \hfill
    \begin{subfigure}[t]{0.24\linewidth}
        \centering
        \includegraphics[width=\linewidth]{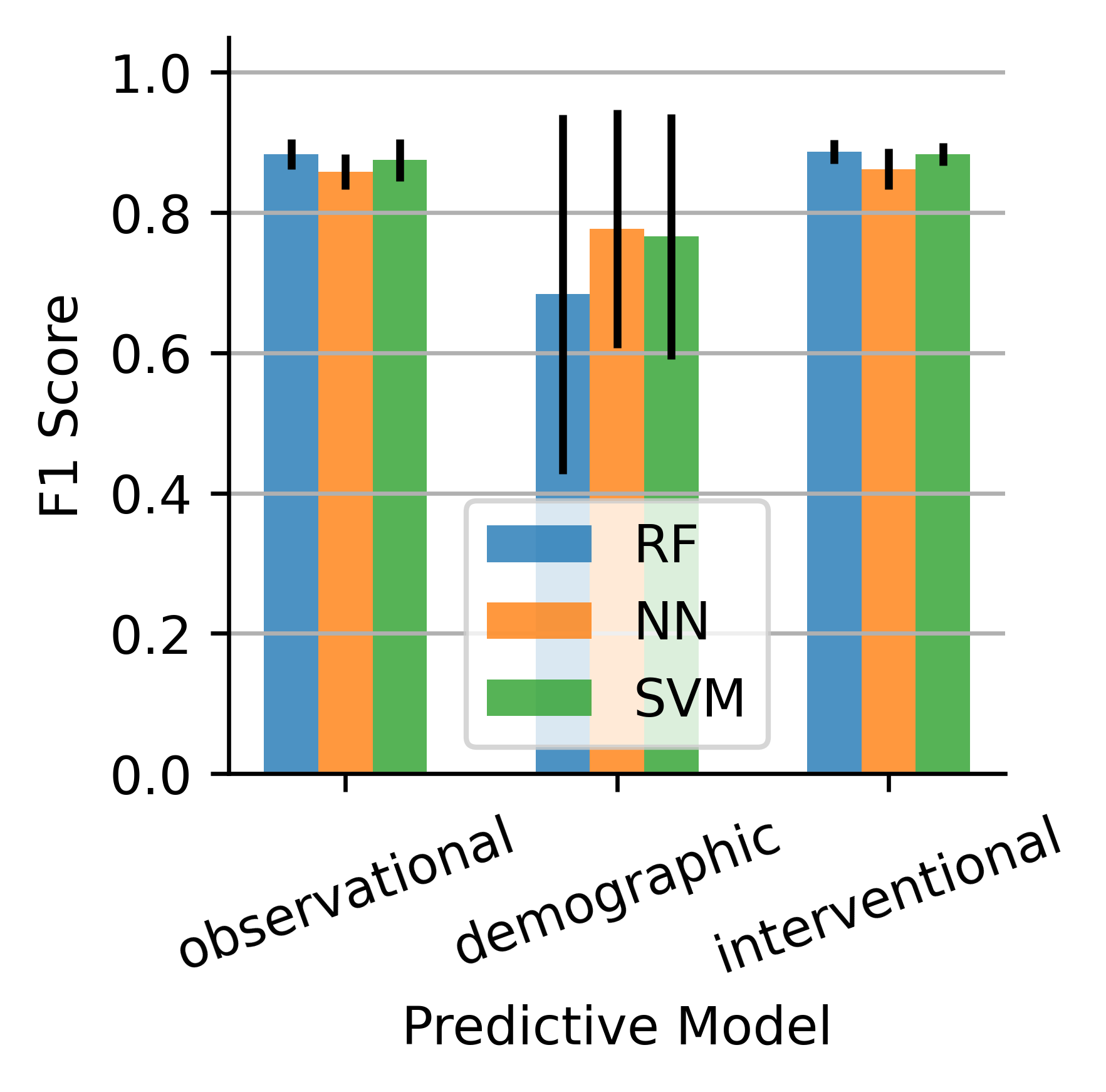}
        \caption{Sensitivity of regularization model.}
        \label{fig:sensitivity_regularizer}
    \end{subfigure}
    \caption{Sensitivity analysis of hyperparameters $\lambda$ and $\gamma$, as well as choice of estimators for 
    subgroup density and local density on the synthetic dataset. 
    }
    \label{fig:sensitivity_lambda_gamma}
\end{figure}

\clearpage

\section{Experiments: Real World}

\subsection{IHDP}
\label{app:ihdp}
We further display the subgroup distributions discovered for each permutation of the IHDP dataset \citep{hill2011bayesian}.
We plot the subgroup distributions in Fig.~\ref{fig:ihdp_overall}.
We list the rules for the discovered subgroups in Table \ref{tab:ihdp_rules}.
\begin{figure}
    \centering
    \begin{tabular}{cccc}
        \includegraphics[width=0.24\textwidth]{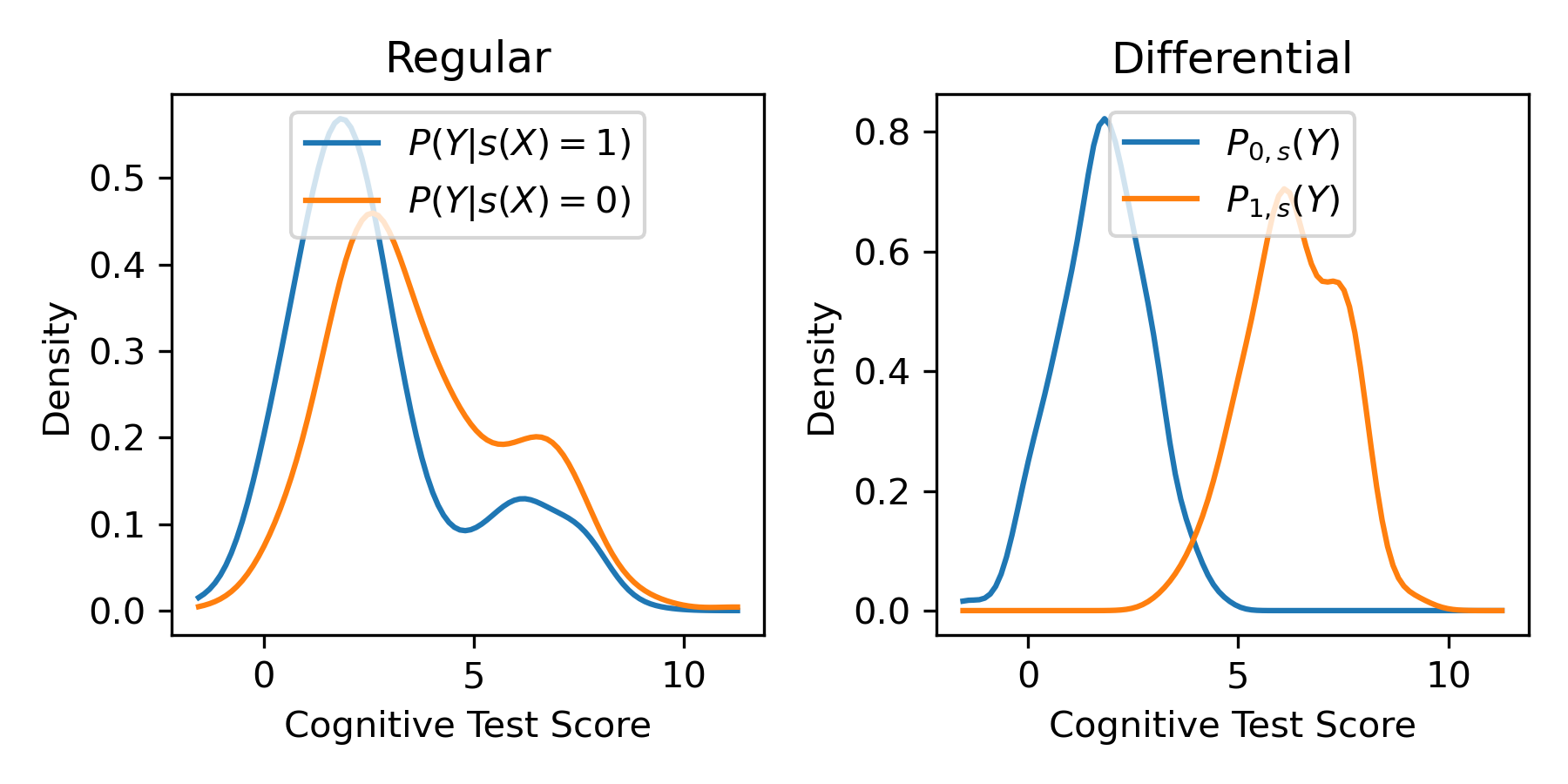} &
        \includegraphics[width=0.24\textwidth]{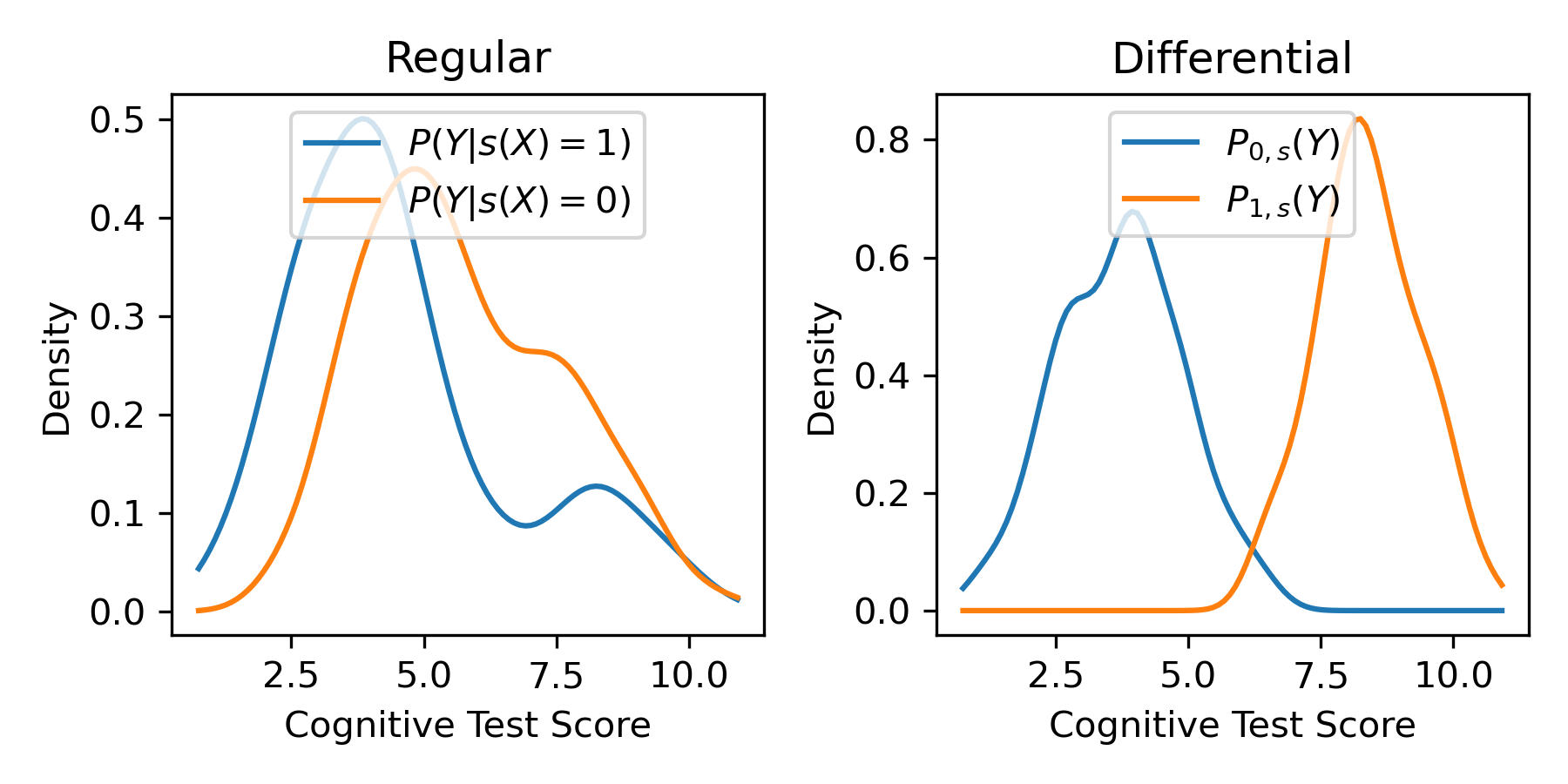} &
        \includegraphics[width=0.24\textwidth]{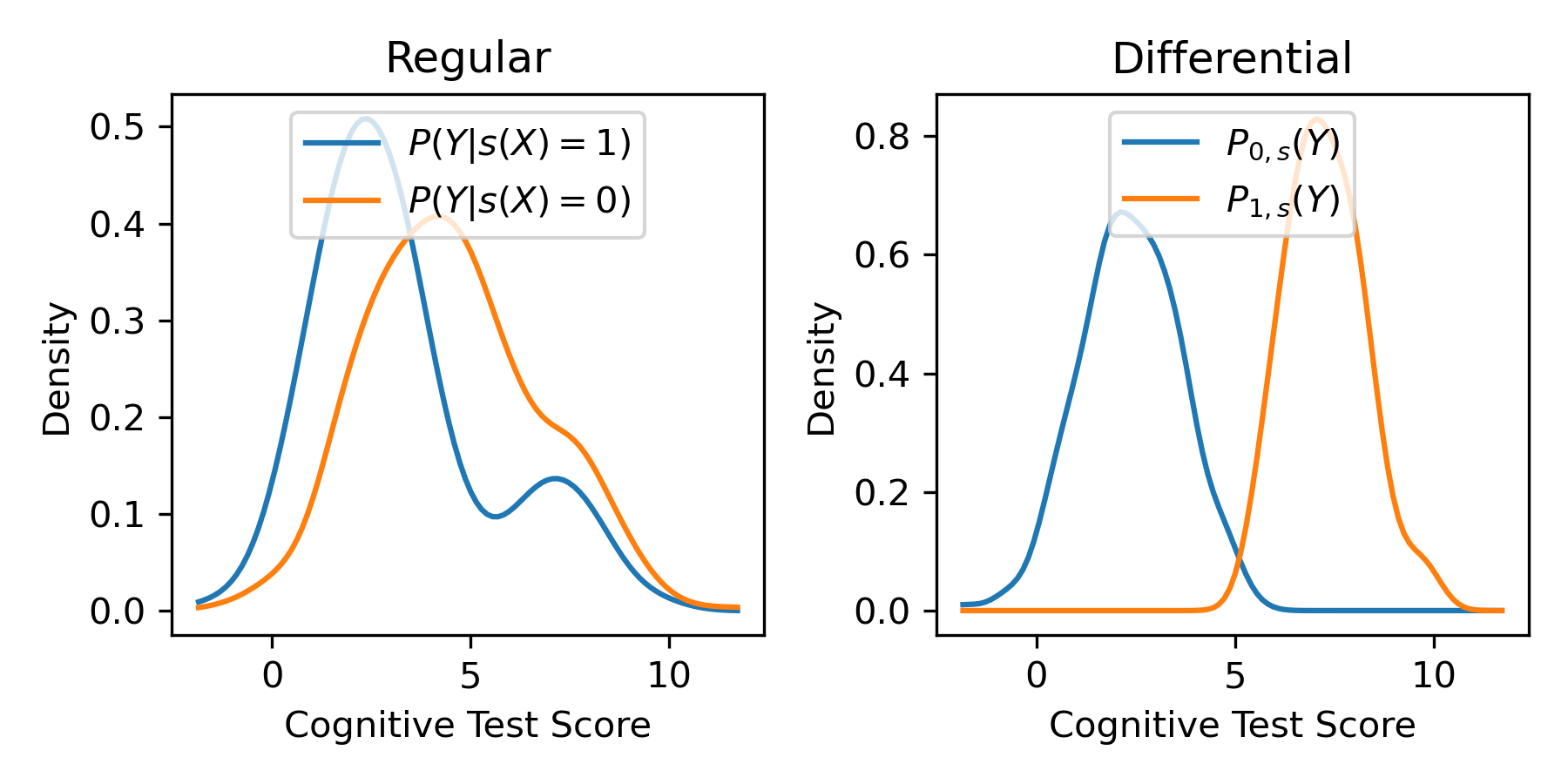} &
        \includegraphics[width=0.24\textwidth]{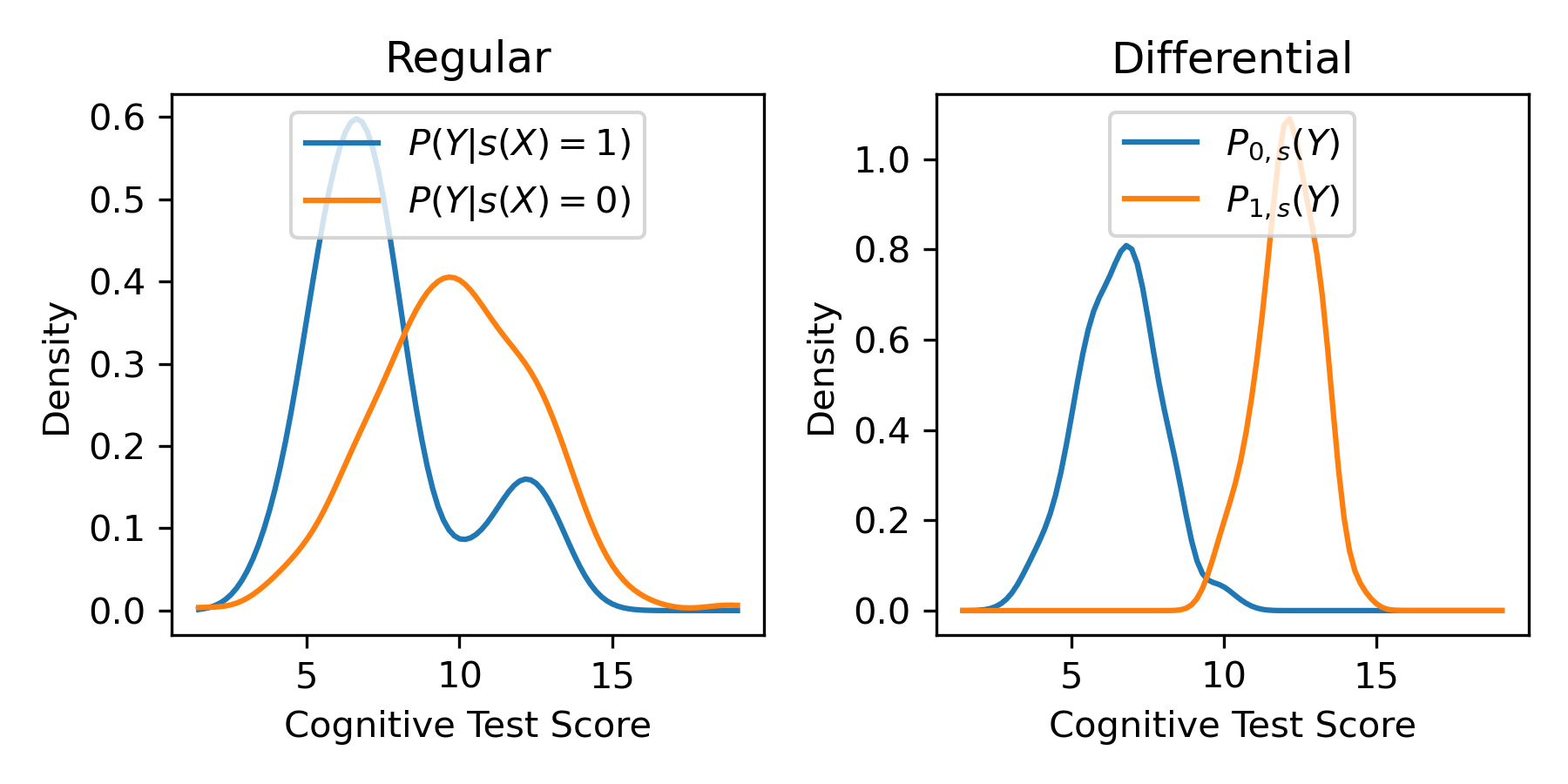} \\
        \includegraphics[width=0.24\textwidth]{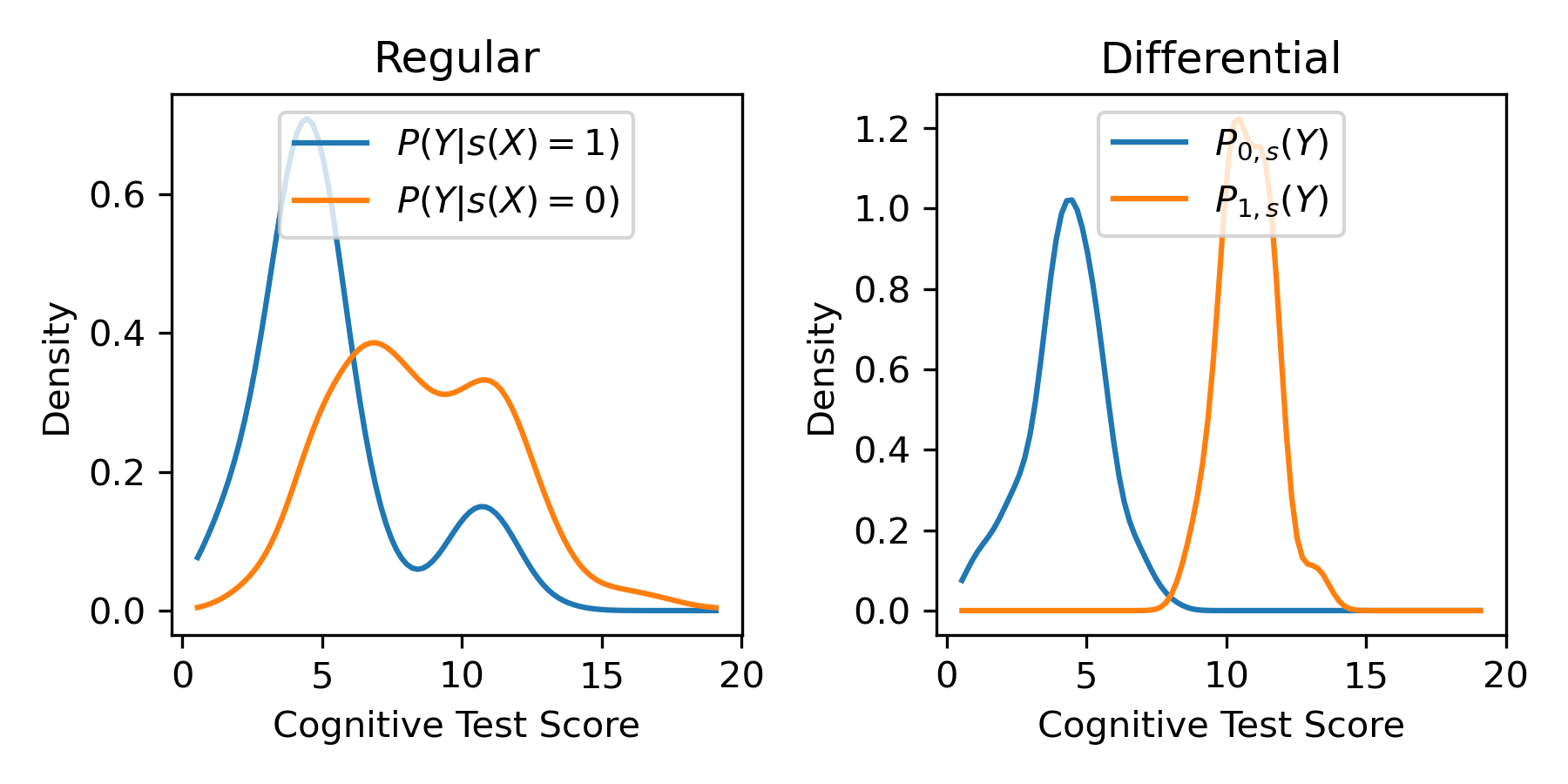} &
        \includegraphics[width=0.24\textwidth]{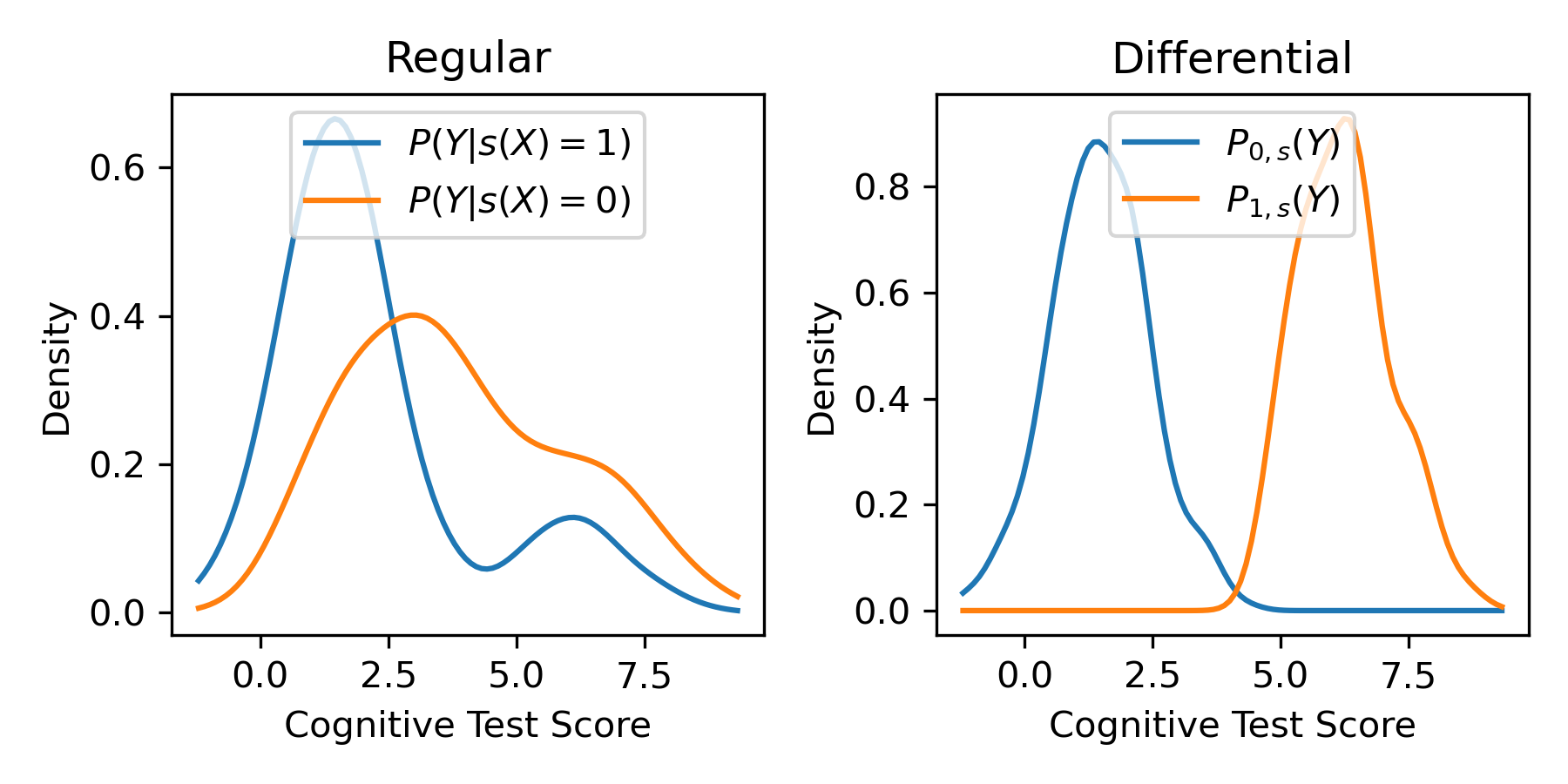} &
        \includegraphics[width=0.24\textwidth]{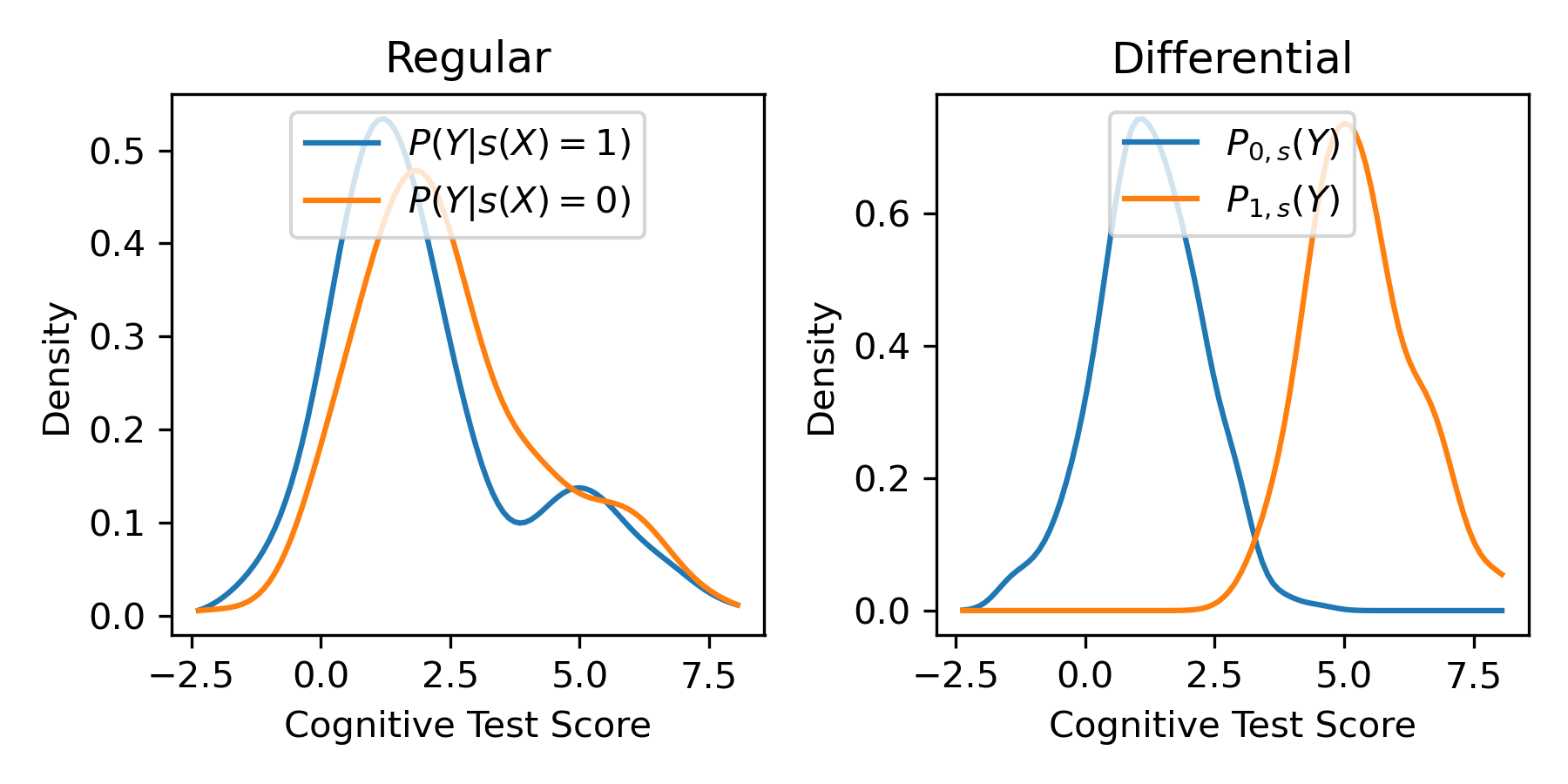} &
        \includegraphics[width=0.24\textwidth]{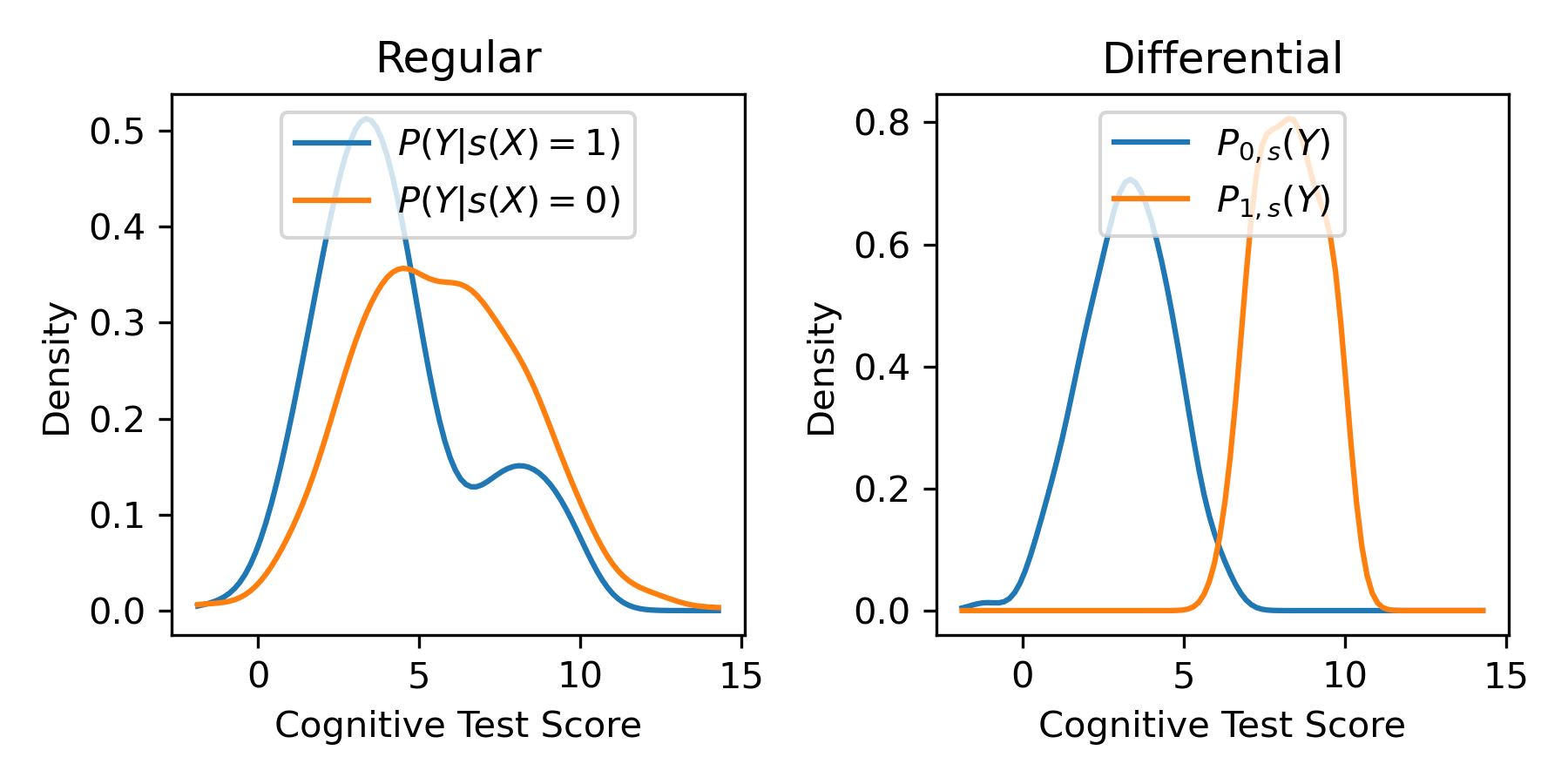} \\ &
        \includegraphics[width=0.24\textwidth]{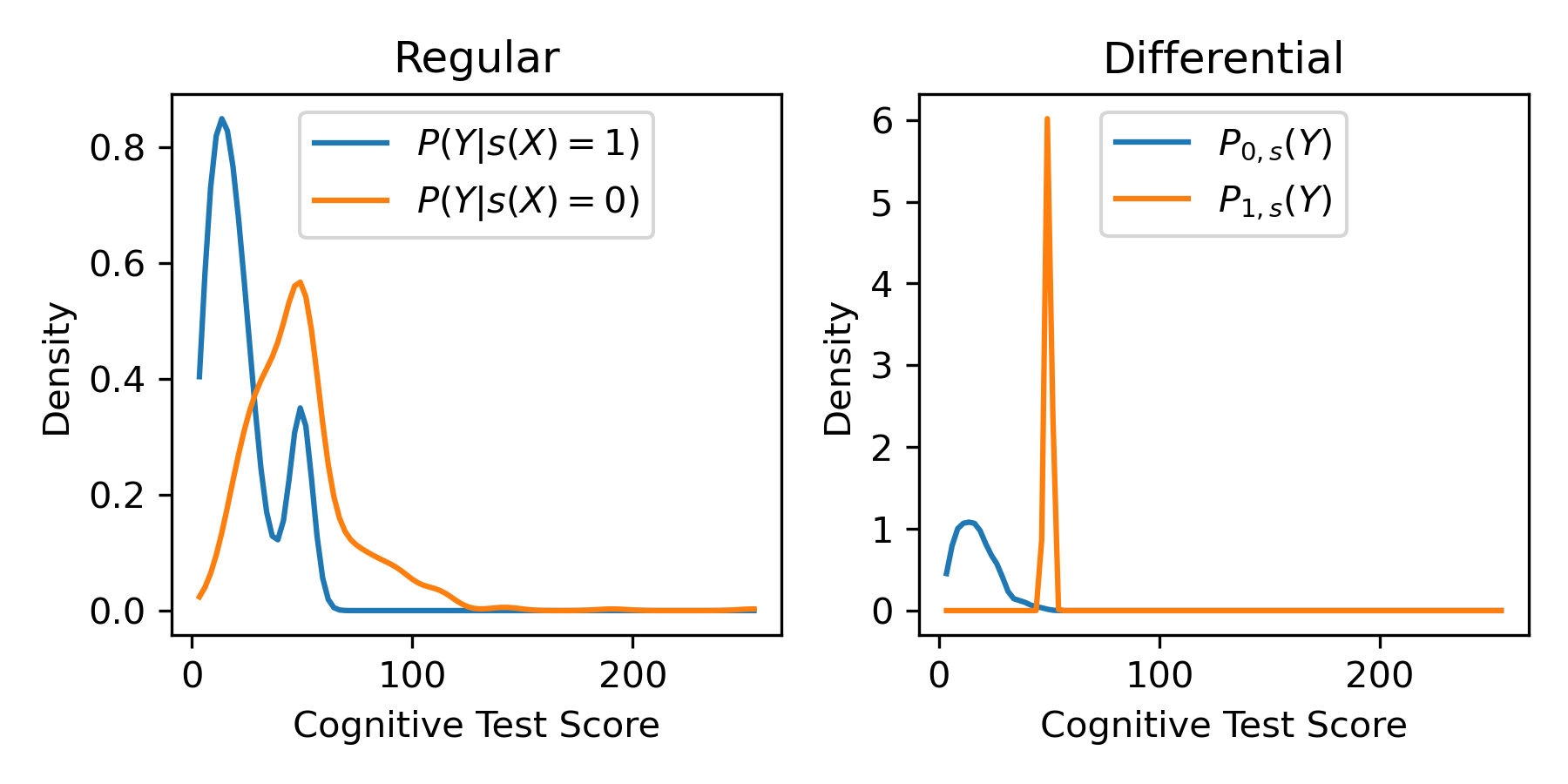} &
        \includegraphics[width=0.24\textwidth]{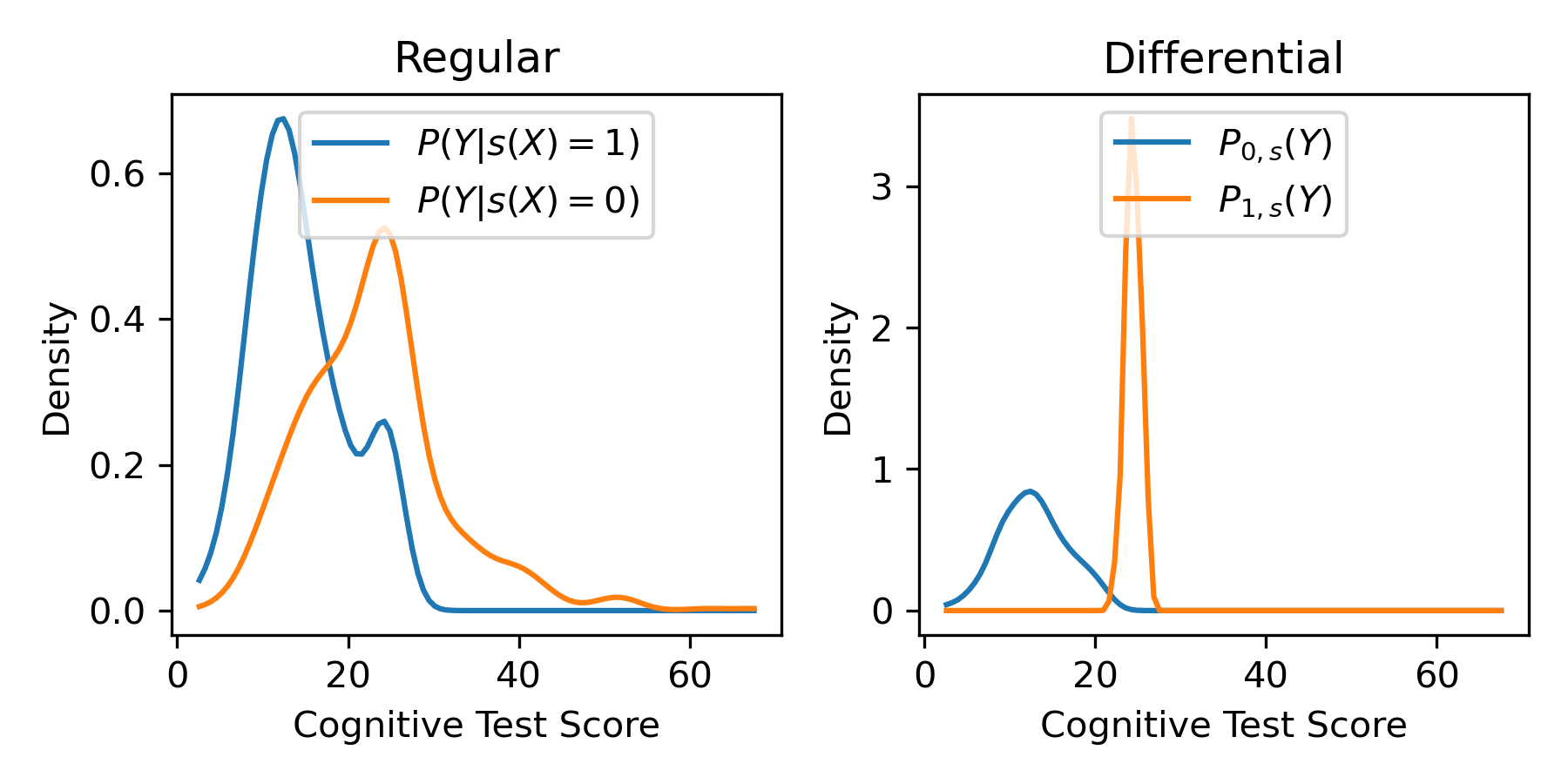} & \\
    \end{tabular}
    \caption{Subgroups discovered in the IHDP dataset by \ourmethod.}
    \label{fig:ihdp_overall}
\end{figure}

\begin{table*}[ht]
    \centering
    \begin{tabular}{|p{0.45\textwidth}|p{0.45\textwidth}|}
        \hline
        \textbf{Subgroup 1-5 Rules} & \textbf{Subgroup 6-10 Rules} \\
        \hline
        $-1.59 < X2 \& X3 < 1.50 \& X5 < 1.05 \& X14 < 0.64 \& X19 < 0.62$
 & $-2.60 < X0 \& -3.48 < X1 \& X2 < 2.82 \& X3 < 1.29 \& X5 < 1.01 \& X14 < 0.65 \& X24 < 0.63$
 \\
        \hline
        $X0 < 1.40 \& -4.48 < X4 \& X5 < 2.63 \& X7 < 0.68 \& X11 < 0.62 \& X20 < 0.65$
 & $X1 < 1.53 \& X3 < 0.75 \& X10 < 0.53$
 \\
        \hline
        $X0 < 0.54 \& X1 < 0.64 \& X3 < 1.63 \& X21 < 0.69$
 & $X1 < 2.29 \& X4 < 0.86 \& X6 < 0.40 \& X14 < 0.64 \& X22 < 0.75 \& X23 < 0.67$
 \\
        \hline
        $X1 < 1.20 \& -4.43 < X4 \& X5 < 0.70 \& X20 < 0.64 \& X23 < 0.65$
 & $0.24 < X17 \& X19 < 0.85$
 \\
        \hline
        $-0.21 < X3 \& -1.78 < X5 \& X6 < 0.40 \& X13 < 1.47 \& X19 < 0.58$
 & $-1.24 < X2 \& X3 < 1.42 \& X4 < 0.66 \& X11 < 0.54 \& X18 < 0.66$
 \\
        \hline
    \end{tabular}
    \caption{Rules for each discovered subgroup in the IHDP dataset.}
    \label{tab:ihdp_rules}
\end{table*}

\subsection{COVID-19}
We report the returned top subgroup from \pysubgroup, and two variants of \ourmethod: \ourmethod-$\max$ and \ourmethod-$\min$.
\ourmethod-$\max$ tries to maximize the distance $D$, whilst for \ourmethod-$\min$ we set $D = - D_{JS}$, such that it minimizes the distance.
In this section, we provide the top three subgroups found by \ourmethod-$\max$, \ourmethod-$\min$ and \pysubgroup on the COVID-19 dataset.
To obtain multiple subgroups using \ourmethod, we iteratively run the optimization and remove those points, which have 
been identified as subgroup from the dataset after each run.

\begin{figure}[h]
    \begin{subfigure}[t]{0.3\linewidth}
        \includegraphics[width=\linewidth]{figs/covid/max_subgroup_1_overall.png}
        \caption{\ourmethod-$\max$ Subgroup 1: race = not black AND hypertension = 0 AND diabetes = 1 AND cerebrovascular.disease = 0}
    \end{subfigure}
    \hfill
    \begin{subfigure}[t]{0.3\linewidth}
        \includegraphics[width=\linewidth]{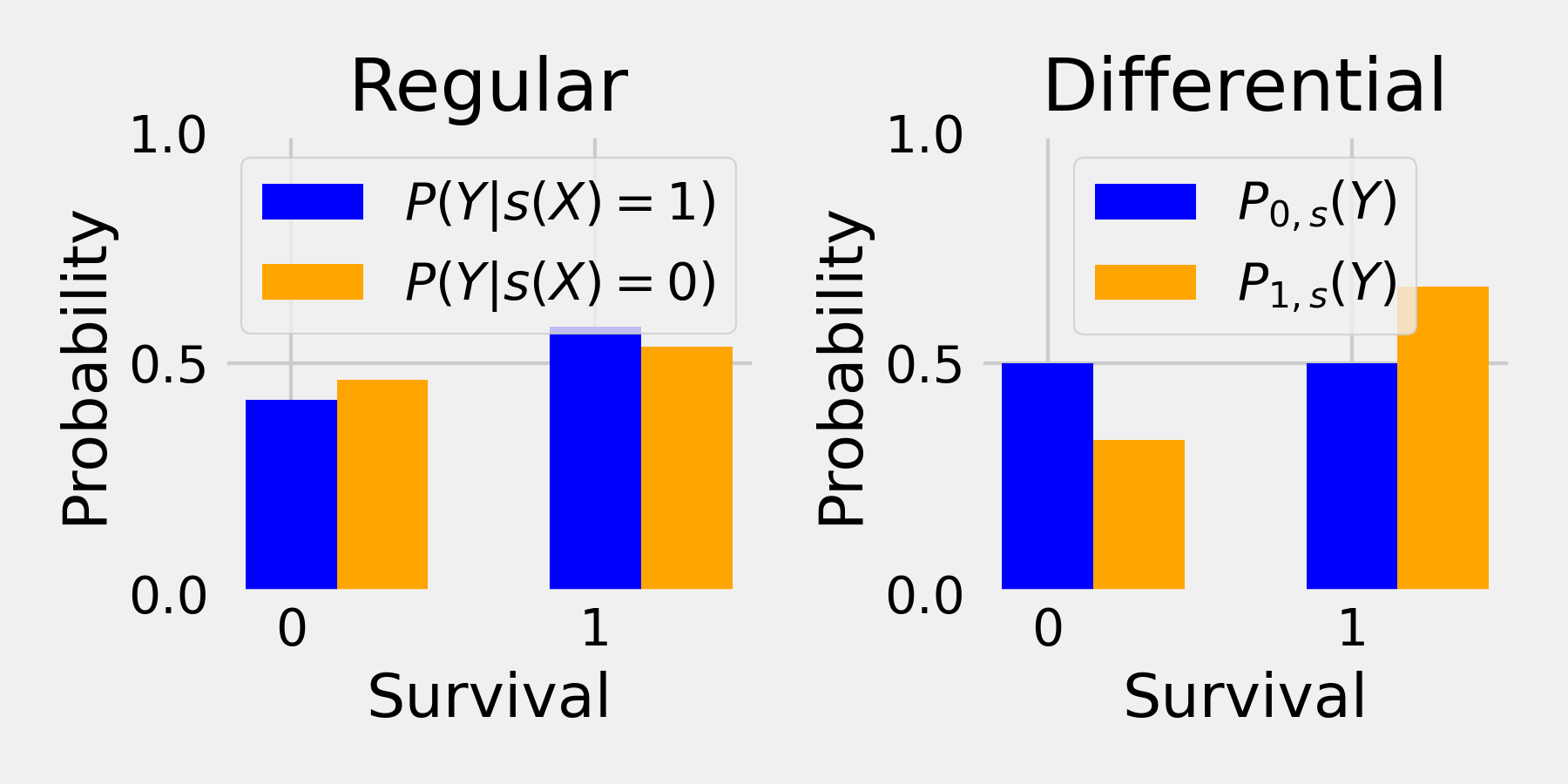}
        \caption{\ourmethod-$\max$ Subgroup 2: age $>30$ AND coronary.artery.disease = 0 AND cerebrovascular.disease = 0 AND hepatitis = 0 AND dementia = 0 AND chronic.kidney.disease = 0 AND dementia = 0 AND cancer = 0}
    \end{subfigure}
    \hfill
    \begin{subfigure}[t]{0.3\linewidth}
        \includegraphics[width=\linewidth]{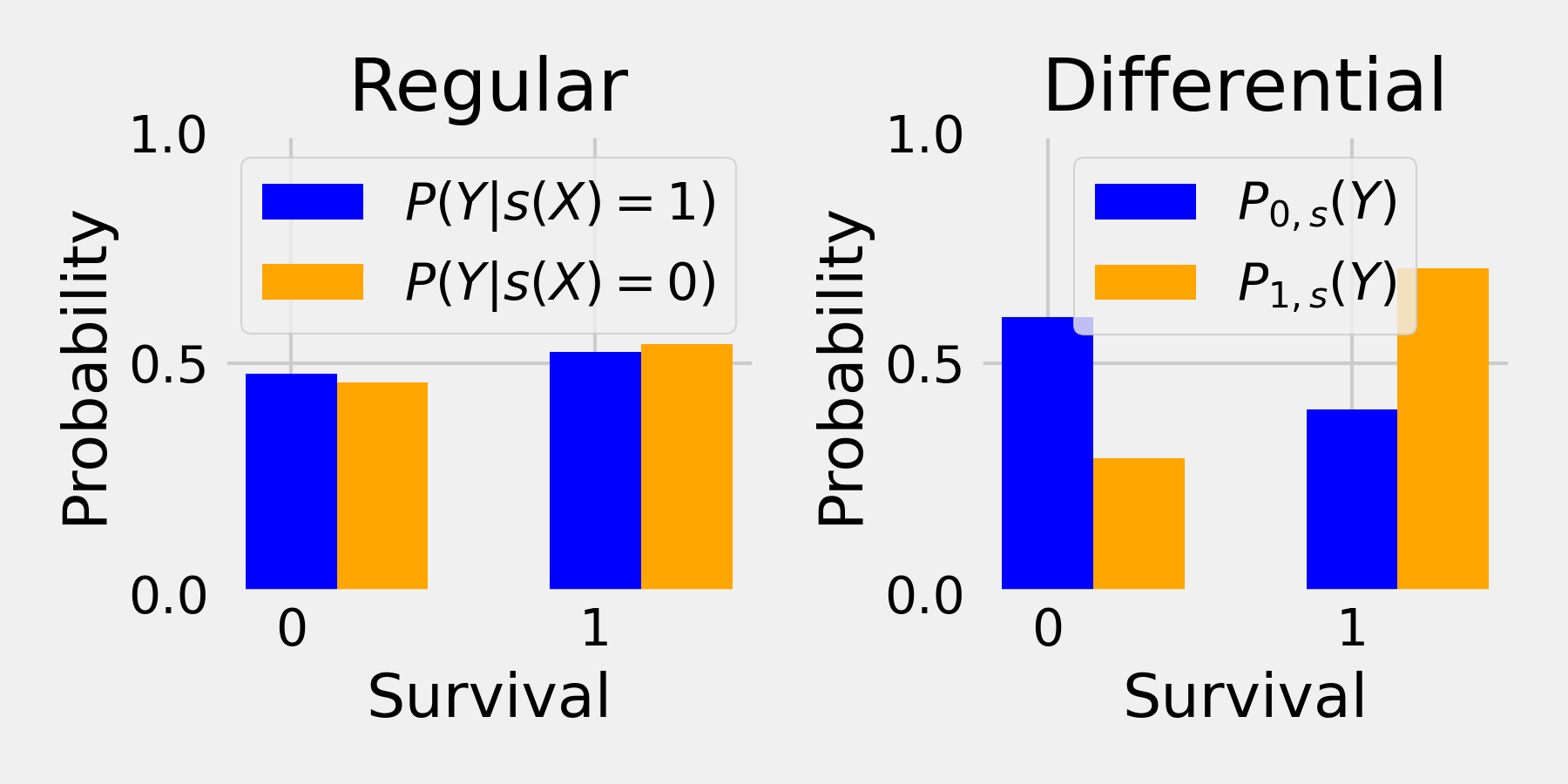}
        \caption{\ourmethod-$\max$ Subgroup 3:  19.10 $<$ age $<$ 79.85 AND race = not white AND race not other AND hypertension = 0 AND copd = 0 AND dementia = 0}
    \end{subfigure}
    \begin{subfigure}[t]{0.3\linewidth}
        \includegraphics[width=\linewidth]{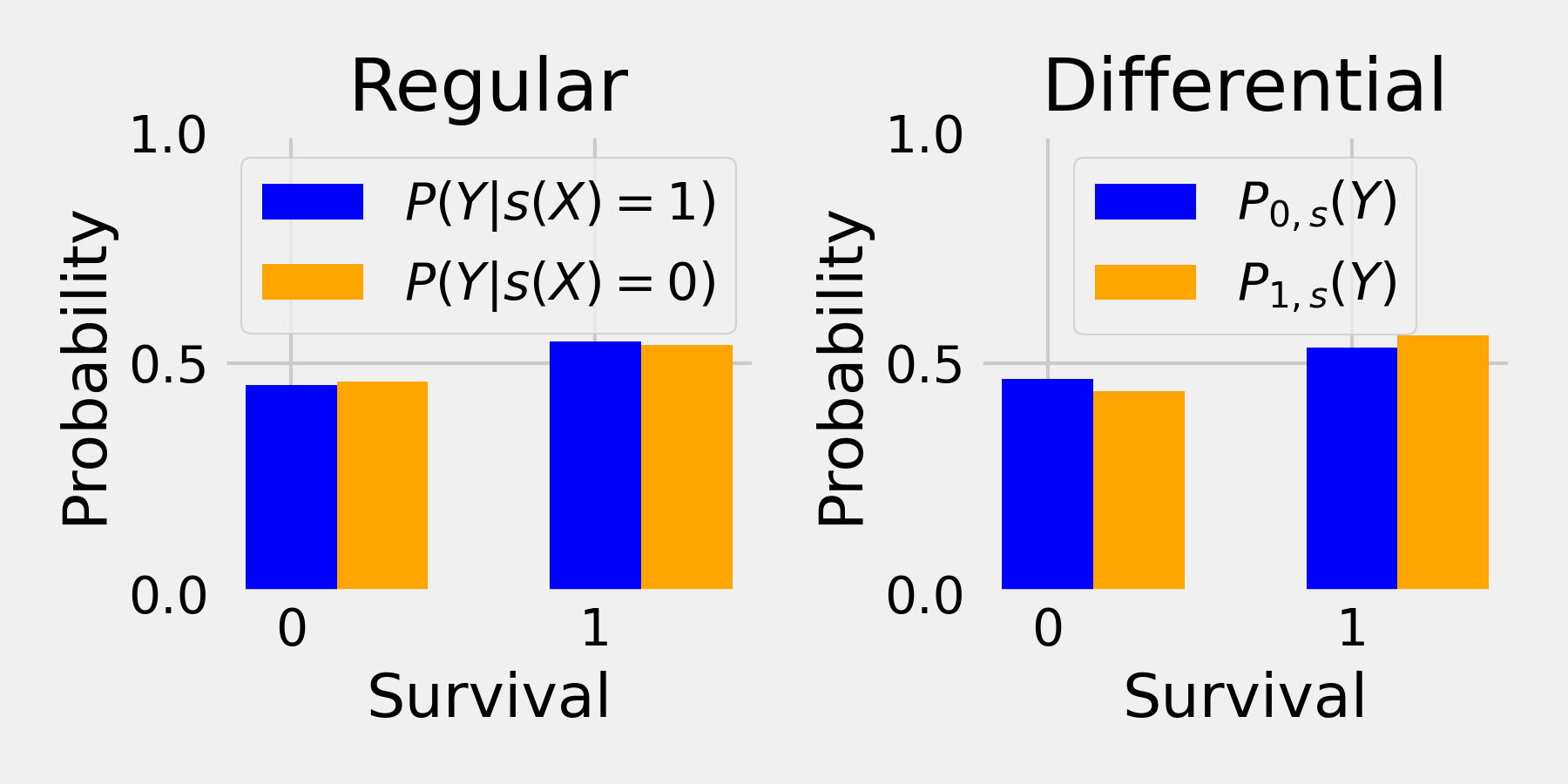}
        \caption{\ourmethod-$\min$ Subgroup 1: 11.56 $<$ age $<$ 80.76 AND  race = not white AND  hypertension = 1 AND hyperlipidemia = 1 AND chf = 1}
    \end{subfigure}
    \hfill
    \begin{subfigure}[t]{0.3\linewidth}
        \includegraphics[width=\linewidth]{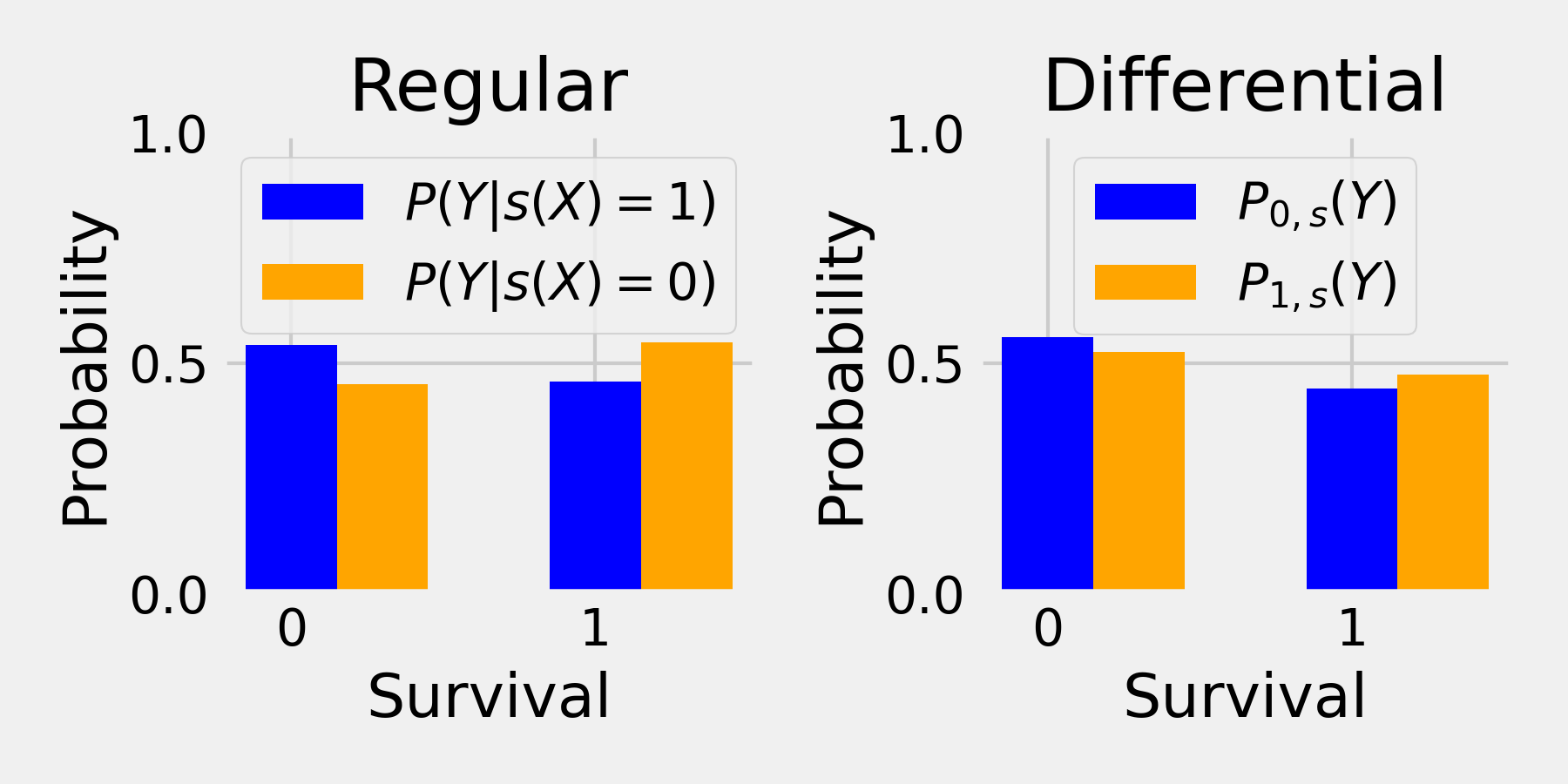}
        \caption{\ourmethod-$\min$ Subgroup 2: 71.90 $<$ age $<$ 99.57 AND race = not white AND hypertension = 1}
    \end{subfigure}
    \hfill
    \begin{subfigure}[t]{0.3\linewidth}
        \includegraphics[width=\linewidth]{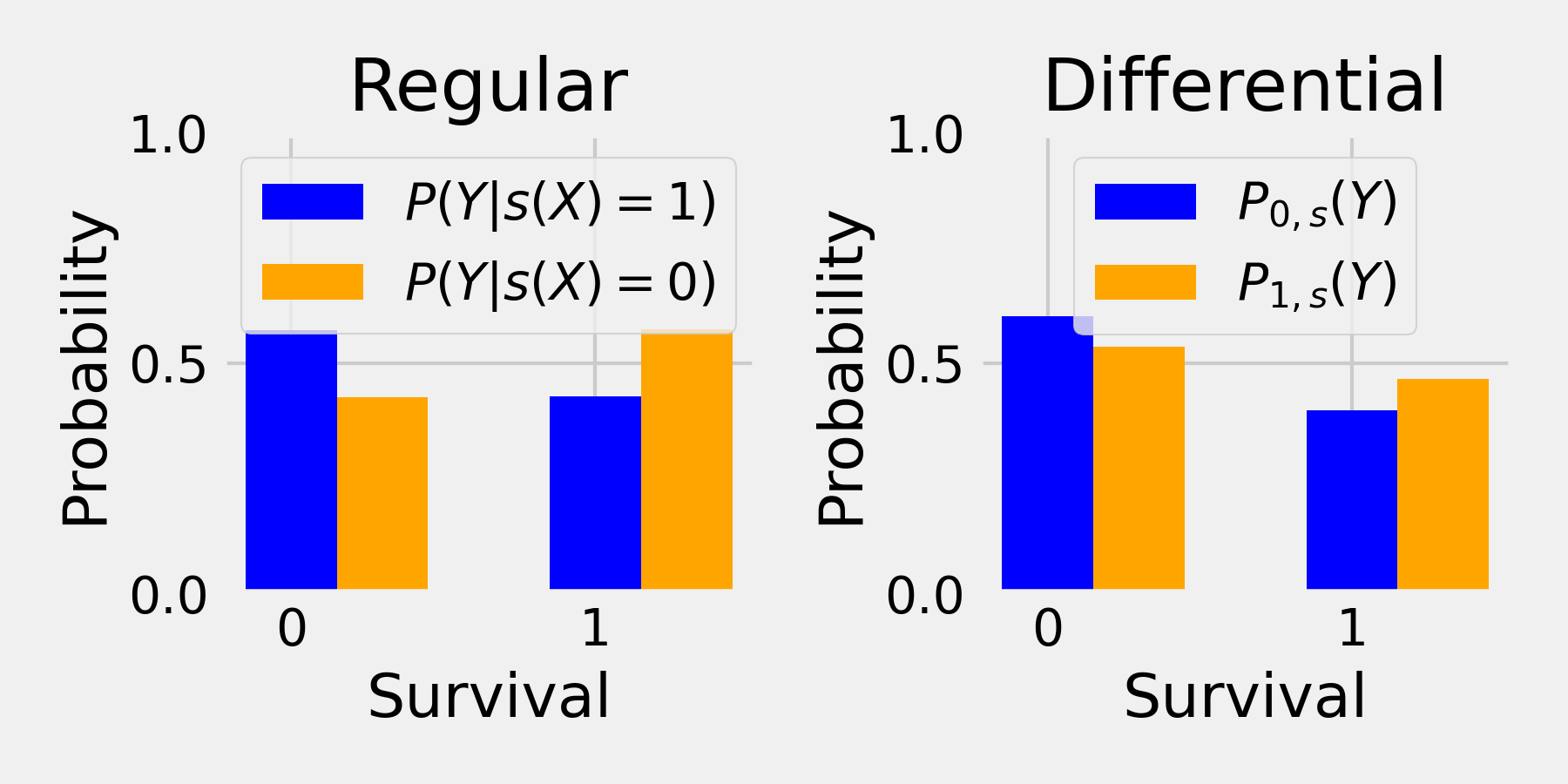}
        \caption{\ourmethod-$\min$ Subgroup 3: 62.09 $<$ age $<$ 76.33 AND hypertension = 1 AND chf = 1}
    \end{subfigure}

    \begin{subfigure}[t]{0.3\linewidth}
        \includegraphics[width=\linewidth]{figs/covid/ps_subgroup_1_overall.png}
        \caption{\pysubgroup Subgroup 1: age$<$74.0 AND coronary.artery.disease==0.0 AND cerebrovascular.disease==0.0}
    \end{subfigure}
    \hfill
    \begin{subfigure}[t]{0.3\linewidth}
        \includegraphics[width=\linewidth]{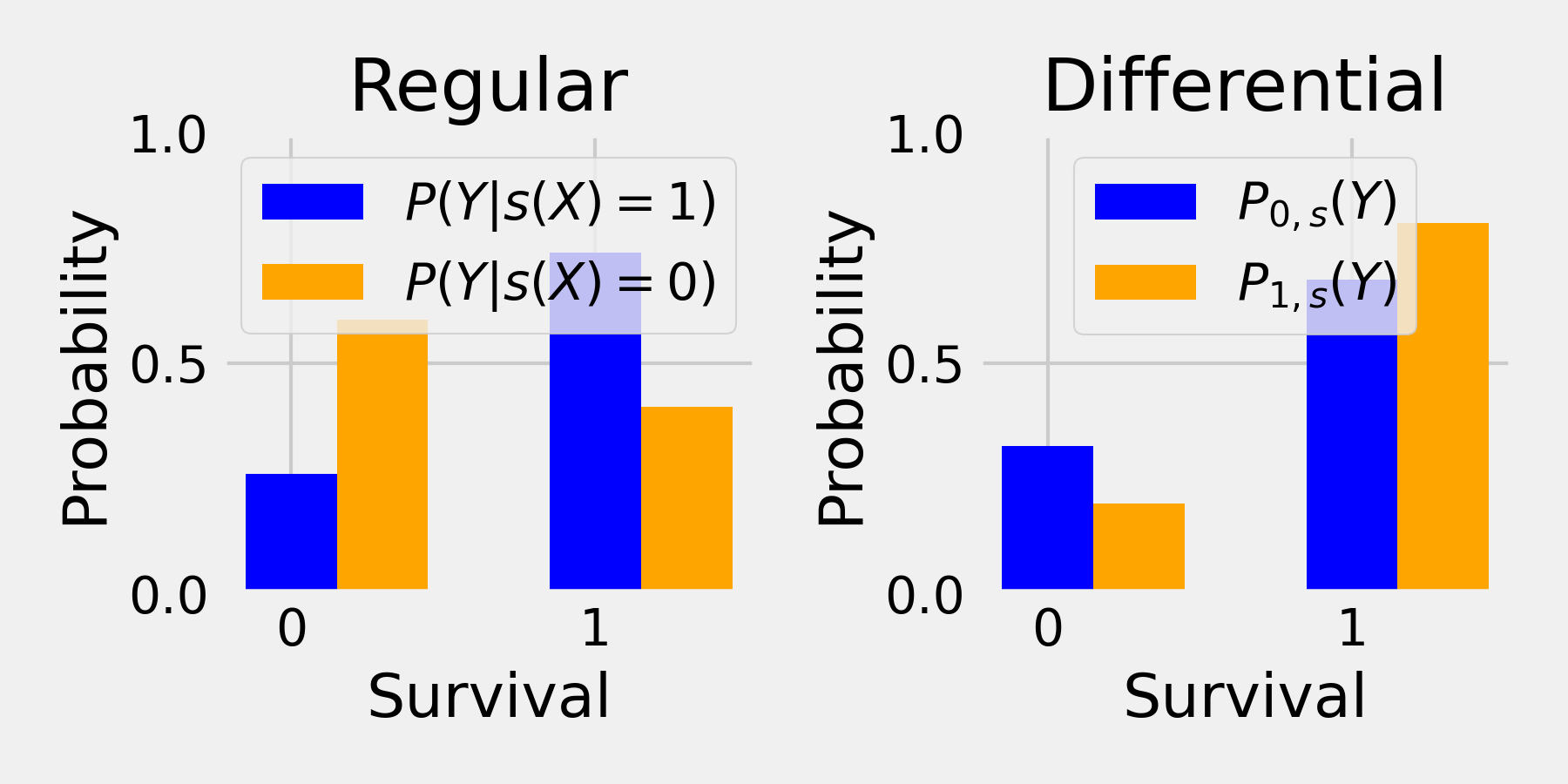}
        \caption{\pysubgroup Subgroup 2: age$<$74.0 AND coronary.artery.disease==0.0 AND cerebrovascular.disease==0.0}
    \end{subfigure}
    \hfill
    \begin{subfigure}[t]{0.3\linewidth}
        \includegraphics[width=\linewidth]{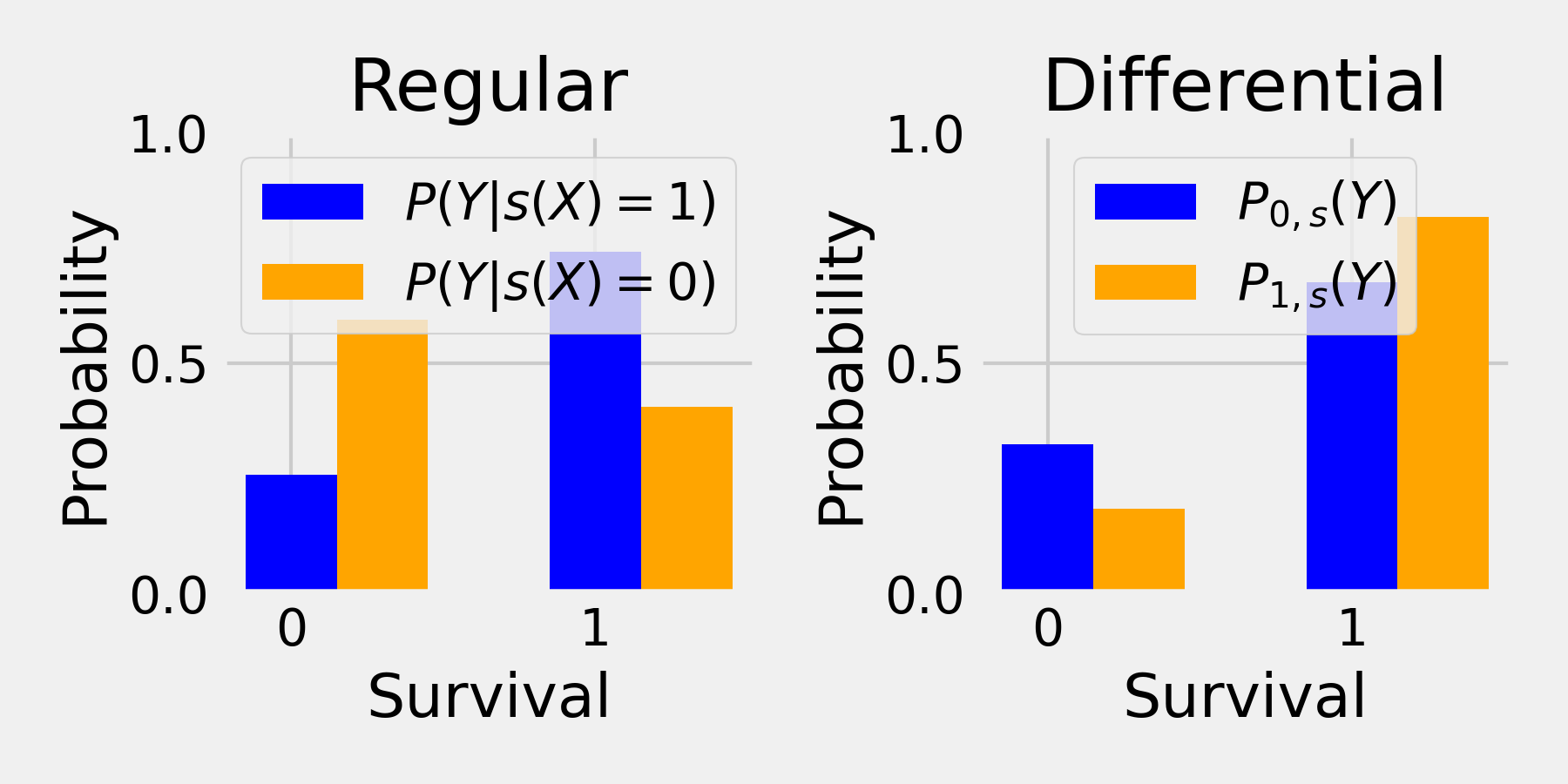}
        \caption{\pysubgroup Subgroup 3: age$<$72.0 AND coronary.artery.disease==0.0 AND hepatitis==0.0}
    \end{subfigure}
    \caption{Top three subgroups discovered by \ourmethod-$\max$, \ourmethod-$\min$ and \pysubgroup on the COVID19 ICU dataset \cite{lambert2022using}.}
\end{figure}

\clearpage
\subsection{Life Expectancy}
We further present the results discovered on the WHO Life expectancy dataset \citep{kumarajarshi_ray_2020}. We differentiate between the developed nations ($\attributevar=1$) and developing/least developed countries ($\attributevar=0$).
We report the returned top subgroup from \pysubgroup, and two variants of \ourmethod: \ourmethod-$\max$ and \ourmethod-$\min$.
\ourmethod-$\max$ tries to maximize the distance $D$, whilst for \ourmethod-$\min$ we set $D = - D_{JS}$, such that it minimizes the distance.
In this section, we provide the top three subgroups found by \ourmethod-$\max$, \ourmethod-$\min$ and \pysubgroup on the COVID-19 dataset.
To obtain multiple subgroups using \ourmethod, we iteratively run the optimization and remove those points, which have 
been identified as subgroup from the dataset after each run.

\ourmethod-$\max$ 
\begin{figure}[h]
    \begin{subfigure}[t]{.3\linewidth}
        \centering
        \includegraphics[width=\linewidth]{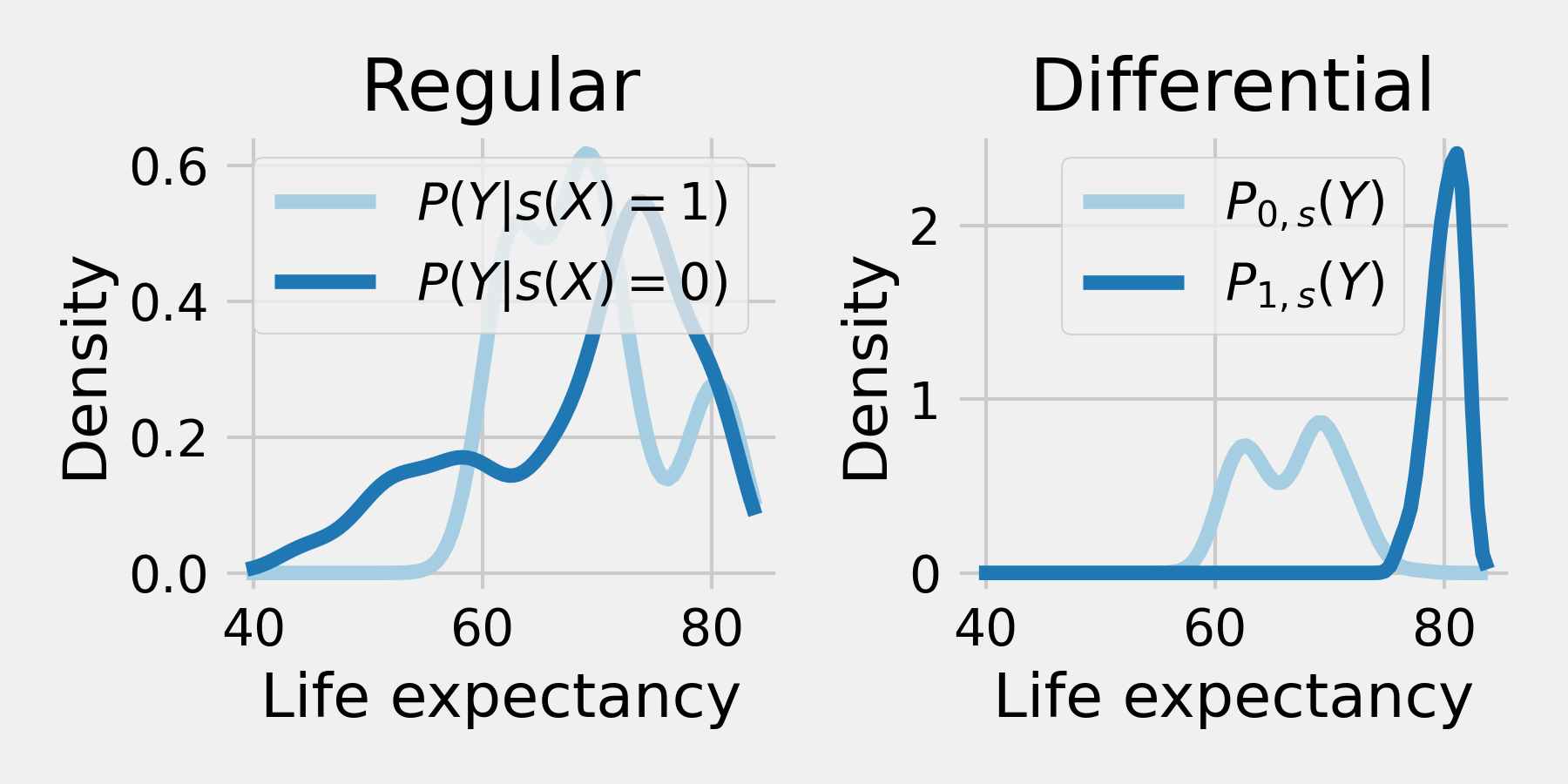}
        \caption{\ourmethod-$\max$ Subgroup 1: Under-five-deaths $<$ 112.36 AND Adult-mortality $<$ 254.76 AND Hepatitis-B $<$ 71.96
}
    \end{subfigure}
    \hfill
    \begin{subfigure}[t]{.3\linewidth}
        \centering
        \includegraphics[width=\linewidth]{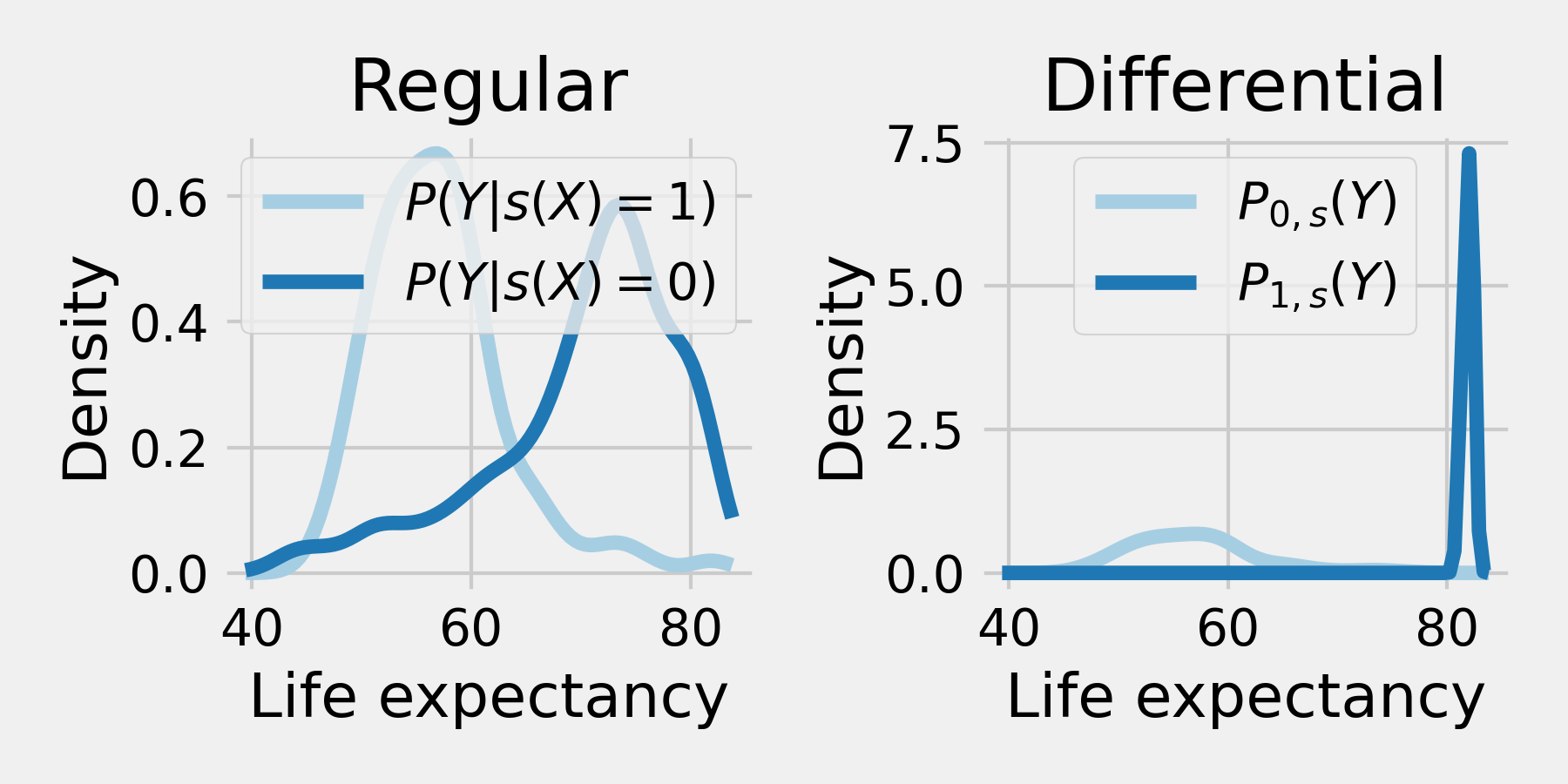}
        \caption{\ourmethod-$\max$ Subgroup 2: Under-five-deaths $<$ 172.39 AND Adult-mortality $<$ 493.28 AND Hepatitis-B $<$ 75.57 AND Measles $<$ 86.18
}
    \end{subfigure}
    \hfill
    \begin{subfigure}[t]{.3\linewidth}
        \centering
        \includegraphics[width=\linewidth]{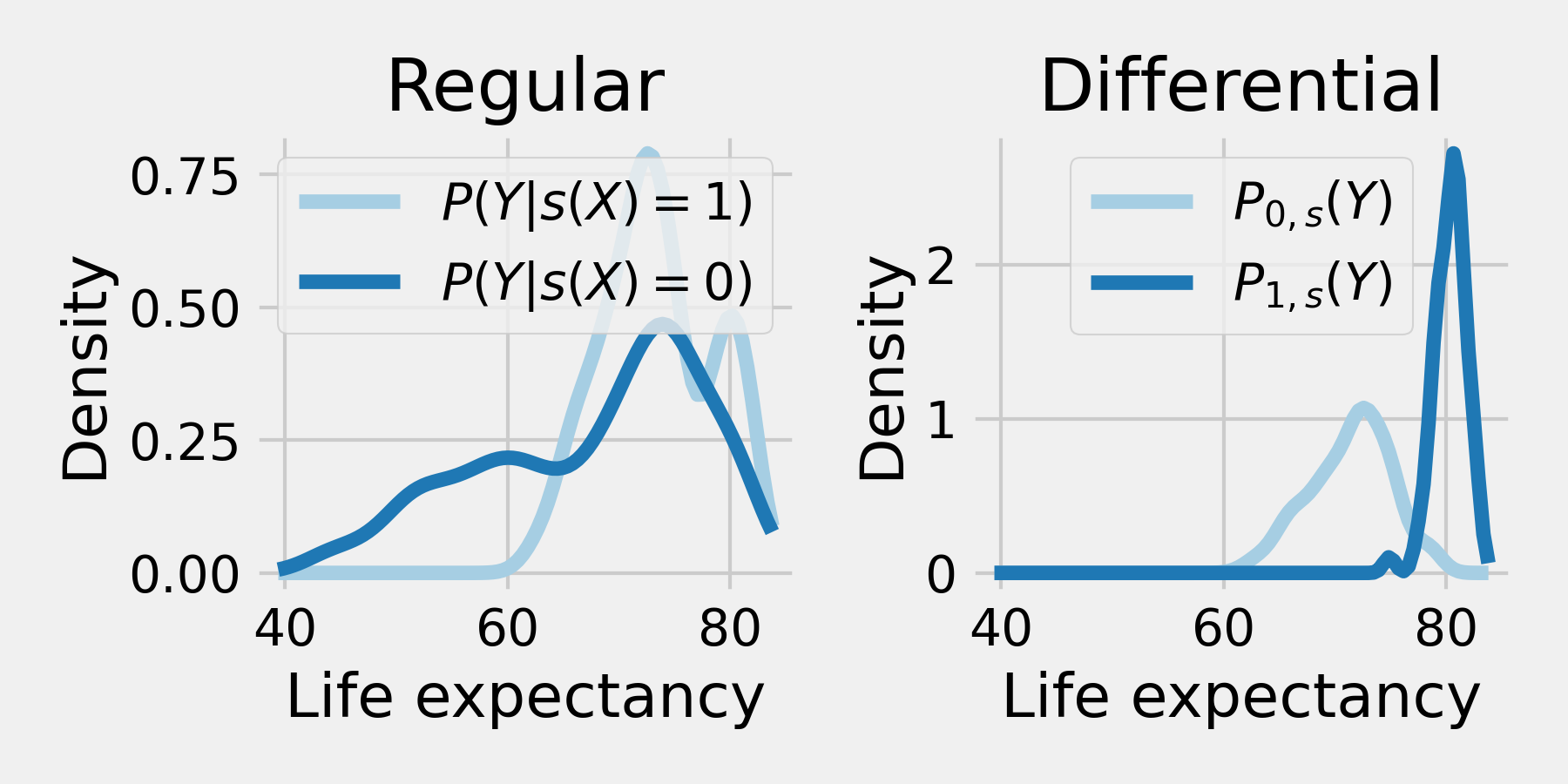}
        \caption{\ourmethod-$\max$ Subgroup 3: 2004.18 $<$ Year AND Infant-deaths $<$ 50.78 AND Adult-mortality $<$ 247.63 AND 32.50 $<$ Hepatitis-B AND Measles $<$ 86.46
}
    \end{subfigure}

    \begin{subfigure}[t]{.3\linewidth}
        \centering
        \includegraphics[width=\linewidth]{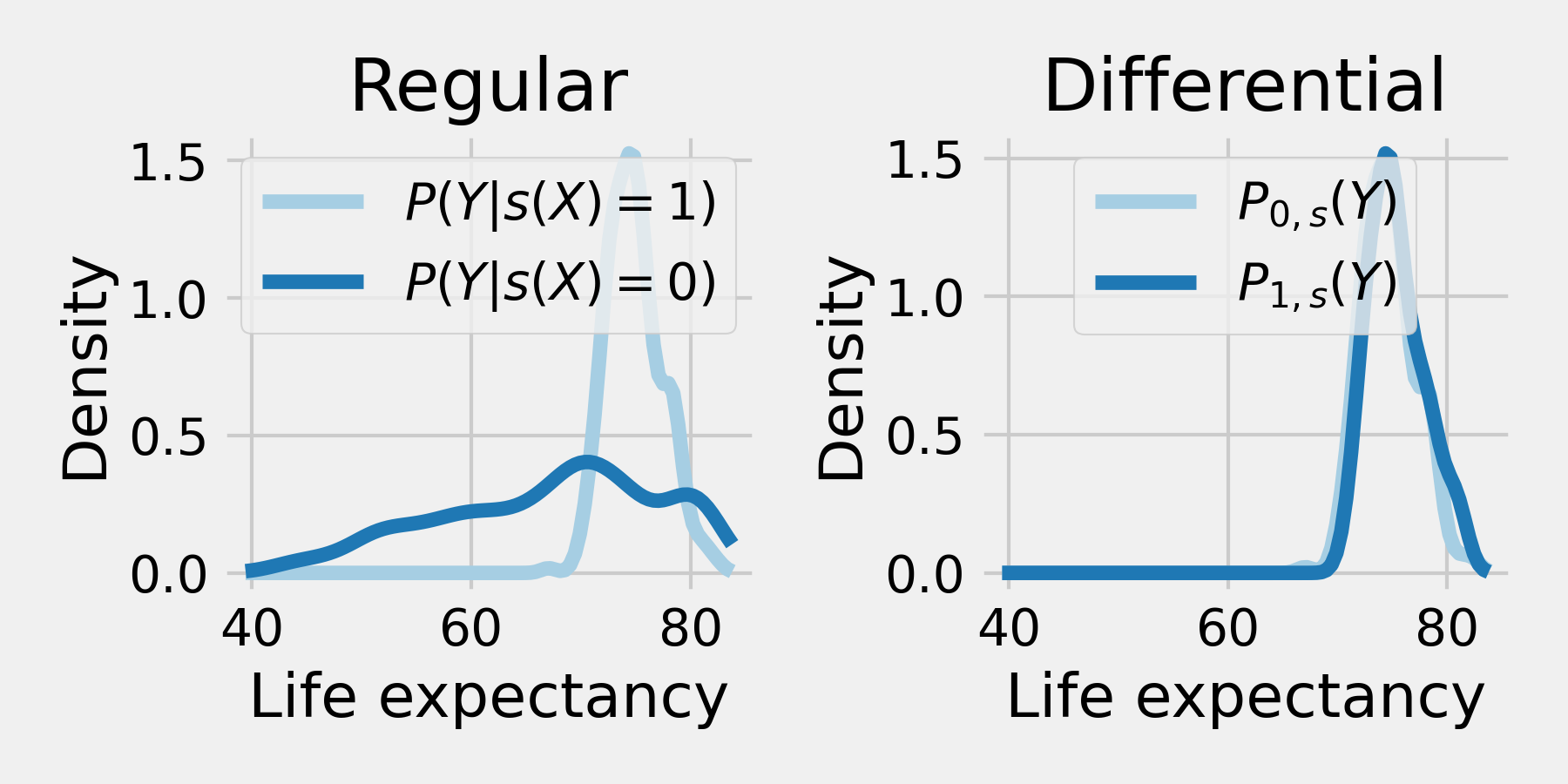}
        \caption{\ourmethod-$\min$ Subgroup 1: Adult-mortality $<$ 178.19 AND 32.09 $<$ Hepatitis-B AND 23.11 $<$ BMI AND 52.38 $<$ Diphtheria AND 4106.81 $<$ GDP-per-capita $<$ 92037.85 AND 1.06 $<$ Thinness-ten-nineteen-years $<$ 22.08 AND 2.16 $<$ Thinness-five-nine-years $<$ 22.49 AND 7.50 $<$ Schooling
}
    \end{subfigure}
    \hfill
    \begin{subfigure}[t]{.3\linewidth}
        \centering
        \includegraphics[width=\linewidth]{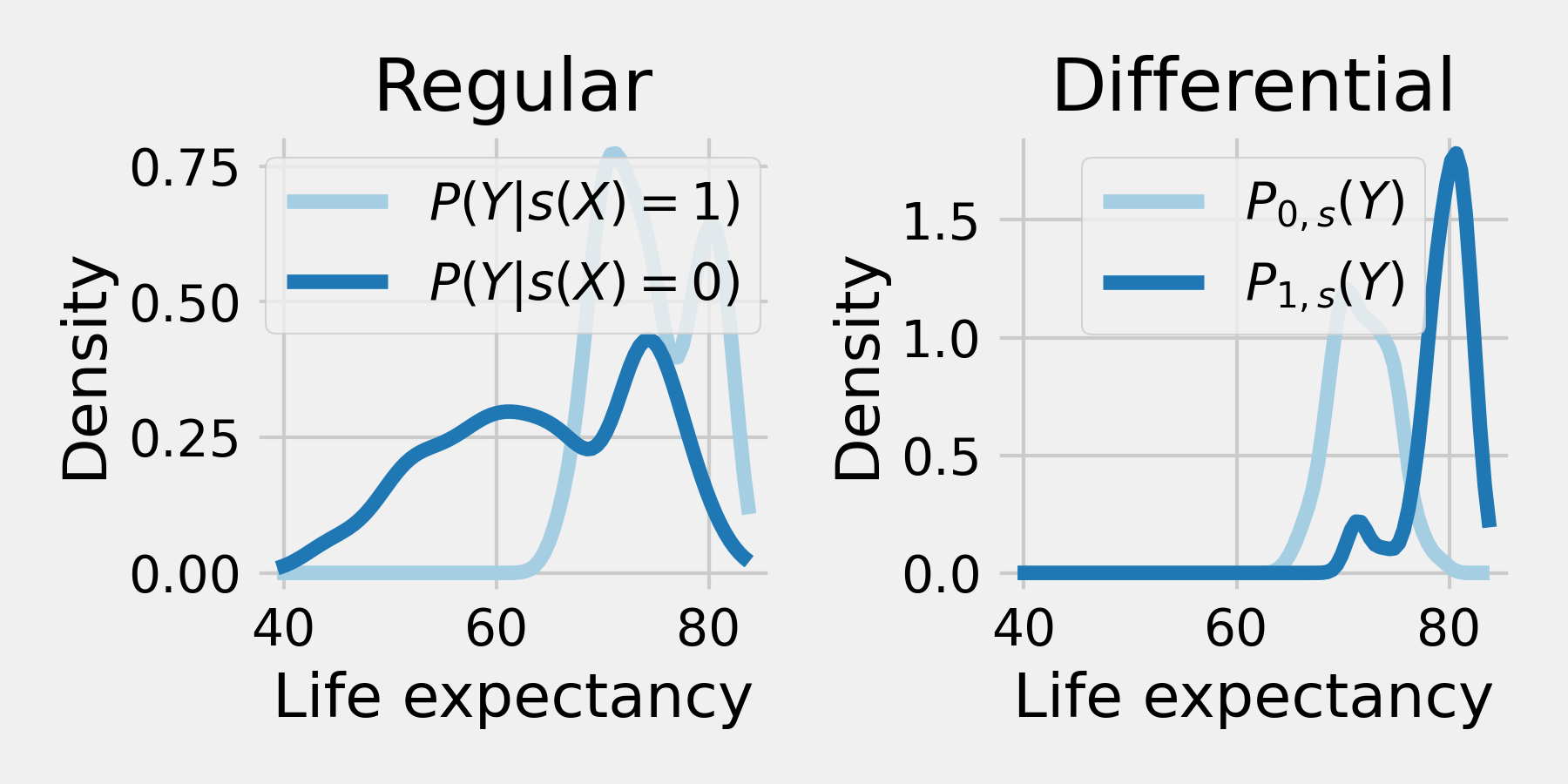}
        \caption{\ourmethod-$\min$ Subgroup 2: 103.97 $<$ Adult-mortality $<$ 219.50 AND 9487.06 $<$ GDP-per-capita $<$ 24595.72 AND Thinness-ten-nineteen-years $<$ 20.01 AND Thinness-five-nine-years $<$ 22.55 AND 5.75 $<$ Schooling
}
    \end{subfigure}
    \hfill
    \begin{subfigure}[t]{.3\linewidth}
        \centering
        \includegraphics[width=\linewidth]{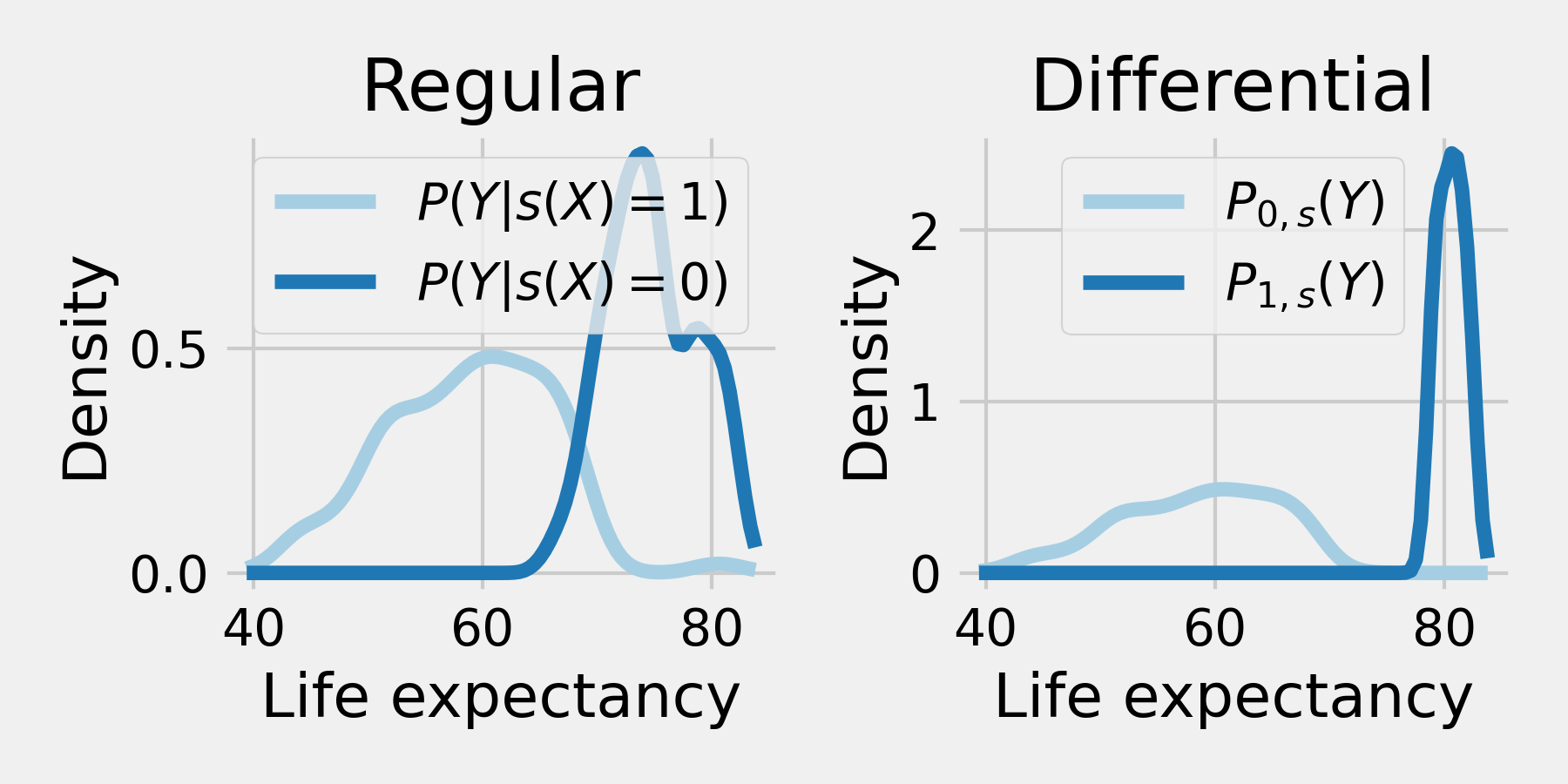}
        \caption{\ourmethod-$\min$ Subgroup 3: Infant-deaths $<$ 83.39 AND Adult-mortality $<$ 436.42 AND Incidents-HIV $<$ 17.45
}
    \end{subfigure}

    \begin{subfigure}[t]{.3\linewidth}
        \centering
        \includegraphics[width=\linewidth]{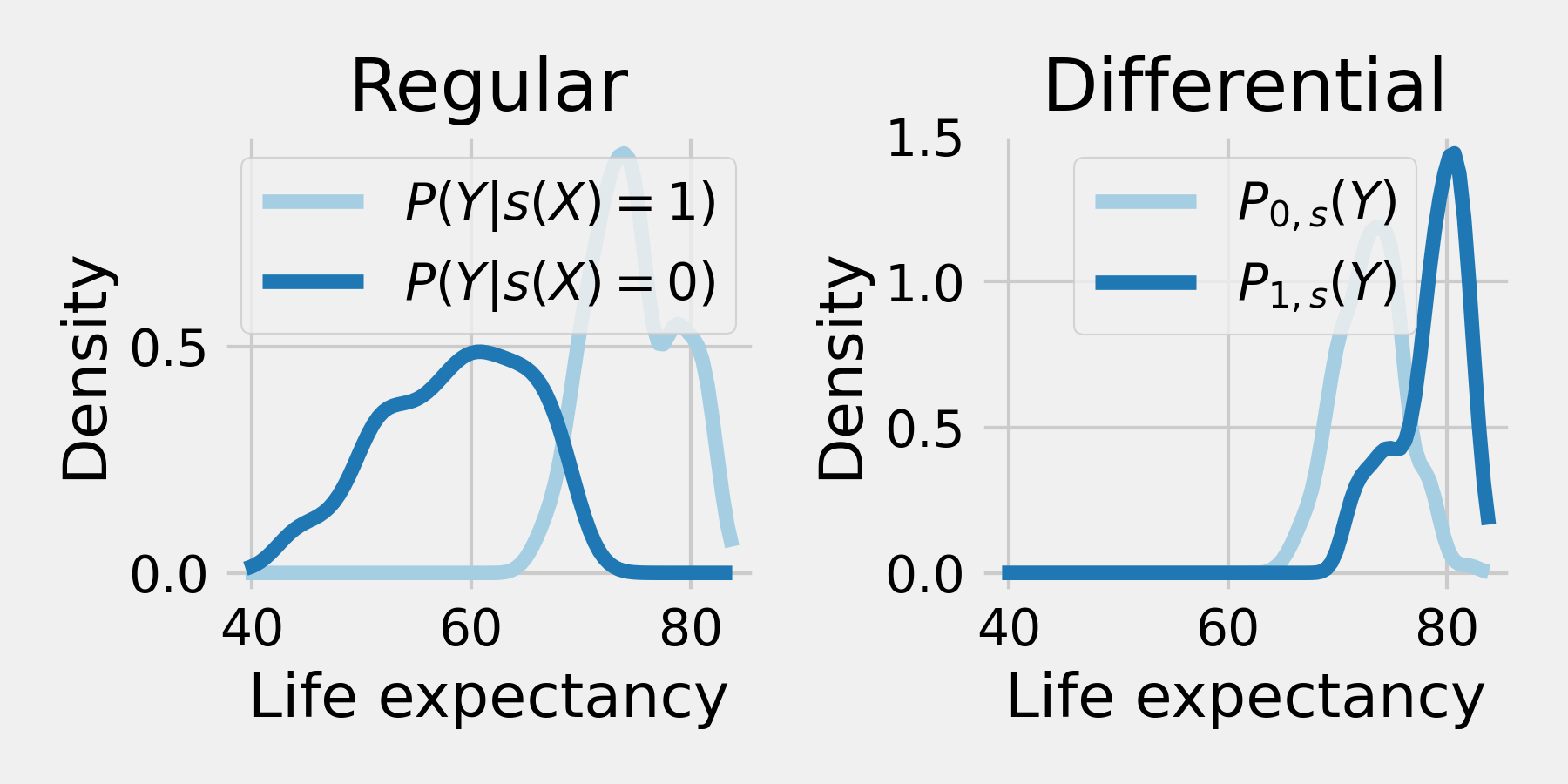}
        \caption{\pysubgroup Subgroup 1: GDP-per-capita$>=$831.0 AND Under-five-deaths$<$41.0 AND Adult-mortality$<$247.14
}
    \end{subfigure}
    \hfill
    \begin{subfigure}[t]{.3\linewidth}
        \centering
        \includegraphics[width=\linewidth]{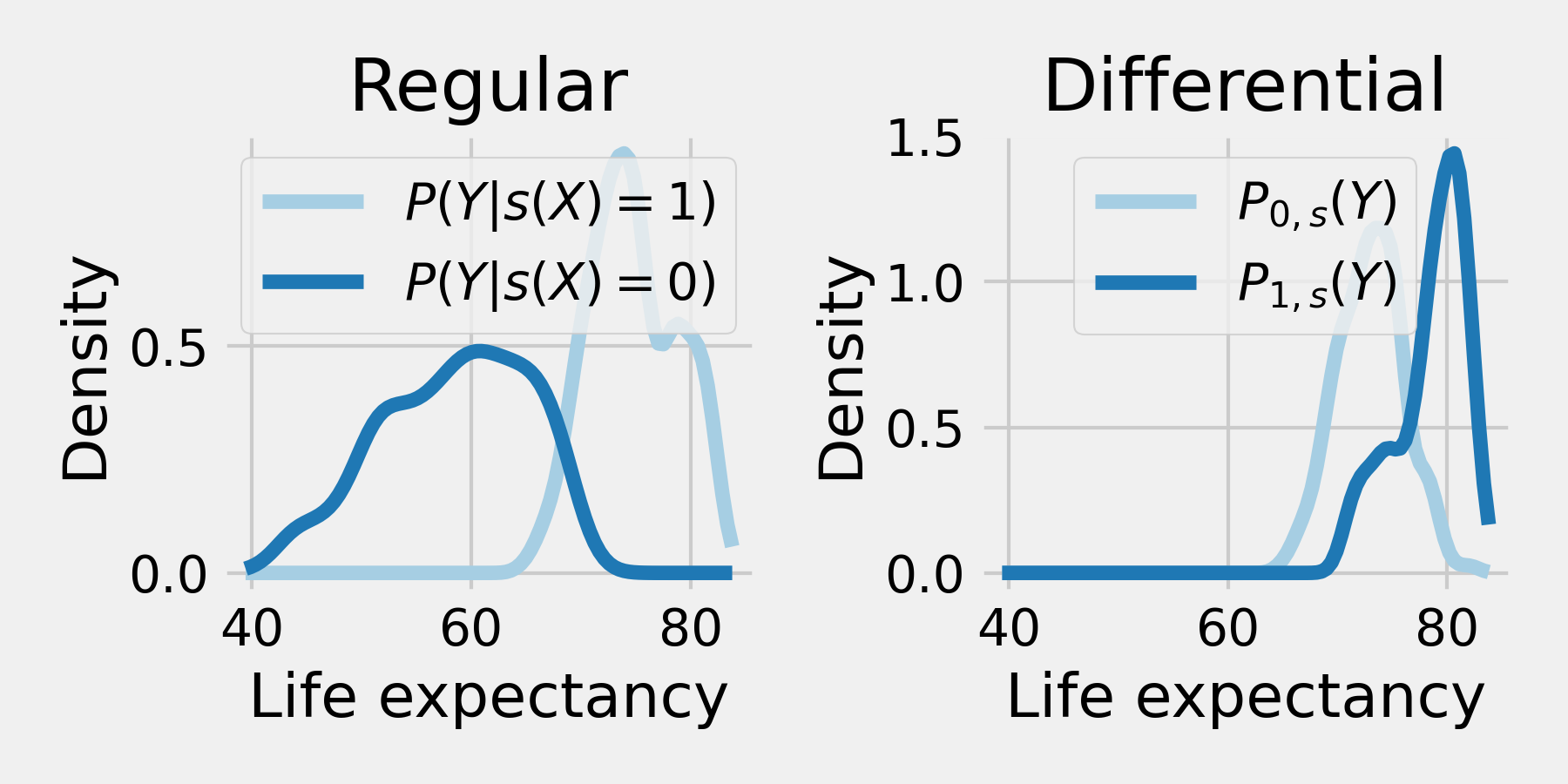}
        \caption{\pysubgroup Subgroup 2: Under-five-deaths$<$41.0 AND Adult-mortality$<$247.14
}
    \end{subfigure}
    \hfill
    \begin{subfigure}[t]{.3\linewidth}
        \centering
        \includegraphics[width=\linewidth]{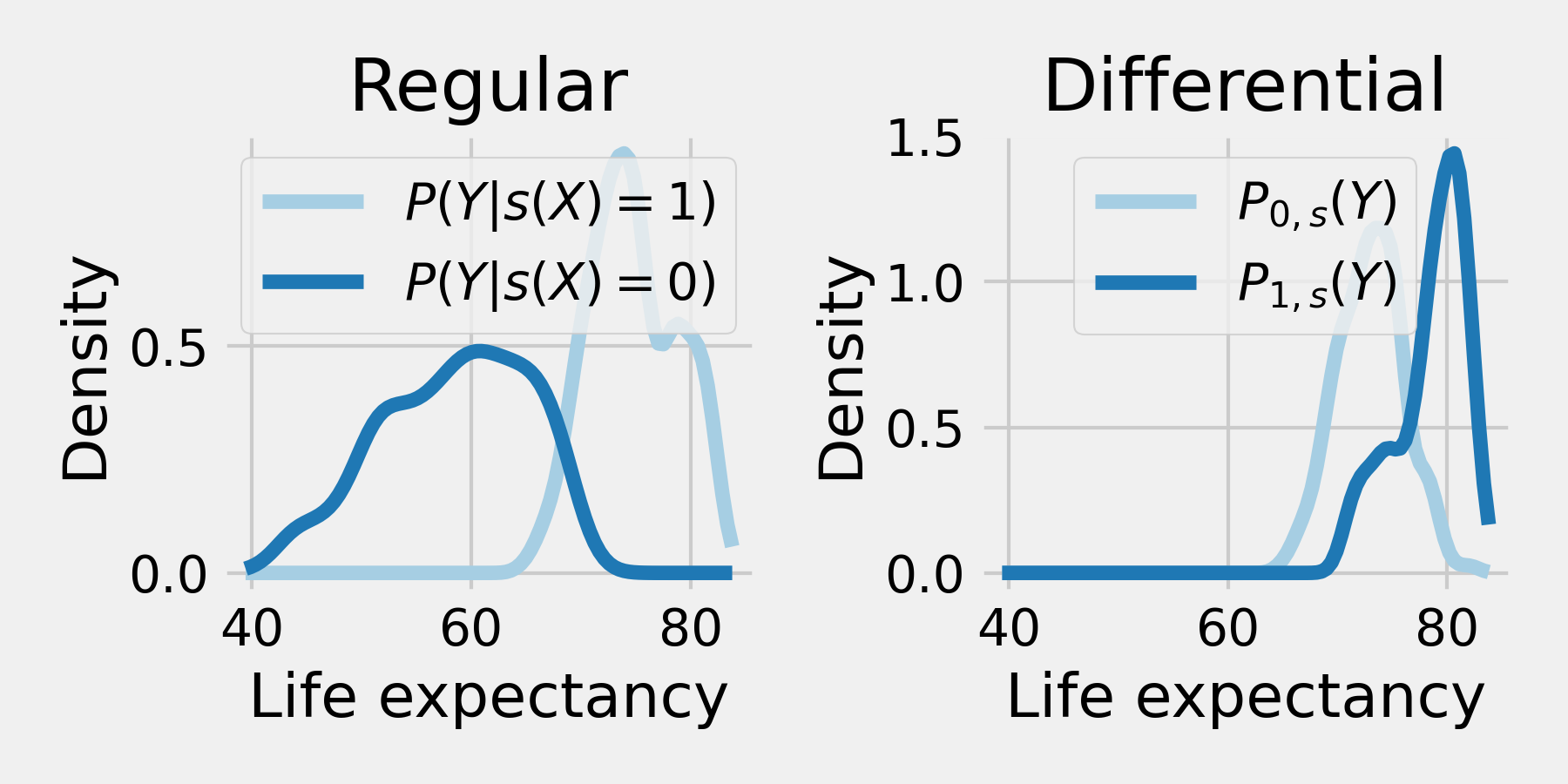}
        \caption{\pysubgroup Subgroup 3: Under-five-deaths$<$41.0 AND Adult-mortality$<$247.14 AND Adult-mortality$<$429.31
}
    \end{subfigure}
        \caption{Top three subgroups discovered by \ourmethod-$\max$, \ourmethod-$\min$ and \pysubgroup on the WHO Life Expectancy Dataset.}
\end{figure}

\end{document}